\documentclass[lettersize,journal]{IEEEtran}
\usepackage{amsmath,amsfonts}
\usepackage{algorithmic}
\usepackage{algorithm}
\usepackage{array}
\usepackage[caption=false,font=normalsize,labelfont=sf,textfont=sf]{subfig}
\usepackage{textcomp}
\usepackage{stfloats}
\usepackage{url}
\usepackage{verbatim}
\usepackage{graphicx}
\usepackage{orcidlink}
\usepackage{lipsum}
\usepackage{bm}
\usepackage{rotating}
\usepackage{pdflscape}
\usepackage{tcolorbox}
\usepackage{hyperref,url,dsfont}
\usepackage{makecell}
\usepackage{anyfontsize}
\usepackage{bbding}
\usepackage{xspace}
\usepackage{xfrac}
\usepackage{threeparttable}
\usepackage{pifont}
\usepackage{tikz}
\usepackage{wrapfig}
\usepackage{accents}
\usepackage{amsmath,mathrsfs,amssymb,amsfonts,enumitem}
\usepackage{multirow}
\usepackage{tabularx}
\usepackage[numbers,sort&compress]{natbib}

\setlist[itemize]{leftmargin=5.5mm}

\newcommand{\cmark}{\ding{51}}%
\newcommand{\xmark}{\ding{55}}%

\renewcommand{\tilde}{\widetilde}
\renewcommand{\hat}{\widehat}

\def \D {\mathcal{D}}
\def \E {\mathbb{E}}

\def \O {\mathcal{O}}

\def \R {\mathbb{R}}

\def \X {\mathcal{X}}

\def \g {\mathbf{g}}
\def \h {\boldsymbol{h}}
\def \r {\boldsymbol{r}}
\def \p {\boldsymbol{p}}

\def \u {\mathbf{u}}
\def \w {\mathbf{w}}
\def \m {\boldsymbol{m}}

\def \x {\mathbf{x}}
\def \y {\mathbf{y}}
\def \z {\mathbf{z}}

\def \xh {\hat{\x}}

\def \Ot {\tilde{\O}}

\def \ellb {\boldsymbol{\ell}}

\def \is {i_{\star}}

\def \rb {\bar{r}}

\usepackage{mathtools}
\let\norm\undefined 
\DeclarePairedDelimiter\norm{\lVert}{\rVert}
\DeclarePairedDelimiter\abs{\lvert}{\rvert}
\newcommand\inner[2]{\langle #1, #2 \rangle}

\def \Reg {\textsc{Reg}}
\DeclareMathOperator*{\AReg}{\textsc{I-Reg}}

\DeclareMathOperator*{\argmin}{arg\,min}

\def \AReg {\textsc{I-Reg}}

\usepackage{amsthm}
\newtheorem{myThm}{Theorem}
\newtheorem{myCor}{Corollary}

\newtheorem{myLemma}{Lemma}

\newtheorem{myProp}{Proposition}

\theoremstyle{definition}
\newtheorem{myAssump}{Assumption}

\newtheorem{myRemark}{Remark}

\newtheoremstyle{proofsketchstyle}{}{} {} {} {\itshape} {.}{ } {\thmname{#1}\thmnote{ #3}} 

\theoremstyle{proofsketchstyle}

\usepackage{graphicx,color} 

\definecolor{wine_red}{RGB}{228,48,64}
\definecolor{DSgray}{cmyk}{0,1,0,0}

\begin{document}

\title{Online Learning with \\ Gradient-Variation Interval Regret}

\author{Yan-Feng Xie, Shuche Wang, Peng Zhao and Zhi-Hua Zhou
  \thanks{Yan-Feng Xie, Peng Zhao and Zhi-Hua Zhou are with State Key Laboratory for Novel Software Technology and the School of Artificial Intelligence, Nanjing University, Nanjing 210023, China. (e-mail: xieyf@lamda.nju.edu.cn; zhaop@lamda.nju.edu.cn; zhouzh@lamda.nju.edu.cn). Shuche Wang is with Institute of Operations Research and Analytics, National University of Singapore, Singapore 117602 (e-mail: shuche.wang@u.nus.edu).}
}

\maketitle

\begin{abstract}
  This paper investigates non-stationary online learning using the metric of \emph{interval regret}, which requires an online algorithm to perform well over every time interval.
  We propose the first online learning algorithm that achieves an interval regret bound scaling with \emph{gradient variation}, a fundamental measure of the cumulative change in online function gradients, which relates to various problem-dependent quantities and is closely connected to stochastic optimization and other problems.
  Our method employs a simple and efficient two-layer online ensemble structure that achieves strong theoretical guarantees. Specifically, it enjoys a regret bound that simultaneously adapts to various problem-dependent quantities while also preserving the minimax-optimal rate in the worst case.
  Moreover, recognizing the challenge of hyperparameter tuning, we introduce a Lipschitz- and smoothness-agnostic variant that automatically adapts to these potentially unknown constants. This is primarily enabled by a novel \emph{Lipschitz-adaptive meta algorithm}, which may be of independent interest.
  Beyond interval regret, our method also yields broader implications: it provides versatile bounds for interval dynamic regret, a stronger measure that competes with changing comparators over any interval, and yields the first piecewise characterization for stochastic extended adversarial optimization. Theoretical findings are validated by experiments.
\end{abstract}


\section{Introduction}
This paper investigates online learning within the framework of online convex optimization (OCO)~\cite{ICML'03:zinkvich, book'16:Hazan-OCO}, a general paradigm for sequential decision-making between a learner and the environments. Under the OCO framework, the learner is required to provide a decision $\x_t \in \X$ from a convex decision set $\X$ iteratively at each round $t \in [T]$. Subsequently, the environments reveal a convex loss function $f_t:\X \mapsto \R$, and the learner incurs the loss $f_t(\x_t)$. A classic objective is to minimize the regret, which measures the cumulative loss against the single best fixed decision in hindsight:
\begin{align}
    \label{eq:regret}
    \Reg_T = \sum_{t=1}^T f_t(\x_t) - \min_{\x \in \X} \sum_{t=1}^T f_t(\x).
\end{align}
This regret benchmarks the performance against a fixed comparator $\x_{\star} \in \argmin_{\x \in \X} \sum_{t=1}^T f_t(\x)$ in hindsight, also known as the static regret. However, static regret becomes an inadequate measure in non-stationary environments where the optimal decision may change over time. To address this limitation, more robust performance measures have been proposed. This paper focuses on \emph{interval regret}~\cite{ICML'15:Daniely-adaptive} (also referred to as \emph{adaptive regret} in some literature~\cite{journal'07:Hazan-adaptive,ICML19:Zhang-Adaptive-Smooth, ICML'20:Ashok}), which is defined as the regret over any contiguous time interval $[r, s] \subseteq [T]$:
\begin{align}
    \label{eq:strongly-adaptive-regret}
    \AReg_{[r,s]} =\sum_{t = r}^{s} f_t(\x_t) - \min_{\x \in \X} \sum_{t = r}^{s} f_t(\x) .
\end{align}
Conceptually, interval regret provides a stronger performance guarantee by ensuring that the algorithm performs well against the local best-in-hindsight decision within an arbitrary interval. This measure is more suitable for changing environments, as the assumption of a single best decision is more plausible over shorter, local intervals than over the entire time horizon.

The central challenge in designing algorithms with interval regret guarantees is the need to handle inherent non-stationarity: the algorithms are required to perform well over any possible time interval, without knowing in advance which intervals will be difficult. To address this, a standard approach is to leverage the two-layer online ensemble framework~\cite{JMLR'24:Sword++} to mitigate the uncertainty. This framework employs a meta learner that dynamically combines the predictions from a pool of base learners. Each base learner is an instantiation of an OCO algorithm. To maintain computational efficiency, this large set of experts is managed by a schedule that determines which subset of base learners to run at each time step.

Algorithms based on such frameworks aim to minimize the interval regret defined in Eq.~\eqref{eq:strongly-adaptive-regret}, thereby ensuring strong performance across all possible intervals. Pioneering work in~\cite{ICML'15:Daniely-adaptive} achieves a regret guarantee of $\O(\sqrt{\abs{I}} \log T)$ for any interval $I\subseteq [T]$. Although this bound is known to be near-optimal up to logarithmic factors for convex functions in the worst case~\cite{conf/colt/AbernethyBRT08}, it is pessimistic. The regret scales with the interval length even in favorable local environments where the data might exhibit simpler patterns. For this reason, such guarantees are referred to as worst-case regret bounds.
To obtain fine-grained local guarantees, researchers have explored \emph{problem-dependent} bounds. One prominent approach utilizes smoothness to derive bounds based on the small loss~\cite{ICML19:Zhang-Adaptive-Smooth}, $F_I = \min_{\x \in \X} \sum_{t \in I} f_t(\x)$~\cite{NIPS'10:smooth, aistats'12:exp-concave-smooth}. These bounds can adapt to the environments' difficulty, becoming smaller when $F_I$ is small.
In this paper, we focus on another crucial problem-dependent quantity: the \emph{gradient variation}~\cite{COLT'12:variation-Yang}. For any interval $I = [r, s]$, this quantity is defined as:
\begin{align}
    \label{eq:gradient-variation}
    V_{[r, s]} = \sum_{t = r+1}^{s} \sup_{\x \in \X} \norm{\nabla f_t(\x) - \nabla f_{t-1}(\x)}^2,
\end{align}
and we sometimes use $V_I = V_{[r, s]}$ for notational simplicity. Intuitively, this quantity measures the intensity of change in online functions. We aim to establish the \emph{first} online learning algorithm equipped with a gradient-variation interval regret bound of $\Ot(\sqrt{V_{I}})$ under the smoothness assumption, where $\Ot(\cdot)$ hides the logarithmic factors. In the following, we will demonstrate the importance of the gradient-variation interval regret bound and the challenges of deriving it.

First, gradient variation provides a more refined, local characterization of the algorithm's performance. The proposed $\Ot(\sqrt{V_{I}})$ bound can be more powerful than a single global bound, such as $\O(\sqrt{V_{[1, T]}})$~\cite{COLT'12:variation-Yang}. A global bound might be pessimistic: if a short period of high variation occurs, $V_{[1, T]}$ becomes large, causing the guarantee to degrade to the worst-case bound across the entire horizon. In contrast, our local guarantee reflects the difficulty of each interval. For any given interval $I$, if the environment is benign ($V_{I}$ is small), we achieve a tight, data-dependent bound. If the interval is adversarial, the fact that $V_I = \O(\abs{I})$ ensures our bound gracefully recovers the optimal $\Ot(\sqrt{\abs{I}})$ worst-case guarantee for that specific interval. Second, the gradient-variation bound usually implies the small-loss bound and is thus viewed as a more fundamental measure for problem-dependent guarantees~\cite{ICML'22:Zhang-simple, JMLR'24:Sword++,NeurIPS'24:universal}. Third, gradient variation connects to fundamental problems in machine learning, such as stochastic extended adversarial optimization~\cite{TIT'04:generalization-ability, TIT'12:agarwal-lower-bound-sco, nips'22:sea}. Despite its importance and existing work in other contexts~\cite{nips'23:portfolio-data-dependent,AISTATS'24:bilevel-gradient-variation, arxiv'24:optimistic-jordan, JMLR'24:Sword++}, how to derive versatile gradient-variation interval regret bounds remains unclear, which we address in this work.
\begin{table*}[t]
    \centering
    \caption{A comparison of guarantees for interval regret and interval dynamic regret. ``\# Gradient'' denotes number of gradient queries per round, and ``\# Active Base'' denotes number of active base learners per round. Problem-dependent quantities include the gradient-variation $V_{I}$, and the small-loss value $F_{I} = \min_{\u} \sum_{t\in I} f_t(\u)$. For interval dynamic regret, $F_{I}^{\u} = \sum_{t \in I} f_t(\u_t)$ is the cumulative loss of any comparator sequence $\{\u_t\}_{t\in I}$, and $P_I = \sum_{t=r+1}^s \norm{\u_t - \u_{t-1}}_2$ is the path length of the comparators. }
    \label{tab:comparison}
    \renewcommand{\arraystretch}{1.5}
    \begin{tabular}{cccccc}
      \hline
      \textbf{Measure} & \textbf{Work} & \textbf{Regret Guarantees} & \textbf{\# Gradient} & \textbf{\# Active Base Learners} \\
      \hline
      \multirow{5}{*}{\makecell[c]{Interval \\ Regret}}
        & \cite{ICML'15:Daniely-adaptive} & $\O(\sqrt{{\abs{I}}}\log T)$ & $\O(\log T)$ & $\O(\log T)$ \\
        & \cite{EJS'17:coin-betting-adaptive} & $\O(\sqrt{{\abs{I}\log T}})$ & $\O(\log T)$ & $\O(\log T)$ \\
        & \cite{ICML19:Zhang-Adaptive-Smooth} & $\O(\sqrt{{F_I\log F_I\log F_T}})$ & $\O(\log F_T)$ & $\O(\log F_T)$ \\
        & \cite{NeurIPS'22:efficient} & $\O(\min \{\sqrt{{F_I\log F_I\log F_T}}, \sqrt{\abs{I}\log T}\})$ & $1$ & $\O(\log F_T)$ \\
        & \textbf{Ours (Theorem~\ref{thm:adaptive-oco})} & $\O(\min \{\sqrt{{\min\{V_I, F_I\}\log F_I\log F_T}}, \sqrt{\abs{I}\log T}\})$ & $1$ & $\O(\log F_T)$ \\
      \hline
      \multirow{4}{*}{\makecell[c]{Interval \\ Dynamic\\ Regret}}
        & \cite{AISTATS'20:Zhang} &$\Ot(\sqrt{\abs{I}(1+P_I)})$ & $\O(\log^2 T)$ & $\O(\log^2 T)$ \\
        & \cite{ICML'20:Ashok} & $\Ot(\sqrt{(F^{\u}_I + P_I)(1+P_I)})$ & $1$ & $\O(\log T)$ \\
        & \cite{arXiv'23:efficient-projection} & $\Ot(\sqrt{(F^{\u}_I + P_I)(1+P_I)})$ & $1$ & $\O(\log (\min_{\u} \{F^{\u}_T + P_T\}) \cdot \log T)$ \\
        & \textbf{Ours (Theorem~\ref{thm:interval-dynamic-regret-formal})} & $\Ot(\sqrt{(\min\{V_I, F^{\u}_I\} + P_I)(1+P_I)})$ & $1$ & $\O(\log (\min_{\u} \{F^{\u}_T + P_T^2\}) )$ \\
      \hline
    \end{tabular}
\end{table*}

Using a two-layer online ensemble structure, we propose an adaptive algorithm that simultaneously ensures an interval regret of $\O(\min\{\sqrt{ \min\{V_I, F_I\} \log F_I \log F_T}, \sqrt{\abs{I}\log T}\})$. We refer to this as a ``best-of-three-worlds'' guarantee, as it takes the minimum over three different types of bounds: the gradient-variation bound, the small-loss bound, and the optimal worst-case bound. This is the first gradient-variation interval regret bound in the literature. Table~\ref{tab:comparison} presents a detailed comparison with prior work, showing that our guarantees are the most comprehensive to date and recover previous results. Furthermore, our algorithm is computationally efficient. Despite managing an ensemble of base learners, it requires only a single gradient query per round. It also automatically adjusts the number of active base learners according to the environment. In the worst case, the number of active learners is at most $\O(\log T)$, and can be smaller based on the difficulty of the environments.

While these interval regret results characterize local non-stationarity, our algorithm's capabilities extend further. We investigate its performance against global non-stationarity, typically measured via dynamic benchmarks. A representative measure capturing both aspects is the \emph{interval dynamic regret}~\cite{ICML'20:Ashok, AISTATS'20:Zhang}, defined as:
\begin{align}
    \label{eq:def-interval-dynamic-regret}
    \textsc{ID-Reg}_{[r, s]}(\u_{r:s}) = \sum_{t=r}^s f_t(\x_t) - \sum_{t=r}^s f_t(\u_t),
\end{align}
where $[r,s]\subseteq [T]$ is an arbitrary time interval, and $\u_r, \dots, \u_s \in \X$ is any sequence of comparators. This provides a stronger guarantee by allowing the comparator sequence to change over time. We demonstrate that our algorithm, using the exact same configuration for interval regret minimization, also achieves versatile problem-dependent bounds for interval dynamic regret. These results are included in Table~\ref{tab:comparison} (details in Section~\ref{subsec:interval-dynamic-regret}), highlighting the algorithm's versatility.

In addition to providing strong performance guarantees, our algorithm is also designed to be robust to initial parameter tuning, a common challenge in non-stationary environments where function properties can vary dramatically. To achieve this, we extend the algorithm with Lipschitz and smoothness adaptivity. This means the algorithm does not require prior knowledge of key function parameters, the global Lipschitz constant $G$ or the smoothness constant $L$. Instead, it adapts to these constants on the fly while ensuring comparable guarantees, enhancing its usability for practical deployment.

We now summarize our main technical contributions, which address three key challenges in optimizing the interval regret for non-stationary online learning.
\begin{itemize}[leftmargin=2em]
    \item \textbf{The First Gradient-Variation Interval Regret Bound.}
    We establish the first interval regret bound that scales with gradient variation. Prior analyses for gradient-variation bounds often rely on leveraging the negative terms provided by the algorithm's own stability to cancel the positive terms introduced during the derivation. However, this approach is incompatible with interval regret algorithms, as the dynamic management of the expert ensemble (activating and deactivating base learners) creates an instability that prevents the use of such terms. Our key technical insight is to tackle this issue by carefully leveraging the intrinsic negative Bregman divergence terms provided by the smoothness of the loss functions, which allows us to achieve the desired cancellation and establish the favorable gradient-variation bound.

    \item \textbf{Versatility of the Resulting Guarantee.}
    The resulting gradient-variation bound is highly versatile. First, our analysis reveals the important connections between three different types of guarantees, showing how the gradient-variation bound can imply both the small-loss bound and the minimax-optimal bound, thus achieving a comprehensive result. Second, we demonstrate that under the same configuration, our algorithm's guarantees can imply stronger interval dynamic regret bounds, providing the first versatile problem-dependent bounds for this measure.

    \item \textbf{Parameter Adaptivity in Meta Algorithm Design.}
    To further improve the algorithm's robustness to parameter tuning in non-stationary environments, we extend our method to be agnostic to the Lipschitz and smoothness constants. This is primarily enabled by a new Lipschitz-adaptive meta algorithm, which is achieved by introducing a novel learning rate design. The new regret guarantee and analysis for this Lipschitz-adaptive meta algorithm may be of independent interest.
\end{itemize}

We conduct numerical experiments to validate the effectiveness of our algorithm empirically. First, using synthetic data, we evaluate the algorithm's Lipschitz adaptivity and perform an ablation study to demonstrate its robust performance in standard settings. Second, we construct a hybrid optimization scenario on a real-world dataset that combines adversarial and stochastic components. This experiment further validates the algorithm's effectiveness in adapting to gradient variation.

This paper is organized as follows. Section~\ref{sec:preliminaries} introduces the assumptions and the basic ideas of algorithms with interval regret guarantees. To better convey our ideas, in Section~\ref{sec:approach}, we present the proposed algorithm and the main result under the setting where both the Lipschitz and smoothness constants are known. Section~\ref{sec:lipschitz_adaptive} extends the algorithm to achieve Lipschitz and smoothness adaptivity. Section~\ref{sec:implications} shows the implications and applications of our results. Section~\ref{sec:exp} presents the experimental results. Section~\ref{sec:conclusions} concludes the paper. All the detailed proofs are deferred to the Appendices.

\section{Related Work}
\subsection{Non-Stationary Online Learning}
The concept of guaranteeing performance over arbitrary local intervals originates from early work on prediction with expert advice~\cite{journals/iandc/LittlestoneW94}. The notion of \emph{adaptive regret} is first formalized in the OCO framework by~\cite{journal'07:Hazan-adaptive}, which considers the goal of minimizing $\max_{I \subseteq [T]} \AReg_{I}$. For general convex loss functions,~\cite{journal'07:Hazan-adaptive} obtains a guarantee of $\max_{I \subseteq [T]} \AReg_{I} \leq \widetilde{\mathcal{O}}(\sqrt{T})$. A key characteristic of this definition is that the regret bound does not scale with the interval length $|I|$. We note, however, that~\cite{journal'07:Hazan-adaptive} does achieve guarantees that scale with $|I|$ for the strongly convex and exp-concave functions. Subsequently,~\cite{ALT'12:closer-adaptive-regret} observes that the method of~\cite{journal'07:Hazan-adaptive} is equivalent to the fixed-share algorithm~\cite{conf/nips/Cesa-BianchiGLS12} when the loss functions exhibit favorable curvature properties.

The notion of \emph{strongly adaptive regret}, which we refer to as \emph{interval regret} in this paper, requires a more fine-grained local guarantee that scales with the interval length~\cite{ICML'15:Daniely-adaptive}. This stronger notion is formalized by~\cite{ICML'15:Daniely-adaptive}, which obtains an $\O(\sqrt{\abs{I}} \log T)$ guarantee. This bound is later improved to $\O(\sqrt{|I| \log T})$ by~\cite{EJS'17:coin-betting-adaptive} using a coin-betting meta-algorithm~\cite{NIPS'16:coin-bet-OrabonaP16}. To handle more favorable environments,~\cite{ICML19:Zhang-Adaptive-Smooth} proposes a strategy that adapts the number of base learners to the small-loss value, achieving the first problem-dependent bound of $\O(\sqrt{F_I\log F_T \log F_I})$. However, their method requires multiple gradient queries per round and does not imply the optimal worst-case rate. More recently,~\cite{NeurIPS'22:efficient} proposes a more efficient method using a single gradient query to achieve a unified bound of $\O( \min \{\sqrt{F_I\log F_T \log F_I}, \sqrt{\abs{I}\log T} \})$.

Other measures also exist to characterize the non-stationarity of the environments~\cite{ICML'03:zinkvich,OR'15:dynamic-function-VT, AISTATS'15:dynamic-optimistic}. Among these, dynamic regret~\cite{ICML'03:zinkvich, TIT'12:andras,ICML'16:GyorgyS-shiftregret, NIPS'18:Zhang-Ader,ICML'22:TV-game, JMLR'24:Sword++, ICML'26:Xie-dynamic-regret} is a popular measure of robustness, benchmarking the performance against changing comparators. The general relationship between interval regret and dynamic regret remains unclear~\cite{ICML'20:Ashok}. Some works therefore consider minimizing both simultaneously~\cite{ICML'20:Ashok,AISTATS'20:Zhang, NeurIPS'22:efficient} by introducing the stronger interval dynamic regret. In Section~\ref{sec:implications}, we present a detailed introduction to interval dynamic regret and demonstrate that our interval regret guarantee can imply a dynamic regret guarantee under specific algorithmic configurations.

\subsection{Problem-Dependent Regret Bound}
In the OCO framework, for convex loss functions, the results in~\cite{ICML'03:zinkvich,conf/colt/AbernethyBRT08} establish a minimax static regret bound of $\Theta(\sqrt{T})$. Problem-dependent bounds offer a more refined analysis by exploiting the properties of the environments rather than focusing on the worst-case dependence on $T$. One well-known example is the small-loss bound, which relates regret to the cumulative loss of the best fixed comparator $F_T = \min_{\x} \sum_{t=1}^T f_t(\x)$~\cite{NIPS'10:smooth,aistats'12:exp-concave-smooth, TCS'18:SOGD}.

Another significant problem-dependent quantity is the gradient variation proposed by~\cite{COLT'12:variation-Yang}. The study of gradient-variation adaptivity has been revived since its introduction to dynamic regret minimization~\cite{NIPS'20:sword,JMLR'24:Sword++}, which has promoted dedicated algorithmic designs and new analytical techniques, and has further stimulated a growing line of research in various settings~\cite{nips'23:portfolio-data-dependent,NeurIPS'23:universal,AISTATS'24:bilevel-gradient-variation, NeurIPS'24:LocalSmooth, NeurIPS'24:universal, ICML'26:Yu-BCO-GV, COLT'26:Zhao-GV-unconstrained}. Gradient variation is recognized as a more fundamental quantity, as guarantees based on gradient variation can imply small-loss guarantees~\cite{JMLR'24:Sword++}. The importance of gradient variation also extends across online learning: it provides a path to fast convergence rates in multi-player games~\cite{ NIPS'15:fast-rate-game,ICML'22:TVgame}, helps bridge the gap between stochastic and adversarial optimization, and links online and offline accelerated methods~\cite{NIPS'19:UnixGrad, ICML'20:Bregman-divergence-negative-term,NeurIPS'25:GV-holder}.
Motivated by these insights, our work focuses on this quantity, and we present the first interval regret bound in terms of gradient variation, which can imply the small-loss bound as well.

\subsection{Lipschitz-Adaptive Algorithm}
Online learning is inherently challenging as the learner must continuously adapt to newly revealed loss functions over time. Many existing algorithms require prior knowledge of key environmental parameters, such as the global Lipschitz constant $G$ and the diameter of the feasible domain $D$, as input to mitigate uncertainty in difficult environments. In the context of online learning, algorithms that do not require the domain diameter $D$ are often referred to as parameter-free~\cite{NIPS'16:coin-bet-OrabonaP16,COLT'18:black-box-reduction, ALT'22:implicit-parameter-free-truncated}. Algorithms that do not require knowledge of $G$ are known as Lipschitz-adaptive~\cite{TCS'18:SOGD, COLT'19:ashok-cubic}. Some recent work has studied the setting where both $G$ and $D$ are unknown, but the optimality of the results is preliminary~\cite{COLT'19:Lipschitz-MetaGrad, COLT'19:ashok-cubic,COLT'22:parameter-free-omd,icml'23:unbounded,COLT'26:Zhao-GV-unconstrained}. In this paper, we focus on the setting where $D$ is known, but $G$ is unknown.

We define the Lipschitz constant as $G = \max_{t \in [T], \x\in \X} \|\nabla f_t(\x)\|_2$. In non-stationary environments, accurately estimating this value is often infeasible, and a misspecified $G$ may degrade performance~\cite{TCS'18:SOGD, COLT'19:ashok-cubic, COLT'19:Lipschitz-MetaGrad}. Prior literature has explored Lipschitz adaptivity in various settings, including unbounded domains~\cite{COLT'19:ashok-cubic, COLT'20:parameter-free-cubic,COLT'21:impossible-tuning}, and alongside adaptation to function curvature~\cite{COLT'19:Lipschitz-MetaGrad,TIT'23:wu-regret-logloss, NeurIPS'24:LocalSmooth, NeurIPS'25:parameter-free-sea}. Nevertheless, a gap remains in understanding how to achieve Lipschitz adaptivity specifically for non-stationary environments. This work addresses this gap by developing a Lipschitz-adaptive algorithm that provides gradient-variation guarantees, thereby improving its parameter tuning robustness.

\section{Preliminaries}
\label{sec:preliminaries}
In this section, we present the assumptions and describe the two-layer online ensemble framework used to design our algorithm for non-stationary environments.
\subsection{Assumptions}
\label{subsec:assumptions}
We make the following assumptions, which are standard in the OCO literature~\cite{NIPS'10:smooth, book'12:Shai-OCO, book'16:Hazan-OCO}.

\begin{myAssump}[bounded domain]
    \label{ass:bounded-domain}
    The feasible domain $\X \subseteq \R^d$ is convex and bounded, with a diameter $D$ known to the learner, i.e., $\norm{\x - \y} \leq D$ for all $\x, \y \in \X$.
\end{myAssump}

\begin{myAssump}[bounded gradient]
    \label{ass:Lipschitzness}
    The gradient norm of each loss function $f_t$ is bounded by $G$, i.e., $\norm{\nabla f_t(\x)}_2 \leq G < \infty$ for all $\x \in \X$.
\end{myAssump}

\begin{myAssump}[bounded function value]
    \label{ass:bounded-function-value}
    The function value of $f_t$ is bounded, i.e., $f_t(\x) \leq GD$ for $\x \in \X, t \in [T]$.
\end{myAssump}

\begin{myAssump}[smoothness and non-negativity]
    \label{ass:smoothness}
    Each $f_t:\X \mapsto \R_+$ is non-negative and $L$-smooth for some $L < \infty$, i.e., $\norm{\nabla f_t(\x) - \nabla f_t(\y)}_2 \leq L \norm{\x - \y}_2$ for all $\x, \y \in \X$.
\end{myAssump}

Throughout this paper, the domain diameter $D$ is known and provided as input to the algorithm. In contrast, the Lipschitz constant $G$ and the smoothness constant $L$ are not assumed to be known to the algorithm unless explicitly stated otherwise in a theorem or lemma.

\subsection{Basic Idea: Two-Layer Online Ensemble Framework}
\label{subsec:online-ensemble-framework}
The non-stationarity in interval regret, defined in~Eq.~\eqref{eq:strongly-adaptive-regret}, primarily arises from the uncertainty about which intervals an algorithm may perform poorly on. To tackle this challenge, a fundamental design principle is to construct an ensemble of base learners and combine them appropriately~\cite{JMLR'24:Sword++}, leveraging the ensemble to mitigate non-stationarity. Specifically, the algorithmic design consists of three components: base learners, a meta learner, and a schedule. Each base learner is an expert specialized for making predictions over a specific time interval $I_i \subseteq [T]$~\cite{journal'07:Hazan-adaptive,ICML'15:Daniely-adaptive}. A meta learner is then employed to aggregate the decisions $\x_{t,i} \in \X$ from each base learner using an online-updated weight $p_{t,i} \in [0, 1]$. Finally, the schedule efficiently manages the ensemble by determining which base learners are active at any given time. Concretely, the schedule is defined as a set of intervals that partitions the time horizon $[1, T]$, with each interval specifying the active period of a base learner. Fig.~\ref{fig:interval} visualizes a classical schedule defined as follows~\cite{journal'07:Hazan-adaptive, ICML19:Zhang-Adaptive-Smooth},
\begin{align}
    \label{eq:problem-independent-schedule}
    \mathcal{S} = \{[i\cdot 2^k, (i+1) \cdot 2^k - 1]: k \in\mathbb{N} , i \text{ is odd}\}.
\end{align}

At the beginning of each interval within this schedule $\mathcal{S}$, a new base learner is initialized and remains active for the duration of that interval, after which it is deactivated. When using this schedule, $\O(\log T)$ base learners must be maintained per round. This schedule is considered problem-independent, as its construction is determined purely by the time horizon rather than any properties of the environments.

To improve scheduling efficiency,~\cite{ICML19:Zhang-Adaptive-Smooth} proposes a problem-dependent schedule that activates base learners based on the algorithm's observed performance. This approach is refined by~\cite{NeurIPS'22:efficient}, which reduces the information required for the schedule's construction. The key idea is to monitor the cumulative loss of the aggregated prediction, $\sum_{t \in I} f_t(\x_t)$. When this loss exceeds a time-varying threshold, the timestamp is recorded as a marker, with $s_m \in [T]$ denoting the $m$-th marker.

This mechanism leads to the problem-dependent schedule shown in Fig.~\ref{fig:interval:CPGC}, defined as:
\begin{align}
    \label{eq:problem-dependent-schedule}
    \tilde{\mathcal{S}} = \{[s_{i\cdot 2^k}, s_{(i+1) \cdot 2^k} - 1]: k \in\mathbb{N} , i \text{ is odd}\}.
\end{align}
While the previous problem-independent schedule $\mathcal{S}$ partitions the time horizon based on a fixed number of timestamps (e.g., dyadic intervals), this schedule $\tilde{\mathcal{S}}$ creates partitions based on a fixed number of markers. For instance, an interval from the $k=1$ level in Fig.~\ref{fig:interval:CPGC} (e.g., $[s_2, s_4-1]$) spans a variable number of timestamps but is defined by $2^1=2$ markers, whereas a $k=1$ interval in Fig.~\ref{fig:interval} (e.g., $[2, 3]$) always spans a fixed $2^1=2$ timestamps. The number of markers per interval grows geometrically, creating an efficient schedule that requires maintaining at most $\mathcal{O}(\log F_T)$ base learners per round. We remark that this construction is performed on the fly, and the algorithm only needs the indices of the markers to manage the ensemble without knowing $\tilde{\mathcal{S}}$ in advance.

We now illustrate the core technique for analyzing interval regret within the online ensemble framework. The key insight is that any arbitrary interval $[r, s]$ can be partitioned into a set of smaller intervals $\{I_1, \dots, I_k\} \subseteq \mathcal{S}$ from the schedule $\mathcal{S}$.
For each interval $I_j, j \in [k]$, a single base learner (indexed by $i_j$) is active throughout. This structure allows the regret over $[r, s]$ to be decomposed as follows:
\begin{align}
    &\sum_{j \in [k]} \sum_{t \in I_j} f_t(\x_t) - f_t(\u) \leq \sum_{j \in [k]} \sum_{t \in I_j} \inner{\nabla f_t(\x_t)}{\x_t - \u} \notag \\
    &= \underbrace{\sum_{j \in [k]} \sum_{t \in I_j} \inner{\nabla f_t(\x_t)}{\x_t - \x_{t, i_j}}}_{\textsc{meta regret}} \notag \\
    &\quad + \underbrace{\sum_{j \in [k]} \sum_{t \in I_j}\inner{\nabla f_t(\x_t)}{\x_{t, i_j} - \u}}_{\textsc{base regret}}. \label{eq:regret-decomposition}
\end{align}
The first term, the meta regret, measures how well the meta learner tracks the best active base learner in each sub-interval. The second term, the base regret, measures the performance of the base learners. While analyzing the base regret is often straightforward, bounding the meta regret is typically more challenging. This is because it requires a sophisticated analysis connecting both the scheduling mechanism and the performance of the base learners.

\begin{figure*}[t]
    \centering
    \begin{tabular}{@{}c@{\hspace{0.2ex}}*{15}{@{\hspace{0.1ex}}c}@{\hspace{0.2ex}}c@{}}
    Length \textbackslash\  Rounds & 1 & 2 & 3 & 4 & 5 & 6 & 7 & 8 & 9 & 10 & 11 & 12 & 13 & 14 & 15 & $\cdots$ \\
    $2^0$ timestamps & [\quad ] &  &  [\quad ] &  & [\quad ] &  & [\quad ] & & [\quad ] &  & [\quad ] & & [\quad ] &  & [\quad ] & $\cdots$   \\
    $2^1$ timestamps&  & [\quad  \phantom{]}& \phantom{[}\quad ] &  &  & [\quad \phantom{]} & \phantom{[}\quad ] &  & & [\quad \phantom{]} & \phantom{[}\quad ] & & & [\quad \phantom{]} & \phantom{[}\quad ] & $\cdots$   \\
    $2^2$ timestamps  &  & &  & [\quad \phantom{]} & \phantom{[}\quad \phantom{]} & \phantom{[}\quad \phantom{]} & \phantom{[}\quad ] &  & & & & [\quad \phantom{]} & \phantom{[}\quad \phantom{]} & \phantom{[}\quad \phantom{]} & \phantom{[}\quad ] & $\cdots$   \\
    $2^3$ timestamps  &  & &  &  &  &  &  & [\quad \phantom{]} & \phantom{[}\quad \phantom{]} & \phantom{[}\quad \phantom{]} & \phantom{[}\quad \phantom{]} & \phantom{[}\quad \phantom{]} & \phantom{[}\quad \phantom{]} & \phantom{[}\quad \phantom{]} & \phantom{[}\quad ] & $\cdots$
    \end{tabular}
    \caption{Problem-Independent Schedule~\cite[Figure 2]{ICML19:Zhang-Adaptive-Smooth}. Each interval is denoted by $[ \quad ]$.}
    \label{fig:interval}
\end{figure*}

\begin{figure*}[t]
    \centering
    \begin{tabular}{@{}c@{\hspace{0.2ex}}*{15}{@{\hspace{0.1ex}}c}@{\hspace{0.2ex}}c@{}}
     Length \textbackslash\  Rounds & 1 & 2 & 3 & 4 & 5 & 6 & 7 & 8 & 9 & 10 &11 &12 & 13 & 14 &15 & $\cdots$ \\
     Markers & $s_1$ &  & $s_2$ & $s_3$ &  &  & $s_4$ & $s_5$ &   & $s_6$ &$s_7$ &$s_8$ &  & $s_9$ &$s_{10}$ & $\cdots$ \\
    $2^0$ markers & [\quad \phantom{]} & \phantom{[}\quad ] &  & [\quad \phantom{]} & \phantom{[}\quad \phantom{]} & \phantom{[}\quad ] &  & [\quad \phantom{]} & \phantom{[}\quad ] &   &[\quad ] & &  & [\quad ] & & $\cdots$   \\
     $2^1$ markers  &  &  &  [\quad \phantom{]} & \phantom{[}\quad \phantom{]} & \phantom{[}\quad \phantom{]} & \phantom{[}\quad ] &  &  &  &  [\quad \phantom{]} &\phantom{[}\quad ] &&  & & [\quad \phantom{]} & $\cdots$   \\
      $2^2$ markers &  &  &   &     &     &  & [\quad \phantom{]} & \phantom{[}\quad \phantom{]} & \phantom{[}\quad \phantom{]} &  \phantom{[}\quad \phantom{]} &\phantom{[}\quad ] & &  & & & $\cdots$   \\
     $2^3$ markers &  &  &      &     &     &  & &  &  &   &&[\quad \phantom{]} & \phantom{[}\quad \phantom{]} & \phantom{[}\quad \phantom{]} & \phantom{[}\quad \phantom{]} & $\cdots$
    \end{tabular}
    \caption{Problem-dependent Schedule~\cite[Figure 4]{ICML19:Zhang-Adaptive-Smooth}. Each interval is denoted by $[ \quad ]$.}
    \label{fig:interval:CPGC}
\end{figure*}

\section{Our Approach: Gradient-Variation Interval Regret Bound}
\label{sec:approach}
In this section, we first introduce the proposed algorithm based on the two-layer ensemble framework and then provide its key theoretical analysis. For clarity of presentation, we assume in this section that the Lipschitz constant $G$ and the smoothness constant $L$ are known to the learner.
\subsection{Basic Components}
\label{subsec:basic-components}
Based on the two-layer ensemble framework, our design requires three key components to achieve the interval regret bound: a base algorithm, a meta algorithm, and a schedule. We describe these components in detail below and then integrate them into the final algorithm. All three components are designed to be adaptive to the environments, and through their coordinated interaction, the overall algorithm achieves the desired guarantees.

\paragraph{Base Algorithm} Each base algorithm is responsible for the online learning task on a specific interval, which allows it to predict without handling the non-stationarity of the entire time horizon. Intuitively, each base learner can be viewed as capturing a certain level of local stability in the environments. For our design, we adopt the optimistic online gradient descent algorithm~\cite{COLT'12:variation-Yang, conf/colt/RakhlinS13}, a two-step update method formalized as:
\begin{align}
    \label{eq:base-algorithm-omd}
    \x_{t,i} = \Pi_{\X}\left[\xh_{t,i} - \eta_{t,i} M_{t} \right], \  \xh_{t+1,i} = \Pi_{\X}\left[\xh_{t,i} - \eta_{t,i}\g_{t}\right],
\end{align}
where $\{\xh_{t,i}\}_{t=1}^T$ is an intermediate sequence used to produce the submitted decision $\x_{t,i} \in \X$, $\g_{t} \in \R^d$ is the gradient of the function to be minimized~(which may be a surrogate loss, hence the general formulation), and $M_{t} \in \R^d$ can be viewed as a ``guess'' of the gradient, also known as the ``optimism''. The principle behind optimistic gradient descent is that if the optimism $M_t$ closely approximates the true gradient $\g_t$, the algorithm achieves stronger theoretical guarantees. Therefore, by selecting a proper optimism, we can leverage favorable properties of the environments to obtain problem-dependent regret at the base-learner level.

With proper step-size tuning, the base learner achieves a regret bound of $\O(\sqrt{\sum_{\tau = \text{start}_i}^t \norm{\g_\tau - M_\tau}_2^2})$ over an interval $[\text{start}_i, t]$, where $\text{start}_i$ denotes the initial timestamp of the $i$-th base learner. In Algorithm~\ref{alg:adaptive-oco}, we choose $\g_t = \nabla f_t(\x_t)$ and $M_t = \nabla f_{t-1}(\x_{t-1})$. This choice provides a quantity at the base-learner level that approximates the gradient variation, which we refer to as the \emph{empirical gradient variation}:
\begin{align}
  \label{eq:empirical-gradient-variation}
  \bar{V}_{[a,b]} = \sum_{t=a+1}^b \norm{\nabla f_t(\x_t) - \nabla f_{t-1}(\x_{t-1})}_2^2.
\end{align}
This result requires further analysis to recover the actual gradient variation. We will elaborate on this in the subsequent sections.

\paragraph{Meta Algorithm} The meta algorithm is responsible for evaluating the performance of base learners over different intervals and adjusting their aggregation weights accordingly. The main challenge is to ensure the meta algorithm remains adaptive to changing environments while managing a dynamic set of base learners. Formally, the meta algorithm produces the weights $\p_t \in \Delta_{\abs{A_t}}$ to aggregate the decisions, where $A_t$ denotes the set of active base learners at time $t$. We require our meta algorithm to support the ``sleeping expert'' paradigm~\cite{STOC'97:sleeping-expert,COLT'15:Luo-AdaNormalHedge}, which allows for a dynamically changing set of active experts. This paradigm naturally fits our two-layer ensemble framework for designing interval regret minimization algorithms. Furthermore, to derive gradient-variation bounds, we require a meta algorithm that incorporates optimism. We employ Optimistic-Adapt-ML-Prod~\cite{NIPS'16:Wei-non-stationary-expert}, which uses the following update rule:
\begin{align}
  \label{eq:meta-algorithm-update-rule-known-G}
  p_{t,i} &= \frac{\epsilon_{t,i}w_{t,i}\exp(\epsilon_{t,i}m_{t,i})}{\sum_{j \in A_t} \epsilon_{t,j}w_{t,j}\exp(\epsilon_{t,j}m_{t,j})}, \\
  w_{t+1,i} &= \left(w_{t,i}\exp\left(\epsilon_{t, i}r_{t,i} - \epsilon_{t,i}^2(r_{t,i} - m_{t,i})^2\right)\right)^{\frac{\epsilon_{t+1,i}}{\epsilon_{t,i}}}. \notag
\end{align}
In the rule above, $w_{t,i}$ are the weight parameters, $m_{t,i}$ is the optimism, and $\epsilon_{t,i}$ is the learning rate. We denote the instantaneous regret of the $i$-th base learner by $r_{t,i} = \inner{\p_t}{\ellb_t} - \ell_{t,i}$, where $\ell_{t,i}$ is the surrogate loss for that learner.

With appropriate tuning, Optimistic-Adapt-ML-Prod ensures a regret bound of $\Ot\big(\sqrt{\sum_{t = \text{start}_i}^t \norm{\r_t - \m_t}_{\infty}^2}\big)$ for any base learner $i$. When incorporating this meta algorithm into Algorithm~\ref{alg:adaptive-oco}, we define the surrogate loss as $\ell_{t,i} = \inner{\nabla f_t(\x_t)}{\x_{t,i}}$, which implies $r_{t,i} = \inner{\nabla f_t(\x_t)}{\x_t - \x_{t,i}}$. We set the optimism as $m_{t,i} = \inner{\nabla f_{t-1}(\x_{t-1})}{\x_t - \x_{t,i}}$. Under this setup, the meta-regret can be bounded by the empirical gradient variation:
\begin{align*}
  & \Ot\left(\sqrt{\sum_{t = \text{start}_i}^t \max_{i \in A_t} \inner{\nabla f_t(\x_t) - \nabla f_{t-1}(\x_{t-1})}{\x_t - \x_{t,i}}^2 }\right)\\ &\leq \Ot\left(\sqrt{D^2\sum_{t = \text{start}_i}^t \norm{\nabla f_t(\x_t) - \nabla f_{t-1}(\x_{t-1})}_2^2}\right).
\end{align*}
Similar to the analysis for the base learner, this bounds the meta regret by the empirical gradient variation. The key remaining part of the analysis, which we present later, is to convert this empirical quantity into the true gradient variation.

Finally, we address a subtle technical point. The definition of the optimism $m_{t,i}$ depends on the aggregated decision $\x_t = \sum_{i \in A_t} p_{t,i}\x_{t,i}$, which in turn depends on the weight vector $\p_t$. However, the update for $\p_t$ itself uses the optimism $m_{t,i}$, creating a circular dependency. This issue can be efficiently resolved using a fixed-point iteration procedure, which converges in an additional $\mathcal{O}(\log T)$ iterations per round, as detailed in~\cite[Section~3.3]{NIPS'16:Wei-non-stationary-expert}. We note that this bisection process can be implemented while still requiring only a single query of the gradient in each iteration.

\paragraph{Schedule} We adopt the problem-dependent schedule from~\cite{NeurIPS'22:efficient}. In addition to efficiently managing the pool of base learners, the schedule also needs to dynamically adjust the number of active base learners based on the difficulty of the environments over time. As described in Section~\ref{subsec:online-ensemble-framework}, the schedule is constructed using markers, which are generated whenever the cumulative loss exceeds a sequence of data-dependent and time-varying thresholds.

Here, we introduce the threshold generation function, $\mathcal{G}(t, i)$, which creates these thresholds on the fly. Its inputs are the current time $t$ and the marker's index $i$. This function generates the sequence of markers $s_1, s_2, \dots$ used to build the schedule via Eq.~\eqref{eq:problem-dependent-schedule}.

Formally, we first define the term ${\Gamma}_{t}$ as:
\begin{align}
      \label{eq:first-gamma-def}
      {\Gamma}_{t} = \ln \left( 1 + \frac{1}{e} \left(
        1
        + \frac{1}{2} \ln (1 + t)
      \right) \right),
\end{align}
which is bounded by $\O(\log\log T)$ and treated as a constant in our analysis, following the convention of~\cite{COLT'14:second-order-Hedge,COLT'15:Luo-AdaNormalHedge}. The threshold generation function is then defined as follows:
\begin{align}
  \label{eq:threshold-function-known-G}
  &\mathcal{G}(t, i) = 56LD^2\left(\frac{3 \ln (2i + 1) + {\Gamma}_{t}}{\sqrt{\ln (2 i + 1)}} + \frac{5}{2}\right)^2 + 5D \notag  \\
  &\quad + 2D\frac{3 \ln (2i + 1) + {\Gamma}_{t}}{\sqrt{\ln (2 i + 1)}} + 9({\Gamma}_{t} + \ln (2 i + 1))GD ,
\end{align}
In this definition, for any time $t$, $s_i$ is the largest marker such that $s_i \leq t < s_{i+1}$.

\subsection{Overall Algorithm}
Algorithm~\ref{alg:adaptive-oco} presents our novel interval regret minimization algorithm, named \underbar{G}radient-v\underbar{A}riation \underbar{I}nterval \underbar{R}egret Minimization~(GAIR). The algorithm uses optimistic online gradient descent (Eq.~\eqref{eq:base-algorithm-omd}) as the base learners and Optimistic-Adapt-ML-Prod as the meta learner. In Lines~\ref{line:schedule}~--~\ref{line:schedule-end}, GAIR employs a scheduling mechanism inspired by~\cite{NeurIPS'22:efficient}, which initializes new base learners based on the observed cumulative loss of the final decision $\x_t$.

We now explain in more detail how the problem-dependent schedule from Eq.~\eqref{eq:problem-dependent-schedule} is implemented. Let $N_t$ be the number of markers registered up to the current time $t$; this value also corresponds to the total number of base learners initialized so far. When the cumulative loss exceeds a time-varying threshold, a new marker is recorded as $s_{N_t} = t$, and the corresponding $N_t$-th base learner is then initialized. Based on the index $N_t$ and the structure of $\tilde{\mathcal{S}}$, a unique deactivation marker $s_j$ can be determined, which in turn defines this learner's active interval as $[s_{N_t}, s_j - 1]$. The algorithm does not need to know the future value of $s_j$ at time $t$; it only needs to track the sequence of markers to deactivate learners at precisely the correct time (Line~\ref{line:ref-to-schedule}). The following theorem then summarizes the algorithm's required parameter tuning and its main theoretical guarantees.
\begin{myThm}
  \label{thm:adaptive-oco}
  Under Assumptions~\ref{ass:bounded-domain}--\ref{ass:smoothness}, additionally assuming that the Lipschitz constant $G$ and the smoothness constant $L$ are known, setting the step size of OMD in Eq.~\eqref{eq:base-algorithm-omd} as $\eta_{t,i} = (2D)/\sqrt{1 + \sum_{\tau = \text{start}_i}^t \norm{\nabla f_\tau(\x_\tau) - \nabla f_{\tau - 1}(\x_{\tau - 1})}_2^2}$, tuning the learning rate of the meta algorithm in Eq.~\eqref{eq:meta-algorithm-update-rule-known-G} as
  $\epsilon_{t,i} = \min \{\frac{1}{4GD}, \sqrt{\frac{2 \ln (1 + 2 i)} { (2GD)^2 + \sum_{s=\text{start}_i}^t (r_{s, i} - m_{s, i})^2}}\},$ setting the threshold function as specialized in Eq.~\eqref{eq:threshold-function-known-G}, then for any $\u \in \X$ and any $I = [r,s ]\subseteq [T]$, GAIR achieves the following interval regret bound:
  \begin{align*}
     \O\left(\min \left\{ \sqrt{\min\{V_{I}, F_{I}\}\log F_{I} \log F_{[1,s]}}, \sqrt{\abs{I}\log T} \right\}\right),
  \end{align*}
  where $F_{[a, b]} = \min_{\x \in \X} \sum_{t=a}^b f_t(\x)$ denotes the small loss, and $V_{[a,b]} = \sum_{t=a+1}^b \sup_{\x \in \X}\norm{\nabla f_t(\x) - \nabla f_{t-1}(\x)}_2^2$ represents the gradient variation over interval $[a,b]$.
\end{myThm}

Notably, our algorithm achieves the first gradient-variation interval regret bounds. With the same configuration, it also implies both the small-loss and worst-case bounds, matching state-of-the-art results~\cite{EJS'17:coin-betting-adaptive,ICML19:Zhang-Adaptive-Smooth,NeurIPS'22:efficient} (see Table~\ref{tab:comparison}) and demonstrating the versatility of our approach. Beyond its theoretical guarantees, our algorithm is also computationally efficient, requiring only a single gradient query per round and maintaining at most $\mathcal{O}(\log F_{[1, t]})$ base learners at time~$t$, as established in Lemma~\ref{lemma:counting-number-barv}. Furthermore, through a white-box reduction, our main result directly implies a nearly optimal guarantee for interval dynamic regret~\cite{ICML'20:Ashok, AISTATS'20:Zhang}, a strictly stronger performance measure. This result, stated in Theorem~\ref{thm:interval-dynamic-regret}, is the first gradient-variation bound for interval dynamic regret to our knowledge. Further details are provided in Section~\ref{subsec:interval-dynamic-regret}.

In Theorem~\ref{thm:adaptive-oco}, our algorithm achieves a bound that is the minimum of three different types of bounds, each matching the best-known rates: the gradient-variation bound, the small-loss bound, and the optimal worst-case bound. Consequently, we also refer to it as a ``best-of-three-worlds'' result. We now highlight two technical connections that enable this result. First, we show how the empirical gradient variation, defined in~\eqref{eq:empirical-gradient-variation}, a quantity that naturally arises in our analysis, implies the small-loss bound. The key step is to bound this empirical quantity by a sum of squared gradient norms:
\begin{align*}
  \sum_{t\in I}\norm{\nabla f_t(\x_t) - \nabla f_{t-1}(\x_{t-1})}_2^2 \leq \O\Big(\sum_{t\in I} \norm{\nabla f_t(\x_t)}_2^2\Big).
\end{align*}
Then, by applying the self-bounding property of smooth functions, i.e., $\norm{\nabla f_t(\x_t)}_2^2 \leq 2Lf_t(\x_t)$, and following standard analytical steps, we can derive the small-loss bound.
Second, we demonstrate that this small-loss bound can be used to tighten the worst-case bound. The key insight lies in our problem-dependent schedule, which uses cumulative loss to construct a sequence of geometric intervals. This construction can avoid the use of the Cauchy-Schwarz inequality, allowing us to instead leverage a geometric summation to eliminate a logarithmic factor for a tighter final bound.
\begin{myRemark}
A more desirable bound would take the form
$$\mathcal{O}\left(\min\{\sqrt{V_I \log V_I \log V_T}, \sqrt{F_I \log F_I \log F_{T}}, \sqrt{\abs{I}\log T}\}\right),$$
where the logarithmic factors also adapt to the corresponding primary term ($V_I$, $F_I$, or $\abs{I}$). This is motivated by the observation that in some cases, gradient variation can be significantly smaller than the small-loss value, and vice versa~\cite{JMLR'24:Sword++}. This raises an open question: can we fully decouple the gradient variation and small-loss quantities to achieve such a bound?

This remains a technical challenge due to the design of our problem-dependent schedule. The schedule's mechanism is based on monitoring the cumulative loss, $\sum_{t \in I} f_t(\x_t)$, a quantity that is crucial for our analysis and for tightening the worst-case bound.
Because the schedule is tied to the cumulative loss, its performance is naturally coupled with the small-loss quantity. This coupling is the reason the logarithmic factors in our final bound depend on $F_I$. The core difficulty is that there is no straightforward way to relate the cumulative loss to the gradient variation; doing so would require establishing a tight, general connection, such as $V_I = \Theta(F_I)$. We leave this problem as an open question for future research.
\end{myRemark}

\begin{algorithm}[!t]
  \caption{GAIR}
  \label{alg:adaptive-oco}
  \begin{algorithmic}[1]
    \REQUIRE Lipschitz constant $G$, diameter $D$, and the threshold function $\mathcal{G}:\mathbb{N} \mapsto \R_+$.
    \STATE \textbf{Initialize:} active learners' index set $A_0 = \emptyset$; the number of initialized base learners $N_0 = 0$; cumulative loss $L_0 = \infty$, the zeroth marker $s_0 = 0$, and by default $\ellb_0 = \boldsymbol{0}$.
    \FOR{$t=1$ {\bfseries to} $T$}
    \IF{$L_{t-1} > \mathcal{G}(t-1, N_{t-1})$} \label{line:schedule}
        \STATE Update $N_t = N_{t-1} + 1$, register a marker $s_{N_t} = t$, $L_t = 0$, and according to Eq.~\eqref{eq:problem-dependent-schedule}, remove base learners $A_t = A_{t-1} \setminus \{i: [s_i, s_{N_t} - 1] \in \tilde{\mathcal{S}}\}$; \label{line:ref-to-schedule}
        \STATE Initialize the $N_t$-th new base learner according to~Eq.~\eqref{eq:base-algorithm-omd}, whose starting time is $s_{N_t}$, update the index set $A_t = A_t \cup \{N_t\}$;
    \ELSE
        \STATE $A_t = A_{t-1}$, $N_t = N_{t-1}$;
    \ENDIF \label{line:schedule-end}
    \STATE Obtain $\x_{t,i}$ from each base learner $i \in A_t$, send $m_{t,i} = \inner{\nabla f_{t-1}(\x_{t-1})}{\x_t - \x_{t,i}} $ and $A_t$ to the meta learner;
    \STATE Obtain $\p_t \in \Delta_{\abs{A_t}}$ from the meta learner, submit $\x_t = \sum_{i \in A_t} p_{t, i}\x_{t,i}$;
    \STATE Obtain $f_t(\x_t)$ and $\nabla f_t(\x_t)$, update $L_{t+1} = L_t + f_t(\x_t)$, set $\ell_{t,i} = \inner{\nabla f_t(\x_t)}{\x_{t,i}}$ for the meta learner, and set $g_{t} = M_{t+1} = \nabla f_t(\x_t)$ for the base learners. \label{line:oco-update}
    \ENDFOR
  \end{algorithmic}
\end{algorithm}

\subsection{Key Analysis of GAIR}
\label{subsec:key-analysis}
In this section, we begin by revisiting how to obtain gradient-variation bounds when minimizing static regret, in order to convey the core ideas developed in prior work. We then discuss the challenges that arise when extending these techniques to optimize non-stationary performance measures. Finally, we present our proposed solution for establishing the interval gradient-variation regret bound.

The core idea for achieving gradient-variation bounds is to leverage algorithmic stability. Technically, algorithmic stability manifests as negative terms in the regret bound that involve the distance between successive decisions. We use optimistic online gradient descent (Eq.~\eqref{eq:base-algorithm-omd}) as an example. If a base learner sets the optimism $M_t = \nabla f_{t-1}(\x_{t-1,i})$ and uses the gradient $\g_t = \nabla f_t(\x_{t,i})$ (that is, it updates using gradients evaluated at its own decisions rather than the aggregated decision), then the regret $\sum_{t=r}^s \inner{\nabla f_t(\x_t)}{\x_t - \u}$ can be bounded as follows:
\begin{align*}
\sum_{t=r}^s \eta_{t,i}\norm{\nabla f_t(\x_{t, i}) - \nabla f_{t-1}(\x_{t-1, i})}_2^2 - \frac{1}{\eta_{t,i}} \norm{\x_{t,i} - \x_{t-1,i}}_2^2,
\end{align*}
where the negative term reflects algorithmic stability~\cite{JMLR'24:Sword++}. The key step is to decompose the gradient difference via the triangle inequality, i.e., $\norm{\nabla f_t(\x_{t, i}) - \nabla f_{t-1}(\x_{t-1, i})}_2^2 \leq \norm{\nabla f_t(\x_{t, i}) - \nabla f_{t-1}(\x_{t, i})}+ \norm{\nabla f_{t-1}(\x_{t, i}) - \nabla f_{t-1}(\x_{t-1, i})}_2^2$.
The first term captures the desired gradient variation at a fixed point, while the second term can be bounded as $ \norm{\nabla f_{t-1}(\x_{t, i}) - \nabla f_{t-1}(\x_{t-1, i})}_2^2 \leq L^2\norm{\x_{t,i} - \x_{t-1,i}}_2^2$ via smoothness. By tuning $\eta_{t,i} \approx 1/L$, this positive term is canceled by the negative stability term, leaving a bound that depends on the desired gradient variation.

The analysis above applies to static regret, which is well-studied in the literature~\cite{COLT'12:variation-Yang, conf/colt/RakhlinS13, NIPS'15:fast-rate-game}. However, new challenges arise within our two-layer ensemble structure. A key difference between our setting and prior work on gradient variation~\cite{COLT'22:parameter-free-omd, JMLR'24:Sword++,ICML'22:Zhang-simple, NeurIPS'23:universal,NeurIPS'24:universal,ICML'26:Yu-BCO-GV} lies in how base learners are managed. In previous work, base learners are typically initialized once and remain active throughout the entire process. In contrast, algorithms for interval regret must manage a dynamically changing set of base learners. This evolving structure introduces instability into the decision trajectory, making it difficult to leverage the stability between successive decisions to bound the gradient variation.

To better understand the challenges of deriving gradient-variation interval regret bounds with a two-layer ensemble structure, we first revisit the classical analysis approach from the literature~\cite{NeurIPS'23:universal,JMLR'24:Sword++}. Specifically, constructing a gradient-variation bound often requires handling the empirical gradient variation, $\bar{V}_{[r,s]} = \sum_{t=r+1}^{s} \left\|\nabla f_t(\x_t) - \nabla f_{t-1}(\x_{t-1})\right\|_2^2$, as defined in~\eqref{eq:empirical-gradient-variation}. Within the ensemble algorithm, the decision $\x_t = \sum_{i \in A_t} p_{t,i} \x_{t,i}$ is an aggregation from the base learners, which introduces a complex coupling between the meta learner and the base learners that complicates the analysis. Following an argument analogous to the single-learner case discussed above, the empirical gradient variation can be decomposed as:
\begin{align*}
  \bar{V}_{[r,s]} &\lesssim \sum_{t=r+1}^{s} \norm{\nabla f_t(\x_t) - \nabla f_{t-1}(\x_{t})}_2^2 + \sum_{t=r+1}^{s} \norm{\x_t - \x_{t-1}}_2^2\\
  &\lesssim \sum_{t=r+1}^{s} \norm{\nabla f_t(\x_t) - \nabla f_{t-1}(\x_{t})}_2^2 \notag \\
  &\quad + \sum_{t=r+1}^{s} \norm{\p_t - \p_{t-1}}^2_{1} + \sum_{i} p_{t,i} \norm{\x_{t,i} - \x_{t-1,i}}_2^2 .
\end{align*}
Pioneering work~\cite{JMLR'24:Sword++} proposes to cancel the second and third terms using negative terms that arise from the stability of the meta-learner and base learners, under the key algorithmic configuration that the dimension of $\p_t$ is fixed.

However, this analysis approach is not applicable to interval regret minimization, primarily for two reasons. First, it is unclear whether common meta algorithms for interval regret, such as AdaNormalHedge~\cite{COLT'15:Luo-AdaNormalHedge} or Optimistic-Adapt-ML-Prod~\cite{NIPS'16:Wei-non-stationary-expert}, provide the necessary negative terms of the form $-\sum_{t} \|\p_t - \p_{t-1}\|_1^2$ that are crucial for the cancellation. Second, and more fundamentally, algorithms for interval regret must manage a dynamically changing set of base learners. This causes the weight vector $\p_t$ to be unstable not only in value but also in its dimension over time, directly violating the fixed-dimension assumption. This instability makes it difficult to exploit such negative terms even if the meta-algorithm could theoretically provide them.

Inspired by recent advances in offline accelerated optimization~\cite{NIPS'19:UnixGrad, ICML'20:Bregman-divergence-negative-term} and online universal optimization~\cite{NeurIPS'24:universal}, we propose an alternative analysis technique. Our approach does not rely on the algorithmic stability terms used in prior work to bound the empirical gradient variation; instead, we leverage more refined properties of smooth functions to guide our analysis.

While traditional OCO analyses often begin by linearizing the loss function, they may overlook the important role of the negative Bregman divergence. The analysis can start from the following exact identity:
\begin{align}
  \label{eq:linearization-bregman}
  f_t(\x_t) - f_t(\u) = \inner{\nabla f_t(\x_t)}{\x_t - \u} - \D_{f_t}(\u, \x_t),
\end{align}
where $\D_f(\x, \y) = f(\x) - f(\y) - \inner{\nabla f(\y)}{\x - \y}$ is the Bregman divergence, which is central to our analysis. Notably, it is an intrinsic property of the loss function itself, arising independently of any specific algorithm structure.

To exploit the negative Bregman divergence term from~\eqref{eq:linearization-bregman}, we utilize the following property of smooth functions:
\begin{myProp}[Theorem 2.1.5 of~\cite{book'18:Nesterov-OPT}]
  \label{prop:bregman-divergence-property}
  For any $\x, \y \in \X$, an $L$-smooth function $f:\X \mapsto \R$ satisfies
  \begin{align*}
    \norm{\nabla f(\x) - \nabla f(\y)}^2_2 \leq 2L \D_f(\y, \x) \leq L^2 \norm{\x - \y}^2_2.
  \end{align*}
\end{myProp}

Proposition~\ref{prop:bregman-divergence-property} reveals that Bregman divergence provides a tighter approximation of gradient deviation compared to standard norm-based measures. To take advantage of the negative Bregman divergence in Eq.~\eqref{eq:linearization-bregman}, we reconsider how the empirical gradient variation is decomposed. Instead of directly analyzing the difference between $\x_t$ and $\x_{t-1}$, which inevitably introduces algorithm-dependent quantities, we adopt the following decomposition:
\begin{align*}
  \bar{V}_{[r,s]} &\lesssim \sum_{t=r+1}^{s} \norm{\nabla f_t(\x_t) - \nabla f_{t}(\u)}_2^2 + \norm{\nabla f_{t}(\u) - \nabla f_{t-1}(\u)}_2^2 \notag \\
  &\quad + \norm{\nabla f_{t-1}(\x_{t-1}) - \nabla f_{t-1}(\u)}_2^2 \notag \\
  & \lesssim \sum_{t=r+1}^{s} \norm{\nabla f_{t}(\u) - \nabla f_{t-1}(\u)}_2^2 + L\sum_{t=r}^{s} \D_{f_t}(\u, \x_t).
\end{align*}
The first term in the final bound is the gradient variation with respect to the fixed point $\u$, while the second term, involving the sum of Bregman divergences, can now be canceled by the negative Bregman divergence terms from Eq.~\eqref{eq:linearization-bregman}. A similar decomposition appears in~\cite{NeurIPS'24:universal}, but it is used for a different purpose. While~\cite{NeurIPS'24:universal} focuses on adapting to unknown curvature and reducing algorithmic complexity, our goal is to address the instability between $\p_t$ and $\p_{t-1}$, which arises from managing a dynamically changing set of base learners, a structural challenge unique to the interval regret setting.

As shown in Section~\ref{subsec:basic-components}, both the base regret and the meta regret can be bounded by the empirical gradient variation. This analysis allows us to utilize the negative Bregman divergence terms to cancel the positive Bregman terms arising at both the base and meta levels. As a result, we obtain a gradient-variation bound of the form $\sum_{t} \norm{\nabla f_t(\u) - \nabla f_{t-1}(\u)}_2^2$. The full proof is provided in Appendix~\ref{appendix:proof-our-approach}.

\section{Towards Lipschitz-Adaptivity in Non-Stationary Online Learning}
\label{sec:lipschitz_adaptive}
In the previous section, we presented a method for obtaining a gradient-variation interval regret bound under the assumption that the parameters $D$, $L$, and $G$ are known. However, estimating the Lipschitz constant $G$ and smoothness constant $L$ in advance is often impractical in non-stationary environments due to the changing nature of the online functions $f_t(\cdot)$.
This section addresses this challenge by developing a more robust, Lipschitz-adaptive method that provides competitive guarantees \emph{without} prior knowledge of $G$ and $L$. Our approach is twofold. First, we introduce the key component for achieving this adaptivity, a redesigned meta algorithm with corresponding theoretical guarantees, which may be of independent interest. Our meta algorithm can gracefully recover the previous result when the Lipschitz constant is provided. Second, we leverage this new component to construct a Lipschitz-adaptive algorithm that achieves the desired gradient-variation interval regret bounds.

\subsection{A Graceful Optimistic Lipschitz-Adaptive Meta Algorithm}
As reviewed in Section~\ref{sec:preliminaries}, designing algorithms with interval regret guarantees requires addressing three components: the base learners, the meta learner, and the schedule. Achieving Lipschitz adaptivity for the base learners is relatively straightforward, as demonstrated in prior work~\cite{TCS'18:SOGD, NIPS'19:UnixGrad, ICML'20:Bregman-divergence-negative-term}.
The schedule can also be made parameter-free. While the problem-dependent schedule from Section~\ref{sec:approach} requires the smoothness constant $L$ to optimize logarithmic dependencies, one can instead employ the problem-independent schedule from~\eqref{eq:problem-independent-schedule}. This alternative depends only on the timestamp to manage base learners, thereby avoiding the need for prior knowledge of $L$.
The primary technical challenge, therefore, lies in designing a more adaptive meta learner. The meta learner already faces several challenging requirements: it has to exhibit adaptivity to the environments, support the ``sleeping expert'' paradigm~\cite{STOC'97:sleeping-expert,COLT'15:Luo-AdaNormalHedge} for a dynamically changing set of experts, and correctly incorporate the scheduling strategy. These complex demands are what make designing the meta learner the main bottleneck, especially when important information, the Lipschitz constant $G$, is unknown.

To this end, we propose a new meta algorithm, \underbar{L}ipschitz-Adaptiv\underbar{e} \underbar{O}ptimistic-Adapt-ML-Prod (LEO Adapt-ML-Prod), which is presented in Algorithm~\ref{alg:leo-ada-ml-prod}. Our approach is inspired by Optimistic-Adapt-ML-Prod~\cite{NIPS'16:Wei-non-stationary-expert}, a meta algorithm that uses a productive update rule~\cite{COLT'05:second-order-Hedge,COLT'14:second-order-Hedge} and supports predictions with optimism.
Our key technical innovations are twofold. First, the algorithm is Lipschitz-adaptive; it leverages a clipping technique~\cite{COLT'19:ashok-cubic} to learn the Lipschitz constant on the fly (Line~\ref{line:enroll-mt}). Second, we propose a novel learning rate update rule (Line~\ref{line:redesign-lr}) that takes the minimum of two online-estimated quantities. This new learning rate ensures that if the Lipschitz constant is provided as prior information, our algorithm gracefully recovers the minimax-optimal regret bound. Lemma~\ref{lemma:time-varying-lr} provides a refined analysis tailored to this new learning rate design.

\begin{algorithm}[!t]
  \caption{LEO Adapt-ML-Prod}
  \label{alg:leo-ada-ml-prod}
  \begin{algorithmic}[1]
    \REQUIRE prior information of the scale $B_0$.
    \STATE \textbf{Initialize:} initialize the set of active experts $A_0 = \emptyset$, the number of experts $N_0 = 0$.
    \FOR{$t=1$ {\bfseries to} $T$}
    \STATE Update $A_{t}$, $N_t = N_{t-1} + \abs{A_{t} \backslash A_{t-1}}$, set $w_{t,i} = 1$, $\gamma_{i} = \ln(2i + 1)$ and $\eta_{t,i} = \min\big\{\sqrt{\gamma_{i}/(1 + B_{t-1}^2)}, 1 / (2B_{t-1})\big\}$ for new learner $i \in A_{t}\backslash A_{t-1}$; \label{line:update-sleeping-experts}
    \STATE Receive optimism $\m_t \in \R^{\abs{A_t}}$, update $\tilde{w}_{t,i} = w_{t,i}\exp(\eta_{t,i}m_{t,i})$ for $i \in A_t$; \label{line:correction}
    \STATE Calculate decision $\p_t \in \Delta_{\abs{A_t}}$ with $p_{t, i} = \frac{\eta_{t, i} \tilde{w}_{t, i}}{\sum_{j \in A_t} \eta_{t, j} \tilde{w}_{t, j}}$ and submit it;\label{line:calculate}
    \STATE Receive regret $\r_t \in \R^{\abs{A_t}}$, where $r_{t,i} = \inner{\p_t}{\ellb_t} - \ell_{t,i}$, update $B_{t} = \max\{B_{t-1}, \norm{\r_{t} - \m_{t}}_{\infty}\}$, and build the clipped regret $\bar{r}_{t,i} = m_{t,i} + \frac{B_{t-1}}{B_{t}}(r_{t,i} - m_{t,i})$;\label{line:enroll-mt}
    \STATE For all $i \in A_t$, update the learning rate $\eta_{t+1, i}$ using:
    \begin{align}
      \label{eq:double-descent-lr}
        \min\left\{\frac{1}{2B_t}, \sqrt{\frac{\gamma_{i}}{B_t^2 + \sum_{s=\text{start}_i}^t (\bar{r}_{s,i} - m_{s,i})^2 }}\right\};
    \end{align}
    \label{line:redesign-lr}
    \STATE Update $$w_{t+1, i} = \left(w_{t,i}\exp\left(\eta_{t, i}\bar{r}_{t,i} - \eta_{t,i}^2(\bar{r}_{t,i} - m_{t,i})^2\right)\right)^{\frac{\eta_{t+1,i}}{\eta_{t,i}}},$$ for all $i \in A_t$. \label{line:weight-update}
    \ENDFOR
  \end{algorithmic}
\end{algorithm}

In detail, Algorithm~\ref{alg:leo-ada-ml-prod} follows the standard Prediction with Expert Advice (PEA) protocol~\cite{COLT'90:Vovk-aggregate,journals/iandc/LittlestoneW94}. At each round, the learner provides a probability distribution $\p_t \in \Delta_{\abs{A_t}}$ over the available experts and then receives a loss vector $\boldsymbol{\ell}_t \in \R^{\abs{A_t}}$. The instantaneous regret with respect to the $i$-th expert is defined as $ r_{t,i} = \langle \p_t, \boldsymbol{\ell}_t \rangle - \ell_{t,i} $.
The algorithm has two key adaptive features. First, it explicitly supports the ``sleeping expert'' paradigm (Line~\ref{line:update-sleeping-experts}), allowing the set of active experts $A_t$ to change at each round. Second, it is Lipschitz-adaptive, meaning it does not require the true Lipschitz constant of the loss vectors, $\max_{t}\norm{\ellb_t}_{\infty}$, only an initial estimate $B_0$. The core mechanism for this is a clipping technique (Line~\ref{line:enroll-mt}), where the algorithm performs updates using a clipped regret $\bar{\r}_t$. Because the scale of this clipped quantity is known,
\begin{align*}
  \norm{\bar{\r}_{t} - \m_t}_{\infty} = \frac{B_{t-1}}{B_t}  \max_{i \in A_t} \ \abs{r_{t,i} - m_{t,i}} \leq B_{t-1},
\end{align*}
where we use the definition $ \abs{\bar{r}_{t,i} - m_{t,i}}= \frac{B_{t-1}}{B_t} \abs{r_{t,i} - m_{t,i}}$. It allows for the proper tuning of the learning rates without knowledge of the true Lipschitz constant~\cite{COLT'19:Lipschitz-MetaGrad,COLT'21:impossible-tuning}.

However, even with the clipping technique to handle unknown Lipschitzness, designing a Lipschitz-adaptive meta algorithm suitable for our specific goal of gradient-variation interval regret remains a challenge, as existing methods are insufficient. As summarized in Table~\ref{tab:leo-ada-ml-prod}, prior work in the PEA setting does not satisfy our requirements. These approaches either do not provide a direct mechanism for incorporating optimism~\cite{COLT'19:Lipschitz-MetaGrad}, or their regret bounds contain additional logarithmic factors~\cite{COLT'21:impossible-tuning,NeurIPS'24:LocalSmooth}. The presence of these additional factors is a key consideration, as they hinder the recovery of the optimal minimax bound when the Lipschitz constant is provided as prior information. These limitations motivate the design of a new meta algorithm.
Our design is guided by an ideal regret bound of the form $\O\big(\sqrt{\sum_{t=1}^T (\ell_{t,i^*} - m_{t,i^*})^2 \log N}\big)$. While such a bound that depends directly on the loss term $(\ell_{t,i^*} - m_{t,i^*})^2$ with optimal logarithmic dependence remains an open problem, our new algorithm takes an alternative approach. It achieves a bound based on the instantaneous regret term, $(r_{t,i} - m_{t,i})^2$, which proves sufficient for deriving our final gradient-variation guarantees.

\begin{table*}[t]
  \centering
  \caption{Comparison of guarantees for Lipschitz-adaptive meta-algorithms under the standard Prediction with Expert Advice setting. $N$ denotes the number of initialized experts, $B_T = \O(\max\{G, B_0\})$ denotes the maximum value of the Lipschitz constant $G$ and the initialized guess $B_0$. Following the convention of~\cite{COLT'14:second-order-Hedge,COLT'15:Luo-AdaNormalHedge}, we treat $\log \log (TB_T)$ terms as constants.}
  \label{tab:leo-ada-ml-prod}
  \renewcommand*{\arraystretch}{1.8}
  \begin{tabular}{ccc}
    \hline
    \textbf{Work}                                    & {\textbf{Regret Bounds}}                                                                                                                        & \textbf{Optimism} \\ \hline
    \cite{COLT'19:Lipschitz-MetaGrad}                & $\O\left(\sqrt{\sum_{t=1}^T (r_{t, i_{\star}}^2 ) \cdot \log N} + B_T\log N\right)$                                                             & \xmark            \\

    \cite{COLT'21:impossible-tuning}                 & $\O\left(\sqrt{\sum_{t=1}^T (\ell_{t, i_{\star}} - m_{t, i_{\star}} )^2 \cdot \log (NT)} + B_T\log (NT)\right)$                                 & \cmark            \\
    \cite{NeurIPS'24:LocalSmooth}                    & $\O\left(\sqrt{\sum_{t=1}^T (r_{t, i_{\star}} - m_{t, i_{\star}} )^2  }\cdot \log(NB_T) + B_T\log(NB_T)\right)$                                 & \cmark            \\
    \textbf{Ours (Theorem~\ref{thm:meta-algorithm})} & $\O\left(\sqrt{\sum_{t=1}^T (r_{t, i_{\star}} - m_{t, i_{\star}} )^2  }\cdot ({\log N} + \log (B_T))/\sqrt{\log N} + B_T\log(NB_T)\right)$ & \cmark            \\
    \hline
  \end{tabular}
\end{table*}

Theorem~\ref{thm:meta-algorithm} summarizes the guarantees of our proposed meta-algorithm, LEO Adapt-ML-Prod, which may be of independent interest. As shown in Table~\ref{tab:leo-ada-ml-prod}, our result strictly improves upon the bound of~\cite{NeurIPS'24:LocalSmooth}.
In comparison to other state-of-the-art methods, our algorithm achieves a form of Pareto optimality by striking a unique balance. Specifically, while~\cite{COLT'19:Lipschitz-MetaGrad} achieves tighter logarithmic factors, it does not incorporate optimism. On the other hand, while~\cite{COLT'21:impossible-tuning} obtains a bound based on the stronger quantity $(\ell_{t,\is} - m_{t,\is})^2$, it comes with additional $\log T$ factors and is forced to restart periodically. Our algorithm strikes a balance: it supports optimistic updates while achieving a strong second-order bound, a combination that is useful for our goal of designing non-stationary OCO algorithms.
\begin{myThm}
\label{thm:meta-algorithm}
Define $m_{t,i} = \inner{\p_t}{\h_t} - h_{t,i}$ for each $i \in A_t$, where $\h_t \in \R^{\abs{A_t}}$.
With the convention that $\ell_{t-1,i} = 0$ for newly added experts (i.e., $i \in A_t \setminus A_{t-1}$),
Algorithm~\ref{alg:leo-ada-ml-prod} guarantees that, for any expert $i_\star$ active throughout an interval $[r, s] \subseteq [T]$, the regret satisfies:
\begin{align*}
&\sum_{t=r}^s \!\left(\inner{\p_t}{\ellb_t} - \ell_{t,i_\star}\right)
\le  3\big(\tilde{\Gamma}_{s} + \ln N_{s+1}\big) B_s
+ B_s - B_{r-1}\\
& \quad \quad \quad + \sqrt{\sum_{t=r}^s (r_{t,i_\star} - m_{t,i_\star})^2}\,
\frac{2 \gamma_{i_\star} + \ln N_{s+1} + \tilde{\Gamma}_{s}}{\sqrt{\gamma_{i_\star}}},
\end{align*}
where $\tilde{\Gamma}_{s}$ is defined for any interval $[r,s]$ as
\begin{align}
  \tilde{\Gamma}_{s}
  &= \ln\Bigg(
  1 + \frac{1}{e}\bigg(
  \frac{B_s^2}{B_{0}^2}
  + \frac{1}{2} \ln\!\big(1 + (s\!-\!r\!+\!1)\tfrac{B_s^2}{B_{0}^2}\big) \notag \\
  &\quad + \ln\!\big(\tfrac{B_s}{B_{0}}\big)
  + \tfrac{\max_{t \in [1,s]} |B_t - B_{t-1}|}{B_{0}}
  \bigg)\!\Bigg) \label{eq:def-gamma} \\
  &\le \O\!\Big(\log\!\big(\tfrac{B_s^2}{B_{0}^2} + \log\!\big(s \cdot \tfrac{B_s}{B_{0}}\big)\big)\Big).\notag
\end{align}
In the above analysis, $N_{s+1} = \abs{\cup_{t=1}^{s+1} A_t}$ denotes the total number of experts initialized up to round $s+1$, and $B_s = \max\{B_0, \max_{t \in [s]} \norm{\r_t - \m_t}_\infty\}$ represents the maximum scale of the algorithm's input vectors up to time $s$.
\end{myThm}

The key to our theoretical guarantee is the novel learning rate design presented in Line~\ref{line:redesign-lr}. The learning rate for each expert, $\eta_{t,i}$, is set as the minimum of two online-estimated quantities. The first, $1/(2B_{t-1})$, ensures stability by maintaining the condition $\eta_{t,i} \abs{\bar{r}_{t,i} - m_{t,i}} \leq 1/2$. The second is a data-dependent, ``self-confident'' tuning parameter that adapts to observed feedback to ultimately yield a problem-dependent second-order regret bound.
Achieving such second-order bounds requires careful control over the variation of the learning rate, particularly the ratio $\eta_{t,i}/\eta_{t+1,i}$. Previous approaches for Lipschitz-adaptive algorithms either rely on grid search over a fixed set of learning rates~\cite{COLT'21:impossible-tuning} or use tuning strategies that lead to suboptimal guarantees~\cite{NeurIPS'24:LocalSmooth}. Our main contribution is to generalize the learning rate schedule of the classical Prod algorithm~\cite{COLT'14:second-order-Hedge, NIPS'16:Wei-non-stationary-expert} to the fully adaptive setting, providing a rigorous analysis that achieves a comparable guarantee, as established in Lemma~\ref{lemma:time-varying-lr}.
As a special case, we note that our result can recover known bounds in the standard PEA setting. If the number of experts is fixed and the Lipschitz constant is known, we can set $\gamma_{i} = \ln N$ for all $i$ and $B_t = B_0$ for all $t$. This simplifies $\tilde{\Gamma}_{T}$ to $\mathcal{O}(\log \log T)$ and gracefully matches the results of~\cite{COLT'14:second-order-Hedge,NIPS'16:Wei-non-stationary-expert}.

In Theorem~\ref{thm:meta-algorithm}, we specify a particular form for the optimism term: $m_{t,i} = \inner{\p_t}{\h_t} - h_{t, i}$. This choice is a technical necessity for our analysis, which requires the condition $ \inner{\p_t}{\bar{\r}_t} \leq 0 $, where $\bar{\r}_t$ is the clipped regret from Line~\ref{line:enroll-mt}. Our choice of optimism ensures this holds by satisfying the sufficient condition $ \inner{\p_t}{\m_t} \leq 0 $. While not essential in the standard PEA setting~\cite{NIPS'16:Wei-non-stationary-expert}, this form becomes critical in our analysis due to the clipping operation.
This choice of optimism introduces a circular dependency: $m_{t,i}$ depends on $\p_t$, which in turn depends on $m_{t,i}$. As detailed in~\cite[Section~3.3]{NIPS'16:Wei-non-stationary-expert}, this circular dependency can be resolved efficiently via a one-dimensional binary search. This procedure adds only $\O(\log T)$ iterations per round while still maintaining the requirement of a single gradient query per round. Beyond simply resolving this dependency, this procedure yields an additional and important benefit in our setting: it enables us to derive the gradient-variation bound at the meta level. Specifically, by setting $\h_t = \ellb_{t-1}$, the term $(r_{t,i} - m_{t,i})^2$ simplifies to:
\begin{align*}
  (r_{t,i} - m_{t,i})^2 &= \left[\inner{\p_t}{\ellb_t - \ellb_{t-1}} - (\ell_{t,i} - \ell_{t-1, i})\right]^2\\
   &= \mathcal{O}(\norm{\ellb_t - \ellb_{t-1}}_{\infty}^2).
\end{align*}
This result, which connects the regret term to the variation in the loss vectors, is a key component of our algorithm's design.

Finally, we remark on the parameters $B_T$ and $B_0$. Since $B_T$ is a linear factor in the regret bounds shown in Table~\ref{tab:leo-ada-ml-prod}, the regime where $B_T = \Omega(T)$ implies a trivial linear regret. We therefore focus on the more relevant case where $B_T = o(T)$. In this regime, the term $\O(\log \log (TB_T))$ grows very slowly and can be treated as a constant, following standard convention~\cite{COLT'14:second-order-Hedge,COLT'15:Luo-AdaNormalHedge}.
Regarding the initial guess $B_0$, while \cite{COLT'19:Lipschitz-MetaGrad} proposes restart strategies to eliminate this dependency, such a mechanism is not essential for our primary goal. Our focus is on designing an adaptive method for non-stationary OCO. In our framework, setting $B_0 = \O(1/\log T)$ is sufficient to ensure that its contribution to the final regret bound is limited to negligible logarithmic factors.

\subsection{Lipschitz-Adaptive Gradient-Variation Interval Regret}
In this section, we develop a Lipschitz-adaptive gradient-variation interval regret bound for non-stationary online convex optimization. Our method is presented in Algorithm~\ref{alg:la-adaptive-oco}, named GAIR with \underbar{L}ipschitz-Adaptivity (GAIR-L). There are two major modifications compared to Algorithm~\ref{alg:adaptive-oco}.
First, in Lines~\ref{line:la-if-problem-independent} and~\ref{line:la-if-problem-dependent}, we provide two options for the scheduling strategy. The first option employs the problem-independent schedule from~\eqref{eq:problem-independent-schedule}. This choice does not require prior knowledge of the smoothness constant $L$, but it results in a slightly looser regret bound in its logarithmic terms and requires maintaining $\Theta(\log T)$ base learners per round. The second option uses a problem-dependent schedule. This choice requires $L$ (but not $G$) and achieves a tighter regret bound while maintaining only $\O(\log F_{T})$ base learners per round, where $F_T$ is the small loss quantity.
Second, we replace the meta algorithm with our newly proposed LEO Adapt-ML-Prod (Algorithm~\ref{alg:leo-ada-ml-prod}) to handle the unknown Lipschitz constant. The base algorithm remains the same as in~\eqref{eq:base-algorithm-omd}, which is already Lipschitz-adaptive.

We first provide the theoretical guarantees for GAIR-L when using the problem-independent schedule. This result is agnostic to both the Lipschitz constant $G$ and the smoothness constant $L$, making it robust to unknown function properties. The proof is provided in Appendix~\ref{appendix:proof-interval-dynamic-regret-without-L}.
\begin{myThm}[GAIR-L without Smoothness Constant]
  \label{thm:la-adaptive-oco-without-L}
  Under Assumptions~\ref{ass:bounded-domain},~\ref{ass:Lipschitzness}, and~\ref{ass:smoothness}, choosing Algorithm~\ref{alg:la-adaptive-oco} as the meta algorithm, setting the step size of OMD in Eq.~\eqref{eq:base-algorithm-omd} as in Theorem~\ref{thm:adaptive-oco}, applying the problem-independent schedule defined in Eq.~\eqref{eq:problem-independent-schedule}, then for any $\u \in \X$ and any $I = [r,s ]\subseteq [T]$, GAIR-L achieves the following interval regret bound:
  \begin{align*}
    \O\bigg(&\sqrt{\min\{V_{[r,s]}, F_{[r,s]}\} \log (s\!-\!r)}\cdot \Big(\sqrt{\log s} + \frac{\tilde{\Gamma}_{s}}{\sqrt{\log r}}\Big) \\
    &+ \tilde{\Gamma}_{s} \log (s\!-\!r) + B_s \log s \log(s\!-\!r) + \Big( \frac{\tilde{\Gamma}_{s}}{\log r}\Big)^2 \bigg),
  \end{align*}
  where $F_{[a, b]} = \min_{\x \in \X} \sum_{t=a}^b f_t(\x)$ denotes the small loss, $V_{[a,b]} = \sum_{t=a+1}^b \sup_{\x \in \X}\norm{\nabla f_t(\x) - \nabla f_{t-1}(\x)}_2^2$ is the gradient variation over interval $[a,b]$, $\tilde{\Gamma}_{s}$ is defined in Eq.~\eqref{eq:def-gamma}, and $B_{s} = \max\{\max_{t \in [1, s], i \in A_t} \abs{\inner{\nabla f_t(\x_t) - \nabla f_{t-1}(\x_{t-1})}{\x_t - \x_{t,i}}},2G_0D\}$ is the maximum input scale of meta algorithm.
\end{myThm}

When the parameters $G$ and $L$ are unknown, our algorithm can use a problem-independent schedule to achieve the guarantees in Theorem~\ref{thm:la-adaptive-oco-without-L}. This approach is robust to the functions' properties but comes with trade-offs: it introduces problem-independent logarithmic factors (e.g., $\O(\log s)$) and may incur additional computational cost by initializing a new base learner at every round.
However, if we aim for better computational adaptivity and tighter regret bounds, and the smoothness constant $L$ is known, we can design a more efficient problem-dependent strategy. This setting, where the smoothness constant $L$ is known but the Lipschitz constant $G$ is unknown, may seem counterintuitive at first, since higher-order information is often harder to obtain in practice. Nevertheless, such situations can arise; a logistic regression example is provided in Appendix~\ref{subappendix:justification} to illustrate this case. For this scenario, we design a new threshold generation function $\mathcal{G}_{\text{LA}}(t, s_i, i)$, defined as:
\begin{align}
  &\mathcal{G}_{\text{LA}}(t, s_i, i) = 56LD^2\left(\frac{3 \ln (2i + 1) + \tilde{\Gamma}_{t}}{\sqrt{\ln (2 i + 1)}} + \frac{5}{2}\right)^2 \notag\\
  &\quad \quad + 2D\frac{3 \ln (2i + 1) + \tilde{\Gamma}_{t}}{\sqrt{\ln (2 i + 1)}} \notag \\
  &\quad \quad + 5D+ 3\left(3(\tilde{\Gamma}_{t} + \ln (2i + 1))B_{t} + B_t - B_{s_i}\right), \label{eq:la-threshold-function}
\end{align}
where $\tilde{\Gamma}_{t}$ is formally defined in~\eqref{eq:def-gamma}; we assume $t \in [s_i, s_{i+1} - 1]$, where $s_{i+1}$ is the upcoming marker to be determined.

The key innovation of this function is its per-round Lipschitzness adaptivity policy. At each round $t$, our algorithm computes an online estimate of the Lipschitz constant and immediately updates the threshold using this new estimate. This marks a significant difference from prior designs like~\cite{NeurIPS'22:efficient}, where the threshold is only updated when a new base learner is initialized. Our more aggressive, fine-grained mechanism allows the algorithm to better adapt to dynamically changing environments, especially when the underlying Lipschitzness may vary over time. With this newly designed threshold generation function, we can establish the following theoretical result. The proof can be found in Appendix~\ref{appendix:proof-interval-dynamic-regret}.
\begin{myThm}[GAIR-L with Smoothness Constant]
  \label{thm:la-adaptive-oco-with-L}
  Under Assumptions~\ref{ass:bounded-domain}--\ref{ass:smoothness}, additionally assuming that the smoothness constant $L$ is known to the learner, choosing Algorithm~\ref{alg:la-adaptive-oco} as the meta algorithm, setting the step size of OMD in~\eqref{eq:base-algorithm-omd} as in Theorem~\ref{thm:adaptive-oco}, applying the problem-dependent schedule defined in Eq.~\eqref{eq:problem-dependent-schedule}, setting the threshold function as specialized in Eq.~\eqref{eq:la-threshold-function}, then for any $\u \in \X$ and any $I = [r,s ]\subseteq [T]$, GAIR-L achieves the following interval regret bound:
  \begin{align*}
   \O\bigg(&\sqrt{\min\{V_{[r,s]}, F_{[r,s]}\} \log F_I}\cdot \Big(\sqrt{\log F_{[1,s]}} + \tilde{\Gamma}_{s}\Big) \\
   &+\tilde{\Gamma}_{s}^2 + \tilde{\Gamma}_{s} (B_s + \log F_I) + B_s \log F_{[1,s]} \log F_I \bigg),
  \end{align*}
  where $F_{[a, b]} = \min_{\x \in \X} \sum_{t=a}^b f_t(\x)$ denotes the small loss, $V_{[a,b]} = \sum_{t=a+1}^b \sup_{\x \in \X}\norm{\nabla f_t(\x) - \nabla f_{t-1}(\x)}_2^2$ represents the gradient variation over interval $[a,b]$, $\tilde{\Gamma}_{s}$ is defined in Eq.~\eqref{eq:def-gamma}, and $B_{s} = \max\{\max_{t \in [1, s], i \in A_t} \abs{\inner{\nabla f_t(\x_t) - \nabla f_{t-1}(\x_{t-1})}{\x_t - \x_{t,i}}},2G_0D\}$ is the maximum input scale of meta algorithm.
\end{myThm}

Finally, we remark that the two theorems presented in this section are direct corollaries of Theorem~\ref{thm:interval-dynamic-regret-without-L-formal} and Theorem~\ref{thm:interval-dynamic-regret-formal}, which address a more general setting with changing comparators. By specializing the comparators to be fixed, we recover the results discussed here. Moreover, in both cases, the introduction of Lipschitz adaptivity results in additional logarithmic terms. For this reason, we do not further discuss worst-case optimal results for simplicity; however, the techniques developed in the previous section can still be applied if one is interested in this direction.

\begin{algorithm}[!t]
  \caption{GAIR-L}
  \label{alg:la-adaptive-oco}
  \begin{algorithmic}[1]
    \REQUIRE Lipschitz constant estimate $G_0$, domain diameter $D$, choice of the schedule, threshold function $\mathcal{G}_{\text{LA}}:\mathbb{N} \mapsto \R_+$.
    \STATE \textbf{Initialize:} active base-learner index set $A_0 = \emptyset$; number of initialized base learners $N_0 = 0$; cumulative loss $L_0 = \infty$; initial marker $s_0 = 0$; initial loss vector $\ellb_0 = \boldsymbol{0}$; initial scale $B_0 = 2G_0 D$.
    \FOR{$t = 1$ {\bfseries to} $T$}
      \IF{\texttt{using problem-independent schedule}} \label{line:la-if-problem-independent}
        \STATE Update the number of learners: $N_t = N_{t-1} + 1$, and reset $L_t = 0$;
          \STATE Remove inactive base learners $A_t = A_{t-1} \setminus \{i: [i,t-1] \in {\mathcal{S}}\}$;
          \STATE Initialize a new base learner starting from $t$ (Eq.~\eqref{eq:base-algorithm-omd}), and set $A_t = A_t \cup \{t\}$;
      \ELSIF{\texttt{using problem-dependent schedule}} \label{line:la-if-problem-dependent}
        \IF{$L_{t-1} > \mathcal{G}_{\text{LA}}(t-1, s_{N_{t-1}}, N_{t-1})$} \label{line:lagair-prob-dependent}
          \STATE Register a new marker: $N_t = N_{t-1} + 1$, $s_{N_t} = t$, and reset $L_t = 0$;
          \STATE Remove inactive base learners $A_t = A_{t-1} \setminus \{i: [s_i, s_{N_t}-1] \in \tilde{\mathcal{S}}\}$;
          \STATE Initialize a new base learner starting from $s_{N_t}$ (Eq.~\eqref{eq:base-algorithm-omd}), and set $A_t = A_t \cup \{N_t\}$;
        \ELSE
          \STATE $A_t = A_{t-1}$, $N_t = N_{t-1}$;
        \ENDIF
      \ENDIF

      \STATE Obtain $\x_{t,i}$ from each base learner $i \in A_t$, send $m_{t,i} = \inner{\nabla f_{t-1}(\x_{t-1})}{\x_t - \x_{t,i}}$ and $A_t$ to the meta learner; \label{line:lagair-meta-optimism}
      \STATE Receive $\p_t \in \Delta_{\abs{A_t}}$ from the meta learner and play $\x_t = \sum_{i \in A_t} p_{t,i}\x_{t,i}$;
      \STATE Observe $f_t(\x_t)$ and $\nabla f_t(\x_t)$; update $L_{t+1} = L_t + f_t(\x_t)$, set $\ell_{t,i} = \inner{\nabla f_t(\x_t)}{\x_{t,i}}$ for the meta learner, and $\g_t = M_{t+1} = \nabla f_t(\x_t)$ for base learners; \label{line:lagair-base-optimism}
      \STATE Obtain $B_t$ from the meta learner. \label{line:lagair-get-b-t}
    \ENDFOR
  \end{algorithmic}
\end{algorithm}

\section{Implications and Applications}
\label{sec:implications}
To further demonstrate the versatility of our proposed method, we introduce the following implications and applications. In Section~\ref{subsec:interval-dynamic-regret}, we show that GAIR-L implies the interval dynamic regret bound, a stronger measure for non-stationary environments. This result is achieved using the same algorithm configuration through a white-box reduction. In Section~\ref{subsec:sea}, we study the Stochastically Extended Adversarial (SEA) model, a framework that unifies stochastic and adversarial optimization. We demonstrate that our result can be directly applied to the SEA model, yielding the first piecewise characterization of this model over local intervals.

\subsection{Implication for Interval Dynamic Regret Bounds}
\label{subsec:interval-dynamic-regret}
In non-stationary online learning, two classical performance measures have received significant attention. The first is the interval regret, defined in Eq.~\eqref{eq:strongly-adaptive-regret}, which is the primary focus of this work. It is defined over arbitrary intervals and captures non-stationarity in a local manner. The second is the dynamic regret~\cite{ICML'03:zinkvich,NIPS'18:Zhang-Ader,NIPS'20:sword,ICML'22:TV-game}, typically defined as
\begin{align}
  \label{eq:def-dynamic-regret}
  \textsc{D-Reg}_T(\u_{1:T}) = \sum_{t=1}^{T} f_t(\x_t) - \sum_{t=1}^{T} f_t(\u_t),
\end{align}
where $\u_1, \dots, \u_T \in \X$ denotes a time-varying comparator sequence, and the regret is measured against any such sequence. This measure reflects global non-stationarity by comparing the algorithm's performance against a horizon-long sequence. Such robustness considerations are increasingly important in the broader context of building trustworthy and reliable machine learning in open environments~\cite{NSR'22:Zhou-open-environment,SCIS'25:Yuan-encryption-GFM}.

The relationship between these two measures differs by setting. Within the PEA framework, a simplified version of OCO, it is established that an interval regret guarantee can be directly converted to a dynamic regret guarantee through a black-box reduction~\cite{COLT'15:Luo-AdaNormalHedge}. Here, a black-box reduction means that the conversion is achieved without knowing the exact structure of the algorithms. However, under the standard OCO framework, the general relationship between adaptive and dynamic regret remains unclear, and converting guarantees is generally non-trivial~\cite{ICML'20:Ashok,NeurIPS'22:efficient}. Some specific connections are known; for example,~\cite{ICML'18:zhang-dynamic-adaptive} shows that an interval regret guarantee can imply a bound on the worst-case dynamic regret~\cite{OR'15:dynamic-function-VT}. This is a special case of~\eqref{eq:def-dynamic-regret} where the comparator is the sequence of per-round minimizers, $\u_t = \argmin_{\x \in \X} f_t(\x)$.
To bridge the gap between local and global non-stationarity, the literature has proposed a stronger performance measure called interval dynamic regret~\cite{ICML'20:Ashok, AISTATS'20:Zhang}, defined in Eq.~\eqref{eq:def-interval-dynamic-regret}. This measure compares the algorithm's performance against arbitrary comparators over arbitrary intervals. Minimizing the interval dynamic regret is suggested to lead to improved robustness in non-stationary environments~\cite{ICML'20:Ashok, AISTATS'20:Zhang}.

In Table~\ref{tab:comparison}, we report the corresponding results under the interval dynamic regret measure, assuming knowledge of the Lipschitz constant $G$ and the smoothness constant $L$. Our result is a direct implication of Theorem~\ref{thm:la-adaptive-oco-without-L} or Theorem~\ref{thm:la-adaptive-oco-with-L}, requiring no additional parameter tuning or modifications. To the best of our knowledge, this is the first result to simultaneously provide interval dynamic regret guarantees in terms of both gradient variation and small loss.

Compared to prior work, when knowledge of the smoothness constant $L$ is available (under the same assumptions as in previous studies), our algorithm also offers better computational advantages. The method of~\cite{AISTATS'20:Zhang} proposes a three-layer algorithmic structure, resulting in $\mathcal{O}(\log^2 T)$ base learners. While~\cite{ICML'20:Ashok} achieves $\mathcal{O}(\log T)$ base learners, their approach requires operating over a lifted domain with potentially expensive projection steps. More recently,~\cite{arXiv'23:efficient-projection} refines the structure of~\cite{AISTATS'20:Zhang} with a problem-dependent schedule, achieving a base learner complexity of $\O(\log (\min_{\u_{1}, \dots, \u_T} \{\sum_{t=1}^T f_t(\u_t) + \norm{\u_t - \u_{t-1}}_2\}) \cdot \log T)$, but this method still employs a three-layer ensemble, with the worst-case complexity remaining $\O(\log^2 T)$. In contrast, our algorithm uses a simpler two-layer structure and achieves a more efficient bound on the number of base learners, requiring only $\O(\log (\min_{\u_{1}, \dots, \u_T} \{\sum_{t=1}^T f_t(\u_t) +(\sum_{t=1}^T \norm{\u_t - \u_{t-1}}_2)^2\}))$ base learners per round asymptotically. In the worst case, this bound reduces to $\O(\log T)$ base learners, matching the efficiency of~\cite{ICML'20:Ashok}. Furthermore, our complexity bound implies that, when choosing $\u_1 = \dots = \u_T = \argmin_{\x \in \X} \sum_{t=1}^T f_t(\x)$, the number of base learners maintained is at most $\O(\log F_T)$ per round. This matches the result for algorithms designed for interval regret minimization~\cite{ICML19:Zhang-Adaptive-Smooth,arXiv'23:efficient-projection}, but our analysis provides a stronger guarantee. When $G$ and $L$ are unknown, our algorithm still ensures comparable $\O(\log T)$ base learners per round.

Our technique builds upon ideas from prior work~\cite{ICML'20:Ashok, NeurIPS'24:optimal-switching-regret}, but requires a more delicate analysis to accommodate the problem-dependent schedule and to simultaneously derive two types of problem-dependent bounds. Taking the gradient-variation bound as an example, our analysis follows a white-box approach, leveraging the fact that each base learner is an instance of the online mirror descent algorithm, a method shown to have advantages when competing with changing comparators~\cite{COLT'22:parameter-free-omd}. We first establish an intermediate bound of the form $\widetilde{\mathcal{O}}((1 + P_I)\sqrt{\bar{V}_I} + P_I)$ on any interval $I = [r,s] \subseteq [1, T]$, where $P_I = \sum_{t=r+1}^s \norm{\u_t - \u_{t-1}}_2$ denotes the path length, and $\bar{V}_I$ denotes the empirical gradient variation as defined in~\eqref{eq:empirical-gradient-variation}. Our goal is then to prove that any target interval $I^\prime = [r^\prime, s^\prime]$ can be partitioned into $k = \O(P_{I^\prime} / D)$ subintervals, $\{I_{j}^\prime\}_{j=1}^k$, such that the path length within each subinterval is bounded, i.e., $P_{I_{j}^\prime} = \O(D)$. By summing the intermediate bounds over these subintervals and applying the Cauchy–Schwarz inequality, we improve the intermediate result to $\widetilde{\mathcal{O}}(\sqrt{(1 + P_{I^\prime} )\bar{V}_{I^\prime}} + P_{I^\prime})$.

Finally, we use the negative Bregman divergence to complete the gradient-variation bound. Unlike in Section~\ref{subsec:key-analysis}, the overall regret analysis here yields a negative term of the form $-\sum_{t \in I^\prime} \D_{f_t} (\u_t, \x_t)$. We develop a careful decomposition of the empirical gradient variation, showing that
\begin{align*}
  \bar{V}_{I^\prime} \lesssim \ &\sum_{t \in I^\prime} D_{f_t}(\u_t, \x_t) + \sum_{t=r^\prime + 1}^{s^\prime} \norm{ \nabla f_t(\u_t) - \nabla f_{t-1}(\u_{t})}_2^2 \\
  &\quad + D\sum_{t=r^\prime + 1}^{s^\prime} \norm{\u_t - \u_{t-1}}_2.
\end{align*}
Then the desired gradient-variation bound is obtained by leveraging the negative Bregman term from the regret analysis to cancel the positive Bregman term arising from the decomposition of $\bar{V}_{I^\prime}$.

We now provide an informal theorem statement for the interval dynamic regret bound achieved by GAIR-L when using the problem-independent schedule. This result does not require prior knowledge of $G$ or $L$. The formal version of this theorem is presented in Appendix~\ref{subappendix:interval-dynamic-formal-without-L}. We also provide the formal version of the interval dynamic regret bound for the case when $L$ is known in Appendix~\ref{subappendix:interval-dynamic-formal-with-L}.
\begin{myThm}[informal]
    \label{thm:interval-dynamic-regret}
    Under Assumptions~\ref{ass:bounded-domain}--\ref{ass:smoothness}, choosing Algorithm~\ref{alg:la-adaptive-oco} as the meta algorithm, setting the step size of OMD in Eq.~\eqref{eq:base-algorithm-omd} as in Theorem~\ref{thm:adaptive-oco}, applying the problem-independent schedule defined in Eq.~\eqref{eq:problem-independent-schedule}, then for any comparators $\u_r, \dots, \u_s \in \X$ and any interval $I = [r,s ]\subseteq [T]$, GAIR-L achieves the interval dynamic regret bound:
    \begin{align*}
      \Ot\left(\sqrt{\min\{V_{[r,s]}, F_{[r,s]}^{\u}\}\cdot(1+P_I) } + P_I  \right),
    \end{align*}
  where $F_{[a, b]}^{\u} = \sum_{t=a}^b f_t(\u_t)$ denotes the cumulative loss of the comparator sequence, $V_{[a,b]} = \sum_{t=a+1}^b \sup_{\x \in \X}\norm{\nabla f_t(\x) - \nabla f_{t-1}(\x)}_2^2$ represents the gradient variation over $[a,b]$, $P_{[a,b]} = \sum_{t=a+1}^b \norm{\u_t - \u_{t-1}}_2$ is the path length. $\Ot(\cdot)$ hides logarithmic factors in $s$ and $B_s$.
\end{myThm}

The proofs for interval dynamic regret are provided in two appendices: Appendix~\ref{appendix:proof-interval-dynamic-regret} addresses the case where $L$ is known but $G$ is unknown, while Appendix~\ref{appendix:proof-interval-dynamic-regret-without-L} addresses the case where both $G$ and $L$ are unknown. We note that these proofs follow a top-down structure, and all preceding results can be viewed as their special cases. For example, setting $P_I = 0$ recovers the result in Section~\ref{sec:lipschitz_adaptive}, and assuming both $G$ and $L$ are known recovers the result in Section~\ref{sec:approach}.

\subsection{Application to Stochastically Extended Adversarial Model}
\label{subsec:sea}
The Stochastically Extended Adversarial (SEA) model provides a unified framework that interpolates between adversarial and stochastic online optimization~\cite{nips'22:sea,JMLR'24:OMD4SEA}. In this model, the stochastic component is that at each time step $t$, the loss function $f_t$ is drawn from a distribution $\mathfrak{D}_t$. The adversarial nature of the environments is captured by shifts in this distribution $\mathfrak{D}_t$ over time.
To formally characterize the environments, the SEA model introduces two key quantities:
\begin{align*}
    \Sigma_{1:T}^2 &= \mathbb{E}\left[\sum_{t=2}^T \sup_{\x\in\X} \norm{\nabla F_t(\x) - \nabla F_{t-1}(\x)}_2^2\right],\\
    \sigma_{1:T}^2 &= \sum_{t=1}^T \sup_{\x\in\X}\mathbb{E}[\norm{\nabla f_t(\x) - \nabla F_t(\x)}_2^2].
\end{align*}
Here, $F_t(\x) = \mathbb{E}_{f_t\sim \mathfrak{D}_t}[f_t(\x)]$ denotes the expected loss function. The term $\Sigma_{1:T}^2$ measures the cumulative distributional shift, representing the adversarial component, while $\sigma_{1:T}^2$ measures the cumulative stochastic variance around the mean, representing the stochastic component.
For convex functions, previous work establishes a static regret bound of $\O(\sqrt{\sigma_{1:T}^2} + \sqrt{\Sigma_{1:T}^2})$~\cite{nips'22:sea}, and a dynamic regret bound of $\O(\sqrt{(\sigma_{1:T}^2 + \Sigma_{1:T}^2 )(1+P_T)}+ P_T)$~\cite{JMLR'24:OMD4SEA}.

However, these results analyze the environments' properties globally. Our work offers a more refined and localized perspective. Specifically, the interval dynamic guarantee of our algorithm, GAIR-L, as established in Theorem~\ref{thm:interval-dynamic-regret}, directly implies the first piecewise characterization of performance under the SEA model over arbitrary intervals. This contribution is summarized in the following corollary.
\begin{myCor}
    \label{cor:sea}
    Under Assumptions~\ref{ass:bounded-domain} and~\ref{ass:smoothness}, and adopting the same configurations as in Theorem~\ref{thm:interval-dynamic-regret-without-L-formal}, additionally assume that the gradient norms of the individual functions are uniformly bounded by $G$, i.e., $\max_{\x \in \X} \norm{\nabla f_t(\x)} \le G$ for all $t \in [T]$.
    Then, for any sequence of comparators $\u_r, \dots, \u_s \in \X$ and any interval $I = [r, s] \subseteq [T]$, Algorithm~\ref{alg:la-adaptive-oco} (GAIR-L) achieves the interval dynamic regret bound under the SEA model:
    \begin{align*}
        &\E\left[\sum_{t=r}^s f_t(\x_t) -  \sum_{t=r}^s f_t(\u_t)\right]\\
        &\leq \Ot\left(\sqrt{(\sigma_{[r,s]}^2 + \Sigma_{[r,s]}^2)(1+P_{[r,s]})} + P_{[r,s]} \right)
    \end{align*}
    where $\sigma_{[r,s]}^2 = \sum_{t=r}^s \sup_{\x\in\X}\mathbb{E}[\norm{\nabla f_t(\x) - \nabla F_t(\x)}_2^2]$ and $\Sigma_{[r,s]}^2 = \mathbb{E}[\sum_{t=r+1}^s \sup_{\x\in\X}\norm{\nabla F_t(\x) - \nabla F_{t-1}(\x)}_2^2]$ are the interval versions of the stochastic variance and distributional shift, respectively.
\end{myCor}
The detailed proof is provided in Appendix~\ref{appendix:proof-sea}. In Section~\ref{subsec:exp-sea}, we design a SEA setting based on real-world data, and validate the effectiveness of our method under this setting.

\section{Experiments}
\label{sec:exp}
In this section, we evaluate the effectiveness of our algorithm through regression tasks on synthetic data and classification experiments on the real-world MNIST dataset~\cite{deng2012mnist}.
We compare our method, GAIR-L (Algorithm~\ref{alg:la-adaptive-oco}), against several interval regret minimization baselines: SAOL~\cite{ICML'15:Daniely-adaptive}; SACS and SACSPP, which are the methods from~\cite{ICML19:Zhang-Adaptive-Smooth} with problem-independent and problem-dependent schedules, respectively; and NIPS22~\cite{NeurIPS'22:efficient}. The experimental details and results are presented as follows.

\subsection{Synthetic Data}
We first evaluate GAIR-L on a synthetic online regression task, with a focus on verifying its Lipschitz-adaptivity. The experiment is set in a time horizon of $T = 2000$. At each round $t$, the learner outputs a model parameter $\x_t$ from a feasible domain $\X$, which is a Euclidean ball of radius $R = 1$ in dimension $d = 5$. The learner then receives a data sample $(\z_t, y_t)$ and incurs the squared loss $f_t(\x_t) = \frac{1}{2} \cdot \text{scale} \cdot (\x_t^\top \z_t - y_t)^2$. To simulate a non-stationary environment, we gradually shift the underlying model's optimal parameters over time.

To validate Lipschitz-adaptivity, we design the loss functions such that their Lipschitz constant increases gradually from an initial value $G_0 = 5$ to a final value $G = 50$ by adjusting the scale factor. For comparison, baseline algorithms must tune their step sizes using the worst-case estimate $G = 50$, whereas GAIR-L can use the more favorable initial estimate $G_0 = 5$ (Figure~\ref{fig:la_results}, Left). Additionally, we conduct an ablation study to isolate the impact of Lipschitz-adaptivity. In this setting, we fix $G_0 = G = 5$, so the Lipschitz constant remains static. This allows us to evaluate the performance of GAIR (our algorithm from Section~\ref{sec:approach} without Lipschitz-adaptivity) as a baseline in this environment (Figure~\ref{fig:la_results}, Right).

\begin{figure}[!t]
\centering
\subfloat{\includegraphics[width=0.48\linewidth]{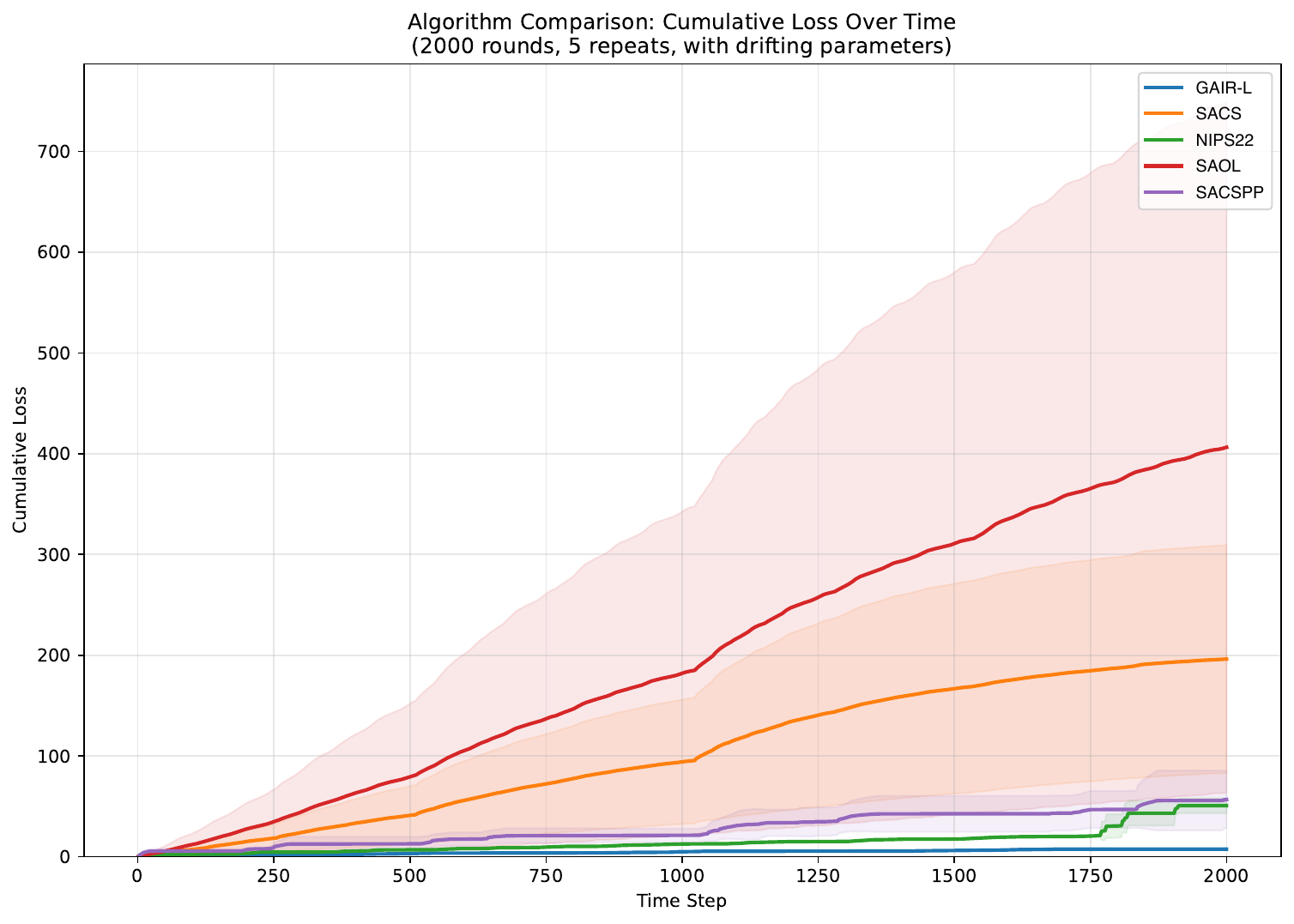}}
\hfill
\subfloat{\includegraphics[width=0.48\linewidth]{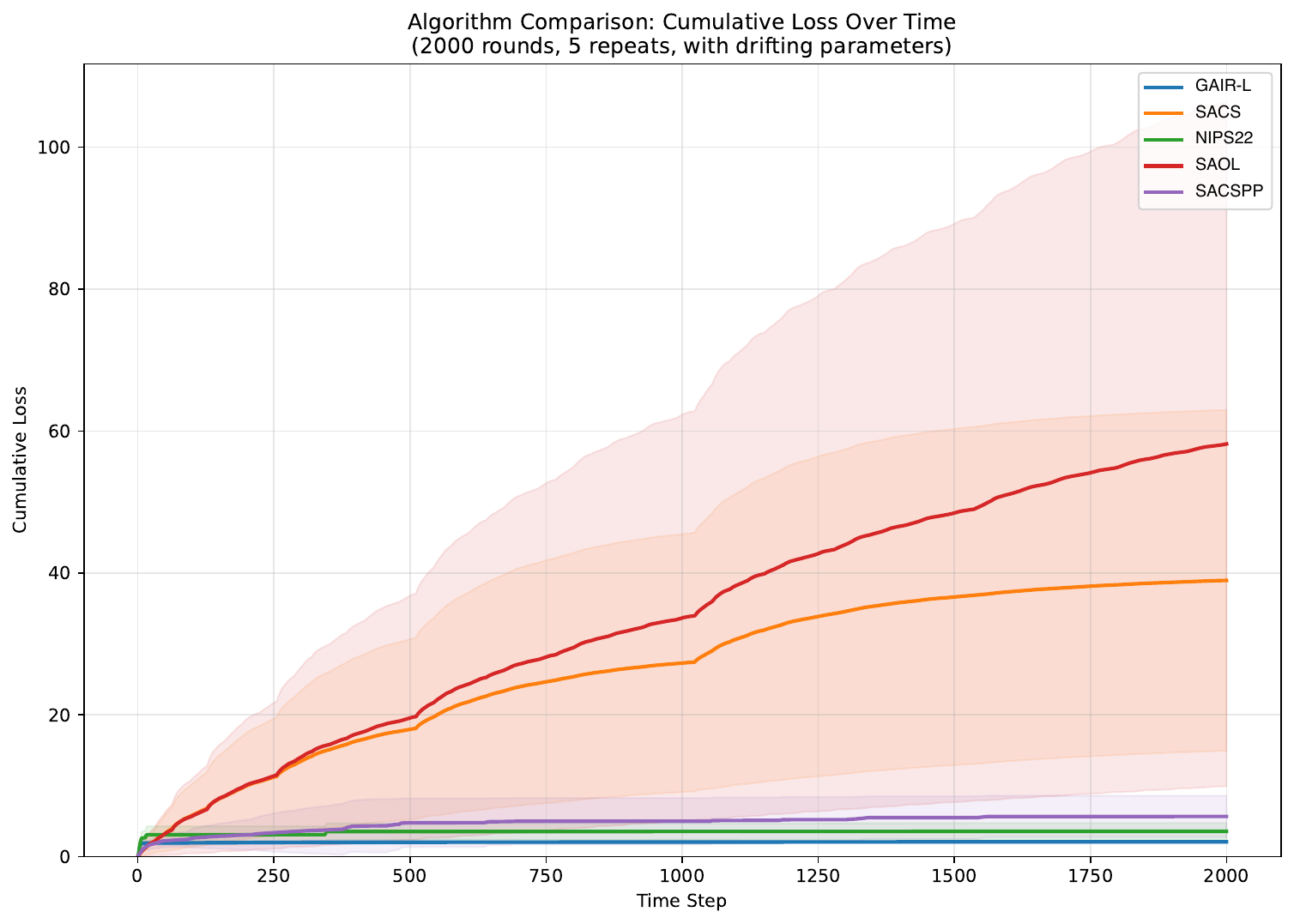}}
\caption{(Left) Performance comparison between GAIR-L and baseline algorithms with $G_0 = 5$ and $G = 50$. (Right) Ablation study under fixed Lipschitz constant $G_0 = G = 5$. Even without adaptivity to Lipschitzness, GAIR maintains strong performance.}
\label{fig:la_results}
\end{figure}

For the Lipschitz-adaptivity evaluation (Figure~\ref{fig:la_results}, Left), when the Lipschitz constant of the loss functions gradually increases from $G_0 = 5$ to $G = 50$, GAIR-L consistently achieves the lowest cumulative loss among all competing methods. This result supports the effectiveness of our Lipschitz-adaptive design, as other methods, relying on step sizes tuned by the worst-case $G = 50$, potentially incur higher loss. For the ablation study (Figure~\ref{fig:la_results}, Right), we observe that GAIR-L, even without leveraging Lipschitz adaptivity, remains competitive with or outperforms other baselines. This suggests that the core algorithmic design of GAIR-L, the gradient-variation adaptivity, contributes meaningfully to performance, offering stronger robustness.
\subsection{Application to Real-World Data with SEA Model}
\label{subsec:exp-sea}
To validate the theoretical guarantees of our methods, particularly the novel gradient-variation bound, we conduct an online multi-class classification experiment on the MNIST dataset. The experimental setup is designed to emulate the SEA model (see Section~\ref{subsec:sea}), which provides a unified framework for stochastic and adversarial optimization.

We simulate the key properties of the SEA model by introducing a non-stationary data generation process.
\begin{itemize}
    \item Distributional Shift ($\Sigma_I^2$): To emulate the adversarial shift of the underlying data distribution ($\mathfrak{D}_t$), we create a non-stationary stream of MNIST digits. The class distribution is designed to change over time: one ``dominant class'' has a significantly higher sampling probability, and this dominant class cycles sequentially from $0$ through $9$. The rate of this change is controlled by a ``transition\_speed'' parameter, which generates the distributional shift component.
    \item Stochastic Behavior ($\sigma_I^2$): The stochastic component of the environment arises naturally from the data sampling. Even when the dominant class is fixed for a period, a different image from that class is randomly sampled from the MNIST dataset at each step. This randomness introduces variability in the loss functions, contributing to the stochastic component.
\end{itemize}
The experiment, an online multi-class logistic regression task, runs for a total of $T=2000$ rounds, with results averaged over $5$ independent runs. We compare the algorithm's performance against the same set of baseline algorithms described previously; our method is the only one that theoretically provides gradient-variation regret bounds. Performance is evaluated based on both the cumulative loss (Figure~\ref{fig:mnist_results}, Left) and the online classification accuracy (Figure~\ref{fig:mnist_results}, Right).

\begin{figure}[!t]
    \centering
    \includegraphics[width=0.98\linewidth]{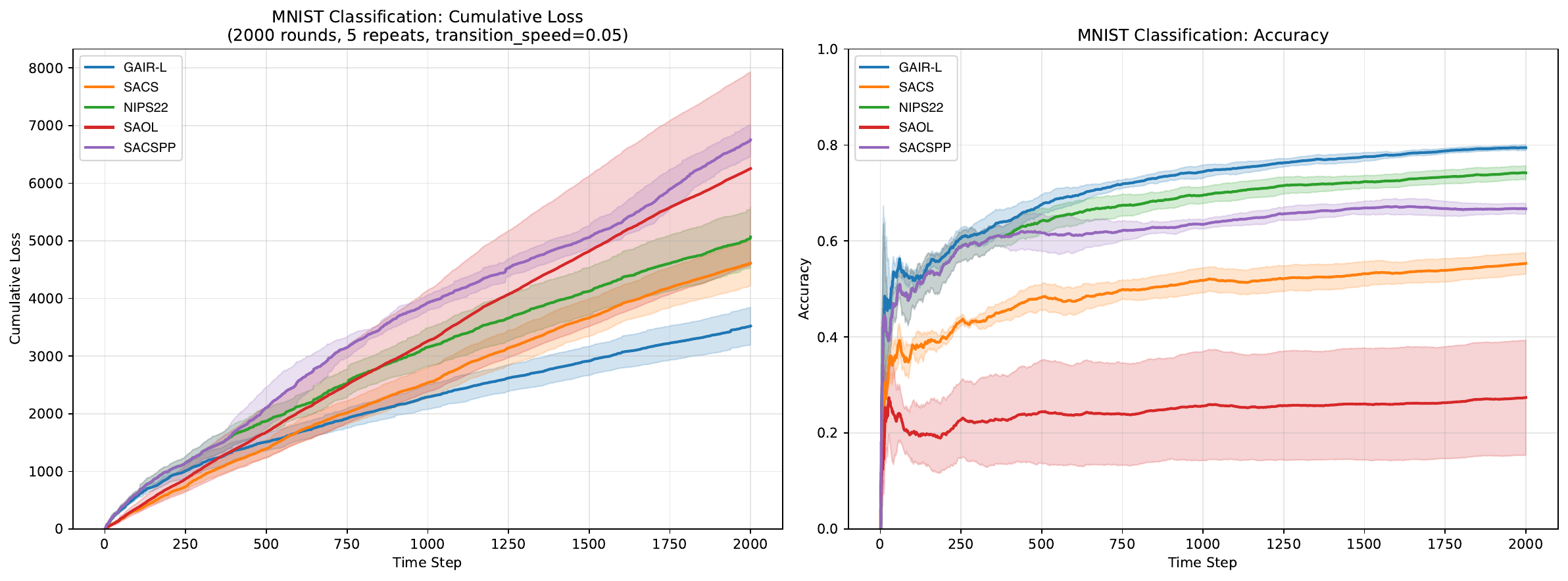}
    \caption{(Left) Cumulative loss comparison between GAIR-L and baseline algorithms. (Right) Online classification accuracy comparison between GAIR-L and baseline algorithms.}
    \label{fig:mnist_results}
\end{figure}
The empirical results demonstrate the robust performance of the proposed GAIR-L algorithm in the non-stationary MNIST classification task. As illustrated in Figure~\ref{fig:mnist_results} (Right), GAIR-L rapidly converges to the highest classification accuracy and maintains this level of performance throughout the experiment. The cumulative loss plot in Figure~\ref{fig:mnist_results} (Left) shows a corresponding trend: while GAIR-L does not have the lowest loss in the early phase, it achieves the lowest cumulative loss as learning progresses into the middle and later stages. This initial higher loss, paired with its strong accuracy, suggests that the algorithm's initial adaptation phase, during which it calibrates the ensemble weights, leads to temporarily higher loss but ultimately better long-term performance. This behavior is consistent with our design goal of attaining strong predictive accuracy in non-stationary environments.

\section{Conclusions}
\label{sec:conclusions}
In this paper, we develop a versatile online algorithm for interval regret minimization, named GAIR-L (Algorithm~\ref{alg:la-adaptive-oco}). Beyond handling non-stationarity, our method has two key advantages: it requires minimal parameter tuning, and it unifies several important results through the gradient variation quantity and algorithmic designs. For future work, one promising direction is to develop regret bounds that depend more comprehensively on gradient variation. Another is to explore the extension of our results to the setting of game optimization.

\bibliographystyle{IEEEtran}
\bibliography{online_learning}
\newpage
\onecolumn
{ \appendices
  \section{Proof of Section~\ref{sec:approach}}
\label{appendix:proof-our-approach}
This section proves the main results stated in Section~\ref{sec:approach}. The central step is to establish the gradient-variation interval regret bound in Theorem~\ref{thm:adaptive-oco} under the assumption that the Lipschitz constant is known. We then show that this bound implies a small-loss guarantee and yields a tighter worst-case interval bound. We begin the proof by first providing several key lemmas.

\subsection{Theoretical Guarantee of Base Algorithm}
The following lemma demonstrates that the online optimistic gradient descent algorithm in~\eqref{eq:base-algorithm-omd} adapts to empirical gradient variation under changing comparators, while remaining agnostic to the smoothness constant $L$.
\begin{myLemma}
    \label{thm:two-step-algorithm}
     Under Assumptions~\ref{ass:bounded-domain} and~\ref{ass:Lipschitzness}, by setting $\eta_t = (2D)/\sqrt{1 + \sum_{s=1}^{t-1} \norm{\g_s - M_s}^2}$, the two-step algorithm in~\eqref{eq:base-algorithm-omd} ensures the following dynamic regret bound for any $\u_1, \dots, \u_T \in \X$:
    \begin{align*}
        \sum_{t=1}^T \inner{\g_t}{\x_t - \u_t} \leq \left(P_T + \frac{5D}{2}\right) \sqrt{1 + \sum_{t=1}^T \norm{\g_t - M_t}_2^2}.
    \end{align*}
    where $P_T = \sum_{t=2}^T \norm{\u_t - \u_{t-1}}_2$ denotes the path length.
\end{myLemma}
\begin{proof}
    Following the classic analysis of the optimistic OMD~\cite[Theorem 1]{JMLR'24:Sword++}, we have:
    \begin{align*}
        \sum_{t=1}^T \inner{\g_t}{\x_t - \u_t} & \leq \underbrace{\sum_{t=1}^T \inner{\g_t - M_t}{\x_t - \xh_{t+1}}}_{\textsc{Term-A}} + \underbrace{\sum_{t=1}^T \frac{1}{2\eta_t} \left(\norm{\u_t - \xh_t}_2^2 - \norm{\u_t - \xh_{t+1}}_2^2\right)}_{\textsc{Term-B}} \\
                                               & \quad - \underbrace{\sum_{t=1}^T \frac{1}{2\eta_t}\left(\norm{\x_t - \xh_{t+1} }^2 + \norm{\x_t - \xh_t}^2\right)}_{\textsc{Term-C}}.
    \end{align*}
    For \textsc{Term-A}, by applying H\"{o}lder's inequality and the AM-GM inequality, we attain:
    \begin{align*}
        \textsc{Term-A} & \leq \sum_{t=1}^T \frac{\eta_{t+1}}{2} \norm{\g_t - M_t}^2 + \frac{1}{2\eta_{t+1}}\norm{\x_t - \xh_{t+1}}^2                                                                          \\
                        & = \sum_{t=1}^T \frac{D\cdot \norm{\g_t - M_t}^2}{\sqrt{1 + \sum_{s=1}^{t} \norm{\g_s - M_s}^2}} + \frac{1}{2\eta_{t+1}}\norm{\x_t - \xh_{t+1}}^2                                     \\
                        & \leq 2D\sqrt{1 + \sum_{t=1}^T \norm{\g_t - M_t}^2} + \underbrace{\frac{1}{2\eta_{t+1}}\norm{\x_t - \xh_{t+1}}^2}_{\textsc{Term-D}}, \tag*{(Lemma~\ref{lemma:self-confident-tuning})}
    \end{align*}
    where the remaining \textsc{Term-D} will be canceled by \textsc{Term-C}.

    For \textsc{Term-B}, observe that for any $\x,\y,\z \in \X$, under the bounded-domain assumption, the squared Euclidean distance is 4D-Lipschitz on $\X$, i.e., $\bigl|\|\x-\z\|_2^{2}-\|\y-\z\|_2^{2}\bigr| \le 4D\,\|\x-\y\|_2$. Consequently, with a non-increasing step size tuning such that $\eta_{t+1} \leq \eta_t$, we attain
    \begin{align*}
        \textsc{Term-B} & \leq  \frac{1}{2\eta_1}\norm{\u_1 - \xh_1}_2^2 + \sum_{t=2}^T  \frac{1}{2\eta_t}\norm{\u_t - \xh_t}_2^2 - \frac{1}{2\eta_{t-1}}\norm{\u_{t-1} - \xh_t}_2^2                     \\
                        & = \frac{1}{2\eta_1}\norm{\u_1 - \xh_1}_2^2 + \sum_{t=2}^T  \left(\frac{1}{2\eta_t}\norm{\u_t - \xh_t}_2^2 -\frac{1}{2\eta_{t}}\norm{\u_{t-1} - \xh_t}_2^2\right)               \\
                        & \quad \quad + \sum_{t=2}^T\left(\frac{1}{2\eta_{t}}\norm{\u_{t-1} - \xh_t}_2^2- \frac{1}{2\eta_{t-1}}\norm{\u_{t-1} - \xh_t}_2^2\right)                                        \\
                        & \leq \frac{1}{2\eta_1}\norm{\u_1 - \xh_1}_2^2 +  \sum_{t=2}^T \frac{2D}{\eta_t}\norm{\u_t - \u_{t-1}}_2 + D^2 \sum_{t=2}^T \big(\frac{1}{2\eta_t} - \frac{1}{2\eta_{t-1}}\big) \\
                        & \leq \frac{2DP_T}{\eta_T} + \frac{D^2}{2\eta_T} = \frac{4DP_T + D^2}{2\eta_T} \leq \left(P_T + \frac{D}{4}\right) \sqrt{1 + \sum_{t=1}^T \norm{\g_t - M_t}_2^2}.
    \end{align*}
    It remains to analyze \textsc{Term-D} with \textsc{Term-C}, where again the bounded domain assumption is applied:
    \begin{align*}
        \textsc{Term-D} - \textsc{Term-C} \leq \sum_{t=1}^T \left(\frac{1}{2\eta_{t+1}} - \frac{1}{2\eta_t}  \right) \norm{\x_{t} - \xh_{t+1}}_2^2 \leq \frac{D^2}{2\eta_{T+1}} = \frac{D}{4}\sqrt{1 + \sum_{t=1}^T \norm{\g_t - M_t}_2^2}.
    \end{align*}
    Combining the above results, we attain the desired bound:
    \begin{align*}
        \sum_{t=1}^T \inner{\g_t}{\x_t - \u_t} \leq \left(P_T + \frac{5D}{2}\right) \sqrt{1 + \sum_{t=1}^T \norm{\g_t - M_t}_2^2}.
    \end{align*}
\end{proof}

\subsection{Proof of Algorithm~\ref{alg:adaptive-oco}}
\label{subappendix:proof-adaptive-oco}
In this part, we provide the proof of our method under the standard assumptions of the OCO framework, specifically that the Lipschitz and smoothness constants are known to the learner. We focus on the interval regret~\cite{ICML'15:Daniely-adaptive}, which requires the algorithm to achieve nearly optimal performance on every interval when compared against any fixed comparator. We will prove Theorem~\ref{thm:adaptive-oco}, the guarantee of Algorithm~\ref{alg:adaptive-oco}.
It is worth noting that in Appendix~\ref{appendix:proof-interval-dynamic-regret}, we provide guarantees for Algorithm~\ref{alg:la-adaptive-oco} under a more general setting in which the Lipschitz constant is unknown a priori and the comparator sequences may vary over time arbitrarily. When restricted to the standard setting, Algorithm~\ref{alg:la-adaptive-oco} reduces exactly to Algorithm~\ref{alg:adaptive-oco}, which is the focus of this section. Moreover, in Lemma~\ref{lemma:interval-dynamic-regret-untuned} we establish an important intermediate result. The subsequent proof builds upon this result, but is presented in the simpler, standard setting for clarity. First, we provide a key lemma that is derived from Lemma~\ref{lemma:interval-dynamic-regret-untuned}.

\begin{myLemma}
    \label{lemma:adaptive-regret-untuned}
     Under the same assumptions as in Theorem~\ref{thm:adaptive-oco}, for any interval $[r, s] \subseteq [1, T]$, there exists a time $\tau \in [r, s]$ such that Algorithm~\ref{alg:adaptive-oco} achieves the following interval regret bound for any comparator $\u \in \X$:
    \begin{align*}
        \sum_{t=r}^s f_t(\x_t) -  \sum_{t=r}^s f_t(\u)
        &\leq \left[ D \cdot \left(3\sqrt{\ln \left(3 + \frac{4}{\mathcal{T}_{s_1}}F_{[1, s]}\right)} + {\Gamma}_{s}\right) + \frac{5D}{2}\right] \cdot \sqrt{1 + \log_2\left(1 + \frac{4}{\mathcal{T}_{s_1}}F_{[r, s]}\right)}\cdot \sqrt{\bar{V}_{[\tau, s]}} \notag\\
        &\quad + \left[ D \cdot \left(3\sqrt{\ln \left(3 + \frac{4}{\mathcal{T}_{s_1}}F_{[1, s]} \right)} + {\Gamma}_{s}\right) + \frac{5D}{2} + 6{\Gamma}_{s}GD \right] \cdot \log_2\left(1 + \frac{4}{\mathcal{T}_{s_1}}F_{[r, s]}\right) \\
        &\quad +6GD\log_2\left(1 + \frac{4}{\mathcal{T}_{s_1}}F_{[r, s]}\right)\ln \left(3 + \frac{4}{\mathcal{T}_{s_1}}F_{[1, s]}  \right) + 3GD - \sum_{t\in[\tau, s]}\D_{f_t} (\u, \x_t)\\
        &\quad + \mathcal{T}_{\text{marker before } r} \\
        &= \O\left(\sqrt{\bar{V}_{[\tau,s]}\cdot \log F_{[1,s]}\cdot \log F_{[r,s]}} - \sum_{t\in[\tau, s]}\D_{f_t} (\u, \x_t)\right),
    \end{align*}
    where $\bar{V}_{[a,b]}$ is the empirical gradient variation on the interval $[a,b]$, defined in Eq.~\eqref{eq:empirical-gradient-variation}, $F_{[a,b]} = \min_{\x \in \X} \sum_{t=a}^b f_t(\x)$ denotes the small loss on the interval $[a,b]$, ${\Gamma}_{s} = \O(\log\log T)$ is defined in~\eqref{eq:first-gamma-def}, $\mathcal{T}_{s_1} = \mathcal{G}(0, 0)$ is the threshold value when registering the first marker $s_1$ with the threshold generation function defined in~\eqref{eq:threshold-function-known-G}.
    We use $\mathcal{T}_{\text{marker before } r}$ to denote the threshold value when registering the marker just before time $r$. This value can be upper bounded as:
    \begin{align*}
        \mathcal{T}_{\text{marker before } r}
        &\leq 56LD^2\left[ 3\sqrt{\ln\left(3 + \frac{4}{\mathcal{T}_{s_1}}F_{[1, s]}\right)} + {\Gamma}_{s} + \frac{5}{2} \right]^2 \\
        &\quad + 2D\left[ 3\sqrt{\ln\left(3 + \frac{4}{\mathcal{T}_{s_1}}F_{[1, s]}\right)} + {\Gamma}_{s} \right]+ 5D + 3\left( 3 \cdot GD \left( {\Gamma}_{s} + \ln\left(3 + \frac{4}{\mathcal{T}_{s_1}}F_{[1, s]}\right) \right) \right) \\
        &= \O\left( \log F_{[1,s]} + \Gamma_{s}^2 \right).
    \end{align*}
\end{myLemma}

\begin{proof}
    This proof follows from Lemma~\ref{lemma:interval-dynamic-regret-untuned} in Appendix~\ref{appendix:proof-interval-dynamic-regret}.

    We verify that when the Lipschitz constant is known, Algorithm~\ref{alg:la-adaptive-oco} reduces to Algorithm~\ref{alg:adaptive-oco}.
    Specifically, under the assumptions of Theorem~\ref{thm:adaptive-oco} and the configurations of Algorithm~\ref{alg:adaptive-oco}, one can verify that $\max_{t \in [T]} \norm{\r_t - \m_t}_{\infty} \leq 2GD$. Because we have prior knowledge of the Lipschitz constant $G$, we can set $B_0 = 2GD$ in Algorithm~\ref{alg:la-adaptive-oco}, and $B_0 = B_1 = \dots = B_T$ remain fixed. Therefore, the quantity $\tilde{\Gamma}_{t}$ reduces to ${\Gamma}_{t}$, and the threshold generation function in~\eqref{eq:la-threshold-function} reduces to:
    \begin{align*}
        \mathcal{G}(t, i) = 56LD^2\left(\frac{3 \ln (2i + 1) + {\Gamma}_{t}}{\sqrt{\ln (2 i + 1)}} + \frac{5}{2}\right)^2 + 2D\frac{3 \ln (2i + 1) + {\Gamma}_{t}}{\sqrt{\ln (2 i + 1)}} + 5D+ 9({\Gamma}_{t} + \ln (2 i + 1))GD,
    \end{align*}
    where, for any time $t \in [T]$, we define $s_i$ as the largest marker not exceeding $t$, i.e., $s_i \leq t < s_{i+1}$. This generation function is exactly the same as that in Algorithm~\ref{alg:adaptive-oco}.

    Secondly, it is straightforward to verify that the meta algorithm used in Algorithm~\ref{alg:la-adaptive-oco} reduces exactly to the Optimistic Adapt-ML-Prod algorithm used in Algorithm~\ref{alg:adaptive-oco}.

    Finally, we consider benchmarking Algorithm~\ref{alg:adaptive-oco} against any fixed comparator $\u \in \X$ on any interval $I = [r,s] \subseteq [T]$. In this case, we can choose $\u_r = \u_{r+1} = \dots = \u_s = \u$, and thus the path length $P_{[r,s]} = 0$.
    Hence, we can apply Lemma~\ref{lemma:interval-dynamic-regret-untuned} to this setting, and we complete the proof.
\end{proof}

\begin{proof}[Proof of Theorem~\ref{thm:adaptive-oco}]
Based on Lemma~\ref{lemma:adaptive-regret-untuned}, we can prove Theorem~\ref{thm:adaptive-oco} as follows.

\paragraph{Small-loss bound} We show that the empirical gradient variation $\bar{V}_{[\tau, s]}$ can be bounded by the sum of gradient norms, which further implies the small-loss bound. By Lemma~\ref{lemma:adaptive-regret-untuned}, we have:
\begin{align}
     \sum_{t=r}^s f_t(\x_t) -  \sum_{t=r}^s f_t(\u) &\leq \O\left(\sqrt{\bar{V}_{[\tau,s]}\cdot \log F_{[1,s]}\cdot \log F_{[r,s]}} - \sum_{t\in[\tau, s]}\D_{f_t} (\u, \x_t)\right) \notag \\
     &\leq \O\left(\sqrt{\bar{V}_{[r,s]}\cdot \log F_{[1,s]}\cdot \log F_{[r,s]}}\right) \notag\\
     &\leq \O\left(\sqrt{\Big(\sum_{t=r}^s \norm{\nabla f_t(\x_t)}_2^2\Big)\cdot \log F_{[1,s]}\cdot \log F_{[r,s]}}\right) \notag \\
     &\leq \O\left(\sqrt{\Big(\sum_{t=r}^s f_t(\x_t)\Big)\cdot \log F_{[1,s]}\cdot \log F_{[r,s]}}\right), \label{eq:oco-small-loss-proof}
\end{align}
where in the third step we use the fact that
\begin{align*}
    \bar{V}_{[r,s]} = \sum_{t = r+1}^s \norm{\nabla f_t(\x_t) - \nabla f_{t-1}(\x_{t-1})}_2^2 \leq 2\sum_{t=r}^s \norm{\nabla f_t(\x_t)}_2^2 + 2\sum_{t=r}^{s-1} \norm{\nabla f_t(\x_{t})}_2^2 \leq 4\sum_{t=r}^s \norm{\nabla f_t(\x_t)}_2^2,
\end{align*}
and in the last step we use the property of smooth functions by Lemma~\ref{lemma:self-bounded}. Then based on Eq.~\eqref{eq:oco-small-loss-proof}, we can derive the following inequality by applying Lemma~\ref{lemma:substitute-F_T}:
\begin{align*}
    \sum_{t=r}^s f_t(\x_t) -  \sum_{t=r}^s f_t(\u) &\leq \O\left(\sqrt{\Big(\sum_{t=r}^s f_t(\u)\Big)\cdot \log F_{[1,s]}\cdot \log F_{[r,s]}} + \log F_{[1,s]}\cdot \log F_{[r,s]}\right),
\end{align*}
which finishes the proof of the small-loss bound by following the convention of~\cite{ICML19:Zhang-Adaptive-Smooth,NeurIPS'22:efficient} to hide the logarithmic terms in the big-O notation.

\paragraph{Gradient-variation bound} We now prove that by carefully using the negative Bregman divergence terms, we can derive the gradient-variation interval bound. Notice that the empirical gradient variation $\bar{V}_{[\tau, s]}$ can be bounded in the following way:
\begin{align*}
    \bar{V}_{[\tau, s]} & = \sum_{t=\tau+1}^s \norm{\nabla f_t(\x_t) - \nabla f_{t-1}(\x_{t-1})}_2^2\\
    &\leq 3\sum_{t=\tau+1}^s\left(\norm{\nabla f_t(\x_t) - \nabla f_{t}(\u)}_2^2 + \norm{\nabla f_{t}(\u) - \nabla f_{t-1}(\u)}_2^2 + \norm{\nabla f_{t-1}(\x_{t-1}) - \nabla f_{t-1}(\u)}_2^2\right)\\
    &\leq 3\sum_{t=\tau+1}^s\left(2L\D_{f_t}(\u,\x_t) + \norm{\nabla f_{t}(\u) - \nabla f_{t-1}(\u)}_2^2 + 2L\D_{f_{t-1}}(\u,\x_{t-1})\right) \tag*{(Proposition~\ref{prop:bregman-divergence-property})}\\
    &\leq 3\sum_{t=\tau+1}^s\left(2L\D_{f_t}(\u,\x_t) + \sup_{\x \in \X}\norm{\nabla f_{t}(\x) - \nabla f_{t-1}(\x)}_2^2 + 2L\D_{f_{t-1}}(\u,\x_{t-1})\right).
\end{align*}
Therefore, again by Lemma~\ref{lemma:adaptive-regret-untuned}, we have:
\begin{align*}
     \sum_{t=r}^s f_t(\x_t) - \sum_{t=r}^s f_t(\u) &\leq \O\left(\sqrt{\bar{V}_{[\tau,s]}\cdot \log F_{[1,s]}\cdot \log F_{[r,s]}} - \sum_{t=\tau}^s\D_{f_t} (\u, \x_t)\right)\\
     &\leq \O\left(\sqrt{V_{[\tau,s]}\cdot \log F_{[1,s]}\cdot \log F_{[r,s]}} + \sqrt{\left(\sum_{t = \tau}^s \D_{f_t}(\u,\x_t)\right) \cdot \log F_{[1,s]}\cdot \log F_{[r,s]}} - \sum_{t=\tau}^s\D_{f_t} (\u, \x_t)\right)\\
     & \leq \O\left(\sqrt{V_{[r,s]}\cdot \log F_{[1,s]}\cdot \log F_{[r,s]}} +\log F_{[1,s]}\cdot \log F_{[r,s]}\right),
\end{align*}
where the last inequality follows from the fact that $\sqrt{ax} - x \leq \frac{a}{4}$.

\paragraph{Worst-case bound} From Lemma~\ref{lemma:interval-dynamic-interval-in-Ct}, for any interval $[s_i, s_j - 1] \in \tilde{\mathcal{S}}$, when the comparator sequence is fixed to $\u_t = \u$ and the Lipschitz constant $G$ is known, for any stopping time $\tau \in [s_i, s_j - 1]$, we have:
\begin{align*}
    \sum_{t=s_i}^{\tau} f_t(\x_t) - \sum_{t=s_i}^{\tau} f_t(\u) \leq \sum_{t=s_i}^{\tau} \inner{\nabla f_t(\x_t)}{\x_t - \u} \leq \O\left(\sqrt{\bar{V}_{[s_i, \tau]}\cdot \log T} + \log T \right).
\end{align*}
Applying the same analysis as in the derivation of Eq.~\eqref{eq:oco-small-loss-proof}, we obtain:
\begin{align}
    \label{eq:worst-case-adaptive-proof}
     \sum_{t=s_i}^{\tau} f_t(\x_t) - \sum_{t=s_i}^{\tau} f_t(\u) \leq \O\left(\sqrt{\sum_{t=s_i}^{\tau} f_t(\x_t) \cdot \log T} + \log T\right).
\end{align}
Next, we conduct a standard argument to decompose the target interval $[r,s]$ into a series of consecutive intervals in $\tilde{\mathcal{S}}$. Let $s_p$ denote the smallest marker satisfying $s_p > r$, and $s_q$ the largest marker such that $s_q \le s$. Hence, $s_{p-1} \le r < s_p$ and $s_q \le s < s_{q+1}$. By Lemma~\ref{lemma:pcgc-number}, $I$ can be covered by at most $v$ consecutive intervals from the problem-dependent schedule $\tilde{\mathcal{S}}$ (defined in~\eqref{eq:problem-dependent-schedule}), where $v \le \lceil \log_2 (q - p + 2) \rceil$. Furthermore, since at most one marker can be registered at each time step, it follows that
\begin{align}
    \label{eq:worst-case-v-bound}
    q - p \le s - r , \quad v \leq \lceil \log_2 (q - p + 2) \rceil \leq \O\left(\log (s-r)\right).
\end{align}

Formally, we have
\begin{align*}
I_1 = [s_{i_1}, s_{i_2} - 1],
I_2 = [s_{i_2}, s_{i_3} - 1],
\ldots,
I_v = [s_{i_v}, s_{i_{v+1}} - 1],
\end{align*}
with $i_1 = p$ and $i_v \le q < i_{v+1}$, which together yield
$I \subseteq [s_{p-1}, s_p - 1] \cup \bigcup_{j=1}^v I_j$.

Therefore the regret on $[r, s]$ can be decomposed into three parts: $[r, s_{p} - 1]$, $[s_{i_1}, s_{i_{v}} - 1]$, and $[s_{i_v}, s]$. For the first part, we have:
\begin{align}
    \label{eq:worst-case-t-gd}
    \sum_{t=r}^{s_{p} - 1} f_t(\x_t) - \sum_{t=r}^{s_{p} - 1} f_t(\u) &\leq \sum_{t=r}^{s_{p} - 1} f_t(\x_t) \leq \sum_{t=s_{p - 1}}^{s_{i_1} - 1} f_t(\x_t) = \sum_{t=s_{p - 1}}^{s_{i_1} - 2} f_t(\x_t) + f_{s_{i_1} - 1}(\x_{s_{i_1} - 1}) \leq \mathcal{T}_{s_{p}} + G D ,
\end{align}
where the first two inequalities follow from the non-negativity of the loss functions. The last inequality follows from the thresholding mechanism in Algorithm~\ref{alg:adaptive-oco}, which monitors the cumulative loss, and the assumption on the loss function values. For notational simplicity, we denote the threshold value used to register the marker $s_{i+1}$ by $\mathcal{T}_{s_{i+1}}$, i.e., $$\mathcal{T}_{s_{i+1}} \triangleq \mathcal{G}(s_{i+1} - 1, i).$$

For the second part, corresponding to the interval $[s_{i_1}, s_{i_v} - 1]$, we consider the regret on each sub-interval $I_k = [s_{i_k}, s_{i_{k+1}} - 1] \in \tilde{\mathcal{S}}$, where $k \in [v]$. Each interval $I_k$ contains $\abs{i_{k+1} - i_k}$ markers, allowing us to apply the marker–threshold mechanism to upper bound the cumulative loss $\sum_{t \in I_k} f_t(\x_t)$ of the final decisions at each round, and consequently, to bound the regret as in Eq.~\eqref{eq:worst-case-adaptive-proof}. Specifically, we have:
\begin{align}
    \label{eq:worst-case-interval-bound-t-g}
    \sum_{t=s_{i_k}}^{s_{i_{k+1}} - 1} f_t(\x_t) = \sum_{a = i_k}^{i_{k+1} - 1} \sum_{t=s_a}^{s_{a+1} - 1} f_t(\x_t) \leq \sum_{a = i_k}^{i_{k+1} - 1} \left(\mathcal{T}_{s_{a+1}} + GD\right) \leq (i_{k+1} - i_k) \left(\max_{a \in [s_{i_k}, s_{i_{k+1}}]} \mathcal{T}_{a} + GD\right),
\end{align}
where the last inequality follows from the same reasoning as in~\eqref{eq:worst-case-t-gd} and the monotonicity of the threshold values.

By Lemma~\ref{lemma:pcgc-number}, for consecutive intervals in $\tilde{\mathcal{S}}$, we have $i_k - i_{k-1} \leq \frac{1}{2}(i_{k+1} - i_{k})$ for any $k \geq 2$. Therefore, by applying Eq.~\eqref{eq:worst-case-adaptive-proof} to each interval $I_k$ and applying~\eqref{eq:worst-case-interval-bound-t-g}, we obtain:
\begin{align*}
    \sum_{t=s_{i_1}}^{s_{i_v} - 1} f_t(\x_t) - \sum_{t=s_{i_1}}^{s_{i_v} - 1} f_t(\u) &\leq \O\left(\sum_{k=1}^{v-1} \sqrt{\sum_{t=s_{i_k}}^{s_{i_{k+1}} - 1} f_t(\x_t) \cdot \log T} + \log T\right) \\
    &\leq  \O\left(\sum_{k=1}^{v-1} \sqrt{(GD+\max_{a \in [s_{i_1}, s_{i_{v}}]} \mathcal{T}_{a}) \cdot (i_{k+1} - i_k)\cdot \log T} + \log T\right)\\
    &\leq  \O\left(\sum_{k=1}^{v-1} \sqrt{(GD+\max_{a \in [s_{i_1}, s_{i_{v}}]} \mathcal{T}_{a}) \cdot 2^{-(v-k-1)} (i_{v} - i_{v-1})\cdot \log T} + \log T\right) \tag*{(Lemma~\ref{lemma:pcgc-number})}\\
    &\leq \O\left( \sqrt{(GD+\max_{a \in [s_{i_1}, s_{i_{v}}]} \mathcal{T}_{a}) \cdot \log T}\cdot\left(\sum_{k=-\infty}^{v-1}\sqrt{2^{-(v-k-1)}\cdot (i_{v} - i_{v-1})}\right) + v\log T\right)\\
    &=\O\left( \sqrt{\max_{a \in [s_{i_1}, s_{i_{v}}]} \mathcal{T}_{a} \cdot \log T \cdot (i_{v} - i_{v-1})} + \log\abs{I}\log T\right),
\end{align*}
where the last step follows from the summation of a geometric series.

Ignoring the doubly-logarithmic terms, the threshold function in Eq.~\eqref{eq:threshold-function-known-G} implies that $\mathcal{T}_{s_{i+1}} = \Theta(\log (1 + i) + 1)$. Assuming that $\max_{t \in [T]} \max_{\x \in \X} f_t(\x) \leq GD$, during $[s_{i_1}, s_{i_v} - 1]$, the number of generated markers is at most:
\begin{align*}
    i_v - i_{v-1}  \leq \frac{\sum_{t=r}^s f_t(\x_t)}{\mathcal{T}_{s_{i_{v-1}}}} \leq \frac{GD\abs{I}}{\mathcal{T}_{s_{i_{v-1}}}} = \Theta \left(\frac{GD\abs{I}}{\log (1 + i_{v-1}) + 1}\right),
\end{align*}
i.e., the cumulative loss divided by the threshold value at the beginning of the interval. By noticing that: $\max_{a \in [s_{i_1}, s_{i_{v}}]} \mathcal{T}_{a} = \Theta(\log (1 + i_v) + 1)$, we can conclude that:
\begin{align*}
     \sum_{t=s_{i_1}}^{s_{i_v} - 1} f_t(\x_t) - \sum_{t=s_{i_1}}^{s_{i_v} - 1} f_t(\u) &\leq \O\left( \sqrt{\max_{a \in [s_{i_1}, s_{i_{v}}]} \mathcal{T}_{a} \cdot \log T \cdot (i_{v} - i_{v-1})} + \log\abs{I}\log T\right)\\
     &= \O\left(\sqrt{\abs{I}\cdot \frac{\log (1 + i_v) + 1}{\log (1 + i_{v-1}) + 1}\log T} + \log\abs{I}\log T\right)\\
     &\leq \O\left(\sqrt{\abs{I}\cdot \frac{\log (1 + 2i_{v-1}) + 1}{\log (1 + i_{v-1}) + 1}\log T} + \log\abs{I}\log T\right)\\
     &=\Theta\left(\sqrt{\abs{I}\cdot \log T} + \log\abs{I}\log T\right),
\end{align*}
where in the third line we use the fact of the schedule that $i_v \leq 2i_{v-1}$.

For the third part, corresponding to the interval $[s_{i_v}, s]$, by applying Eq.~\eqref{eq:worst-case-adaptive-proof}, we have:
\begin{align*}
    \sum_{t=s_{i_v}}^{s } f_t(\x_t) - \sum_{t=s_{i_v}}^{s} f_t(\u) &\leq \O\left(\sqrt{\sum_{t=s_{i_v}}^{s} f_t(\x_t) \cdot \log T} + \log T\right) \leq \O\left(\sqrt{\abs{I} \cdot \log T} + \log T\right).
\end{align*}

Combining the above three parts, we obtain the worst-case regret bound:
\begin{align*}
    \sum_{t=r}^{s} f_t(\x_t) - \sum_{t=r}^{s} f_t(\u) &\leq \O\left(\sqrt{\abs{I}\cdot \log T} + \log\abs{I}\log T\right) = \O\left(\sqrt{(\abs{I} + \log^2 \abs{I}\log T) \cdot \log T}\right).
\end{align*}

It remains to discuss whether the additional $\O(\log^2 \abs{I}\log T)$ term can be absorbed into the dominant term $\O(\abs{I})$ asymptotically. To ensure the interval regret is meaningful, we restrict attention to non-trivial intervals, i.e., those with $\abs{I} = \Omega(\log T)$. Otherwise, the above bound does not imply the no-regret guarantee.

For $\abs{I} = \Omega(\log T)$, we consider the following two cases:
\begin{enumerate}[label=(\roman*)]
    \item When $|I| = \Theta(T^\alpha)$ with $\alpha \in (0, 1]$, we have $\log|I| = \Theta(\log T)$. Thus,
    $$
        \frac{\log T \cdot \log^2 |I|}{|I|} = \frac{\Theta(\log T) \cdot \Theta(\log^2 T)}{\Theta(T^\alpha)} = \Theta\left(\frac{\log^3 T}{T^\alpha}\right) = o(1).
    $$
   This ratio tends to zero as $T \to \infty$.

    \item When $|I| = \Theta(\log^\beta T)$ with $\beta > 1$, we have $\log|I| = \Theta(\log\log T)$. Thus,
    $$\frac{\log T \cdot \log^2 |I|}{|I|} = \frac{\Theta(\log T) \cdot \Theta((\log\log T)^2)}{\Theta(\log^\beta T)} = \Theta\left(\frac{(\log\log T)^2}{\log^{\beta-1} T}\right).$$
    This ratio tends to zero for any $\beta > 1$.
\end{enumerate}

In both cases, we have shown that $\log T \cdot \log^2 |I| = o(|I|)$. This implies $|I| + \log T \cdot \log^2 |I| = \Theta(|I|)$. Substituting this back into the original expression gives the desired worst-case bound.

\end{proof}

  \section{Proof of Meta Algorithm in Section~\ref{sec:lipschitz_adaptive}}
\label{appendix:proof-la}
In this section, we establish theoretical guarantees for the proposed Lipschitz-adaptive meta-algorithm. Beyond the adaptive-regret results, these guarantees may be of independent interest. To further enhance the theorem's practicality and generality, we prove guarantees under the sleeping-expert paradigm.
\subsection{Key Lemmas of Lipschitz-Adaptive Meta Algorithm}
The following lemma clarifies the extra analysis required to apply the classical Prod algorithm's learning rates~\cite{COLT'14:second-order-Hedge,NIPS'16:Wei-non-stationary-expert} in the Lipschitz-adaptive setting.
\begin{myLemma}
  \label{lemma:time-varying-lr}
  For any $i \in [N]$, let the corresponding active interval be $[r,s]$. Then, with the learning rate tuned as
  \begin{align*}
    \eta_{t+1, i} = \min\left\{\frac{1}{2B_t}, \sqrt{\frac{\gamma_{i}}{B_t^2 + \sum_{\tau=r}^t (\bar{r}_{\tau,i} - m_{\tau,i})^2 }}\right\},
  \end{align*}
  the cumulative deviation of the learning rates is bounded by:
  \begin{align*}
    \sum_{t=r}^{s} \left(\frac{\eta_{t,i}}{\eta_{t+1,i}} - 1\right) & \leq  \frac{B_s^2}{B_{r-1}^2} + \frac{1}{2} \ln \left(1 + (s-r+1)\frac{B_s^2}{B_{r-1}^2}\right) + \ln\left(\frac{B_s}{B_{r-1}}\right) + \frac{\max_{t \in [r, s]} \abs{B_t - B_{t-1}}}{B_{r-1}} \\
                                                                    & \leq \O \left(\frac{B_s^2 }{B_{r-1}^2} +  \log \left(s\frac{B_s}{B_{r-1}}\right) +  \frac{\max_{t \in [r, s]} \abs{B_t - B_{t-1}}}{B_{r-1}}  \right).
  \end{align*}
\end{myLemma}
\begin{proof}
  For simplicity, we introduce $\Delta_{t,i} = B_t^2 + \sum_{\tau=r}^t (\bar{r}_{\tau,i} - m_{\tau,i})^2 $, and the learning rate is $\eta_{t+1, i} = \min\left\{\frac{1}{2B_{t}}, \sqrt{\frac{\gamma_{i}}{\Delta_{t, i}  }}\right\}$. We bound the ratio $\eta_{t,i}/\eta_{t+1,i}$ by case analysis:
  \begin{enumerate}[label=(\roman*)]
    \item $\eta_{t+1,i} = \frac{1}{2B_t}$, $\eta_{t,i} = \frac{1}{2B_{t-1}}$. In this case,
          \begin{align*}
            \frac{\eta_{t,i}}{\eta_{t+1,i}} = \frac{B_t}{B_{t-1}} .
          \end{align*}
    \item  $\eta_{t+1,i} = \frac{1}{2B_t}$, $\eta_{t,i} = \sqrt{\frac{\gamma_i}{\Delta_{t-1, i}  }}$. In this case, $ \eta_{t,i} \leq \frac{1}{2B_{t-1}}$, which implies:
          \begin{align*}
            \frac{\eta_{t,i}}{\eta_{t+1,i}} \leq \frac{B_t}{B_{t-1}}.
          \end{align*}
    \item $\eta_{t+1,i} = \sqrt{\frac{\gamma_i}{\Delta_{t, i}   }}$, $\eta_{t,i} = \frac{1}{2B_{t-1}}$. In this case, $\eta_{t,i} \leq  \sqrt{\frac{\gamma_i}{\Delta_{t-1, i}  }}$. Then we have,
          \begin{align*}
            \frac{\eta_{t,i}}{\eta_{t+1,i}} \leq \sqrt{\frac{\Delta_{t,i}   }{\Delta_{t-1,i}  }}.
          \end{align*}
    \item $\eta_{t+1,i} = \sqrt{\frac{\gamma_i}{\Delta_{t, i}   }}$, $\eta_{t,i} = \sqrt{\frac{\gamma_i}{\Delta_{t-1, i}  }}$. Directly,
          \begin{align*}
            \frac{\eta_{t,i}}{\eta_{t+1,i}} = \sqrt{\frac{\Delta_{t,i}   }{\Delta_{t-1,i}  }}.
          \end{align*}
  \end{enumerate}
  With the aforementioned discussion and the non-increasing property of the learning rates, we can conclude that:
  \begin{align*}
    0 \leq \frac{\eta_{t,i}}{\eta_{t+1,i}} - 1 \leq \max\left\{\frac{B_t}{B_{t-1}} - 1, \sqrt{\frac{\Delta_{t,i}   }{\Delta_{t-1,i}  }}  - 1 \right\},
  \end{align*}
  which further implies,
  \begin{align*}
    \sum_{t=r}^s \frac{\eta_{t,i}}{\eta_{t+1,i}} - 1 \leq \sum_{t=r}^s\left(\frac{B_t}{B_{t-1}} - 1\right) + \sum_{t=r}^s \left(\sqrt{\frac{\Delta_{t,i}   }{\Delta_{t-1,i}  }}  - 1\right).
  \end{align*}
  For the first term on the right-hand side,
  \begin{align*}
    \sum_{t=r}^s\left(\frac{B_t}{B_{t-1}} - 1\right) & = \sum_{t=r}^s\frac{B_t - B_{t-1}}{B_{t-1}} = \sum_{t=r}^s \frac{B_t - B_{t-1}}{B_{r-1}+\sum_{\tau=1}^{t-1} (B_\tau - B_{\tau-1})}                         \\
                                                     & \leq \frac{\max_{t \in [r, s]} \abs{B_t - B_{t-1}}}{B_{r-1}} + \ln\left(\frac{B_s}{B_{r-1}}\right). \tag*{(Lemma~\ref{lemma:self-confident-int})}
  \end{align*}
  For the second term, following the standard analysis of~\cite{COLT'14:second-order-Hedge}, we have:
  \begin{align}
    \sum_{t=r}^s  \left(\sqrt{\frac{\Delta_{t,i}}{\Delta_{t-1,i}}} - 1\right)
     & =\sum_{t=r}^s \left(\sqrt{\frac{ B_t^2 + \sum_{\tau=r}^{t} (\bar{r}_{\tau,i} - m_{\tau,i})^2 }{B_{t-1}^2 + \sum_{\tau=r}^{t-1}  (\bar{r}_{\tau,i} - m_{\tau,i})^2} } - 1\right)\notag                                                            \\
     & = \sum_{t=r}^s \left(\sqrt{1 +  \frac{ (\bar{r}_{t,i} - m_{t,i})^2 + B_t^2 - B_{t-1}^2 }{B_{r-1}^2 +\sum_{\tau=r}^{t-1}  (\bar{r}_{\tau,i} - m_{\tau,i})^2 +  B_\tau^2 - B_{\tau-1}^2} } - 1\right)\notag                                                   \\
     & \leq \sum_{t=r}^s \left(\frac{1}{2}\cdot \frac{  (\bar{r}_{t,i} - m_{t,i})^2 + B_t^2 - B_{t-1}^2}{B_{r-1}^2 + \sum_{\tau=r}^{t-1}  (\bar{r}_{\tau,i} - m_{\tau,i})^2 + B_\tau^2 - B_{\tau-1}^2} +  1 - 1 \right) \tag*{($\sqrt{1+x} \leq 1 + \frac{x}{2}$)} \\
     & \leq \frac{B_s^2}{B_{r-1}^2} + \frac{1}{2} \ln\left(B_{r-1}^2 + \sum_{t=r}^s (\bar{r}_{t,i} - m_{t,i})^2\right) - \frac{1}{2}\ln \left(B_{r-1}^2\right)\tag*{(Lemma~\ref{lemma:self-confident-int})}                                             \\
     & \leq \frac{B_s^2}{B_{r-1}^2} + \frac{1}{2} \ln \left(1 + (s-r+1)\frac{B_s^2}{B_{r-1}^2}\right). \label{eq:where-we-can-improve}
  \end{align}

  Combining the above, we conclude the result.
\end{proof}

The following lemma serves as an intermediate step in regret analysis, employing a potential-based argument following~\cite{NIPS'16:Wei-non-stationary-expert}.
\begin{myLemma}
  \label{lemma:regret-upper-bound}
  Under the same assumptions as in Theorem~\ref{thm:meta-algorithm}, given the assumption that $\sum_{i \in A_t} p_{t,i}\bar{r}_{t,i} \leq 0$ and $\eta_{t,i} \abs{\bar{r}_{t,i} - m_{t,i}} \leq 1/2$, then for any expert $i_\star \in \bigcup_{t=1}^s A_t$ with active interval $[r,s]$, the clipped regret satisfies:
  \begin{align*}
    \sum_{t=r}^s \bar{r}_{t,i_\star} \leq \frac{1}{\eta_{s+1, i_\star}}\ln \frac{1}{w_{r,i_\star}} + \sum_{t=r}^s \eta_{t,i_\star}(\bar{r}_{t,i_\star} - m_{t,i_\star})^2+ \frac{1}{\eta_{s+1, i_\star}}\ln \left( N_{s+1} + \frac{1}{e}\sum_{i\in \bigcup_{t=1}^s A_t}\sum_{t: i \in A_t }\left(\frac{\eta_{t, i}}{\eta_{t+1, i}}-1\right)\right),
  \end{align*}
  where $N_{s+1} = \abs{\bigcup_{t=1}^{s+1} A_t}$ denotes the number of experts initialized up to time $s+1$.
\end{myLemma}
\begin{proof}
  We introduce a potential
  \begin{align*}
    W_{s+1} = \sum_{i \in \bigcup_{t=1}^{s+1} A_t} w_{s+1,i}
  \end{align*}
  to represent the cumulation of weights for all initialized experts up to time $s+1$. Recall that we use $A_t$ to denote the indices of experts active at time $t$. We first consider an upper bound for $W_{s+1}$.

  By the inequality $x \leq x^{\alpha} + (\alpha -1)/e$ for $x > 0, \alpha > 0$, for any expert $i \in A_{s+1}$, we have:
  \begin{align}
    w_{s+1,i} \leq (w_{s+1,i})^{\frac{\eta_{s,i}}{\eta_{s+1,i}}} + \frac{1}{e} \left(\frac{\eta_{s,i}}{\eta_{s+1,i}} - 1\right). \label{eq:proof-meta-weight-process-1}
  \end{align}
  With the update formulation at Line~\ref{line:weight-update}, the first term above can be further bounded as:
  \begin{align}
    (w_{s+1,i})^{\frac{\eta_{s,i}}{\eta_{s+1,i}}} & = w_{s,i}\exp\left(\eta_{s,i}\bar{r}_{s, i} - \eta_{s,i}^2 (\bar{r}_{s,i} - m_{s,i})^2\right) \notag                   \\
                                                  & =\tilde{w}_{s,i}\exp\left(\eta_{s,i}(\bar{r}_{s, i} -m_{s,i} )- \eta_{s,i}^2 (\bar{r}_{s,i} - m_{s,i})^2\right) \notag \\
                                                  & \leq \tilde{w}_{s,i} (1 + \eta_{s,i}(\bar{r}_{s,i} - m_{s,i})), \label{eq:proof-meta-weight-process-2}
  \end{align}
  where the last inequality is by $\exp(x - x^2) \leq 1 + x$ for $x \geq -1/2$. This criterion is satisfied by the assumption in the lemma statement. Combining Eq.~\eqref{eq:proof-meta-weight-process-1} and Eq.~\eqref{eq:proof-meta-weight-process-2}, we can further analyze the weights for experts active at time $s$:
  \begin{align}
    \sum_{i \in A_{s}} w_{s+1,i} & \leq \sum_{i \in A_{s}}\tilde{w}_{s, i}\left(1+\eta_{s, i}\left(\bar{r}_{s, i}-m_{s, i}\right)\right) + \sum_{i \in A_{s}} \frac{1}{e}\left(\frac{\eta_{s, i}}{\eta_{s+1, i}}-1\right) \notag                                  \\
                                 & = \sum_{i \in A_{s}}\tilde{w}_{s, i}\left(1-\eta_{s, i}m_{s, i}\right) + \sum_{i \in A_{s}} \eta_{s, i} \tilde{w}_{s, i}\bar{r}_{s, i} + \sum_{i \in A_{s}} \frac{1}{e}\left(\frac{\eta_{s, i}}{\eta_{s+1, i}}-1\right) \notag \\
                                 & \leq \sum_{i \in A_{s}}\tilde{w}_{s, i}\exp(-\eta_{s, i}m_{s, i}) + \sum_{i \in A_{s}} \eta_{s, i} \tilde{w}_{s, i}\bar{r}_{s, i} + \sum_{i \in A_{s}} \frac{1}{e}\left(\frac{\eta_{s, i}}{\eta_{s+1, i}}-1\right)  \notag     \\
                                 & =  \sum_{i \in A_{s}}w_{s, i} + \Big(\sum_{j \in A_s} \eta_{s, j}\tilde{w}_{s, j}\Big) \sum_{i \in A_s} p_{s,i}\bar{r}_{s,i} + \sum_{i \in A_{s}} \frac{1}{e}\left(\frac{\eta_{s, i}}{\eta_{s+1, i}}-1\right)  \notag          \\
                                 & \leq \sum_{i \in A_{s}}w_{s, i} + \sum_{i \in A_{s}} \frac{1}{e}\left(\frac{\eta_{s, i}}{\eta_{s+1, i}}-1\right), \label{eq:proof-sum-weights-two-round}
  \end{align}
  where we apply $1 - x \leq \exp(-x)$ for any $x \in \R$ and the last inequality is by the assumption in the theorem statement that $\sum_{i \in A_t} p_{t,i}\bar{r}_{t,i} \leq 0$ for any $t\in[T]$. Next, we decompose $W_{s+1}$ into three parts: (i) weights of new experts initialized right at time $s+1$; (ii) weights of experts active at time $s$; (iii) weights of experts that fell asleep before time $s$:
  \begin{align}
     & W_{s+1} = \sum_{i \in A_{s+1}\setminus A_s} w_{s+1, i}+ \sum_{i \in A_s} w_{s+1, i} + \sum_{i\in A_{1:s+1}\setminus A_{s:s+1} }w_{s+1, i}\notag                                                                                                                                                                                             \\
     & \leq \abs{A_{s+1}\setminus A_s} + \sum_{i \in A_{s}}w_{s, i} + \sum_{i \in A_{s}} \frac{1}{e}\left(\frac{\eta_{s, i}}{\eta_{s+1, i}}-1\right) + \sum_{i\in A_{1:s+1}\setminus A_{s:s+1} }w_{s, i}\notag                                                                                                                                     \\
     & =\abs{A_{s+1}\setminus A_s} +\sum_{i\in \bigcup_{t=1}^s A_t} w_{s, i} + \sum_{i \in A_{s}} \frac{1}{e}\left(\frac{\eta_{s, i}}{\eta_{s+1, i}}-1\right) \leq N_{s+1} + \sum_{i\in \bigcup_{t=1}^s A_t}\sum_{t: i \in A_t }\frac{1}{e}\left(\frac{\eta_{t, i}}{\eta_{t+1, i}}-1\right),\label{eq:proof-meta-algorithm-upper-bound-W}
  \end{align}
  where the first inequality follows from Eq.~\eqref{eq:proof-sum-weights-two-round} and the newly initialized weights are set to $1$, and the last from induction. Therefore, we finish upper bounding the potential $W_{s+1}$. Then, we consider a lower bound for this potential. By noticing that $\ln W_{s+1} \geq \ln w_{s+1, i_\star}$ and with an induction argument:
  \begin{align*}
    \frac{1}{\eta_{s+1, i_\star}}\ln w_{s+1, i_\star} & = \frac{1}{\eta_{s, i_\star}}\left( \ln w_{s, i_\star} + \eta_{s, i_\star} \bar{r}_{s,i_\star} - \eta_{s,i_\star}^2(\bar{r}_{s,i_\star} - m_{s,i_\star})^2\right) \\
                                                      & =\frac{1}{\eta_{s, i_\star}}\ln w_{s, i_\star} -\eta_{s,i_\star}(\bar{r}_{s,i_\star} - m_{s,i_\star})^2 + \bar{r}_{s,i_\star}                                           \\
                                                      & =\frac{1}{\eta_{r, i_\star}}\ln w_{r, i_\star} - \sum_{t=r}^s \eta_{t,i_\star}(\bar{r}_{t,i_\star} - m_{t,i_\star})^2 + \sum_{t=r}^s \bar{r}_{t,i_\star}          \\
                                                      & = - \sum_{t=r}^s \eta_{t,i_\star}(\bar{r}_{t,i_\star} - m_{t,i_\star})^2 + \sum_{t=r}^s \bar{r}_{t,i_\star}. \tag*{ $(w_{r, i_\star} = 1)$}.
  \end{align*}
  Rearranging the above equality and using Eq.~\eqref{eq:proof-meta-algorithm-upper-bound-W}, we conclude:
  \begin{align*}
    \sum_{t=r}^s \bar{r}_{t,i_\star} \leq \frac{1}{\eta_{s+1, i_\star}}\ln w_{s+1, i_\star} + \sum_{t=r}^s \eta_{t,i_\star}(\bar{r}_{t,i_\star} - m_{t,i_\star})^2  + \frac{1}{\eta_{s+1, i_\star}}\ln \left( N_{s+1} + \frac{1}{e}\sum_{i\in \bigcup_{t=1}^s A_t}\sum_{t: i \in A_t }\left(\frac{\eta_{t, i}}{\eta_{t+1, i}}-1\right)\right).
  \end{align*}
\end{proof}

\subsection{Proof of Theorem~\ref{thm:meta-algorithm}}
\begin{proof}
  First, we prove the clipping technique~\citep{COLT'21:impossible-tuning,COLT'19:ashok-cubic} does not hurt the regret bound, apart from a constant factor:
  \begin{align}
    \sum_{t\in[r,s]} r_{t,\is} - \bar{r}_{t, \is} & = \sum_{t\in[r,s]} r_{t,\is}-m_{t, \is} - \frac{B_{t-1}}{B_{t}}(r_{t,\is}-m_{t, \is})\notag  =\sum_{t\in[r,s]} \frac{B_{t} - B_{t-1}}{B_{t}}(r_{t,\is}-m_{t, \is})\notag                                               \\
                                                  & \leq \sum_{t\in[r,s]} \frac{B_{t} - B_{t-1}}{B_{t}} \abs{r_{t,\is}-m_{t, \is}} \leq \sum_{t\in[r,s]} \frac{B_{t} - B_{t-1}}{B_{t}} \norm{\r_{t}-\m_{t}}_\infty \leq B_{s} - B_{r-1}.\label{eq:proof-clipping-constant}
  \end{align}
  In what follows, we focus on the analysis of the clipped regret $\rb_{t, \is}$. In the statement of Theorem~\ref{thm:meta-algorithm}, we set the optimism as $m_{t,i} = \inner{\p_t}{\h_t} - h_{t,i}$, which indicates that the first condition to apply Lemma~\ref{lemma:regret-upper-bound} is satisfied:
  \begin{align*}
    \sum_{i \in A_t} p_{t,i}\bar{r}_{t,i} & = \left(1 - \frac{B_{t-1}}{B_t}\right) \sum_{i \in A_t} p_{t,i} m_{t,i} + \frac{B_{t-1}}{B_t} \sum_{i\in A_t} p_{t,i}r_{t,i} = 0.
  \end{align*}
  As for the second condition, by the learning rate proposed in Line~\ref{line:redesign-lr}, for any $t$ and $i \in A_t$, we have
  \begin{align*}
    \eta_{t,i}\abs{\bar{r}_{t,i} - m_{t,i}} = \eta_{t,i} \frac{B_{t-1}}{B_t} \abs{r_{t,i} - m_{t,i}} \leq \frac{1}{2B_{t-1}} \cdot \frac{B_{t-1}}{B_t} \cdot B_t = \frac{1}{2},
  \end{align*}
  which meets the condition to apply Lemma~\ref{lemma:regret-upper-bound}. Therefore, the clipped regret can be bounded by,
  \begin{align}
    \label{eq:proof-after-applying-lemma-3}
    \sum_{t=r}^s \bar{r}_{t,i_\star} \leq \frac{1}{\eta_{s+1, i_\star}}\ln \frac{1}{w_{r,i_\star}} + \sum_{t=r}^s \eta_{t,i_\star}(\bar{r}_{t,i_\star} - m_{t,i_\star})^2+ \frac{1}{\eta_{s+1, i_\star}}\ln \left( N_{s+1} + \frac{1}{e}\sum_{i\in \bigcup_{t=1}^s A_t}\sum_{t: i \in A_t }\left(\frac{\eta_{t, i}}{\eta_{t+1, i}}-1\right)\right).
  \end{align}
  The notation $\bigcup_{t=r}^s A_t$ denotes the set of base learners that are active at least once in the interval $[r,s]$. Since we set $\w_{r,\is} = 1$, the first term becomes $0$.

  The second term can be further bounded by applying the self-confident tuning lemma~(Lemma~\ref{lemma:self-confident-int}):
  \begin{align}
    \sum_{t=r}^s\eta_{t,\is} (\bar{r}_{t,\is} - m_{t,\is})^2 & = \sum_{t=r}^s \sqrt{\frac{\gamma_{\is}}{ B_{t-1}^2 + \sum_{\tau=r}^{t-1} (\bar{r}_{\tau,\is} - m_{\tau,\is})^2}} \cdot (\bar{r}_{t,\is} - m_{t,\is})^2 \notag \\
                                                             & \leq  \sum_{t=r}^s \sqrt{\frac{\gamma_{\is}}{ \sum_{\tau=r}^{t-1} (\bar{r}_{\tau,\is} - m_{\tau,\is})^2}} \cdot (\bar{r}_{t,\is} - m_{t,\is})^2 \notag \\
                                                             & \leq 2\sqrt{\gamma_{\is} \cdot \sum_{t=r}^{s} (\bar{r}_{t,\is} - m_{t,\is})^2}. \label{eq:proof-meta-self-confident}
  \end{align}

  By Lemma~\ref{lemma:time-varying-lr}, the third term on the right-hand side of~\eqref{eq:proof-after-applying-lemma-3} can be bounded as
  \begin{align}
     & \ln \left( N_{s+1} + \frac{1}{e}\sum_{i \in \bigcup_{t=1}^s A_t}\sum_{t: i\in A_t}^s\left(\frac{\eta_{t, i}}{\eta_{t+1, i}}-1\right)\right) \notag                                                                                                                                                           \\
     & \leq \ln N_{s+1} + \ln \Bigg(1 + \frac{1}{e} \Bigg(\frac{B_s^2}{B_{0}^2} + \frac{1}{2} \ln \left(1 + s \cdot \frac{B_s^2}{B_{0}^2}\right) + \ln\left(\frac{B_s}{B_{0}}\right)  + \frac{\max_{t \in [1, s]} \abs{B_t - B_{t-1}}}{B_{0}}\Bigg) \Bigg) \notag \\
     & \triangleq \ln N_{s+1} + \tilde{\Gamma}_{s}. \label{eq:def-gamma-log-term}
  \end{align}
  In the above expression, we introduce the notation $\tilde{\Gamma}_{s}$. This quantity depends on both the actual execution of the algorithm and the environment.
  With the introduced notation $\tilde{\Gamma}_{s}$, the last term in~\eqref{eq:proof-after-applying-lemma-3} can be bounded by considering two cases of the learning rate $\eta_{s+1, \is}$:
  \begin{align}
    &\frac{1}{\eta_{s+1, \is}} \ln \left( N_{s+1} + \frac{1}{e}\sum_{i\in \bigcup_{t=1}^s A_t}\sum_{t: i \in A_t }\left(\frac{\eta_{t, i}}{\eta_{t+1, i}}-1\right)\right) \notag \\
    &\leq \frac{\ln N_{s+1} + \tilde{\Gamma}_{s}}{\eta_{s+1, \is}} \notag \\
    & \leq \frac{\ln N_{s+1} + \tilde{\Gamma}_{s}}{\sqrt{\gamma_{\is}}} \cdot \sqrt{\sum_{t=r}^{s} (\bar{r}_{t,\is} - m_{t,\is})^2 + B_{s}^2} + 2(\tilde{\Gamma}_{s} + \ln N_{s+1}) \cdot B_s \notag \\
    &\leq \frac{\ln N_{s+1} + \tilde{\Gamma}_{s}}{\sqrt{\gamma_{\is}}} \cdot \sqrt{\sum_{t=r}^{s} (\bar{r}_{t,\is} - m_{t,\is})^2 } + 3(\tilde{\Gamma}_{s} + \ln N_{s+1}) \cdot B_s. \label{eq:proof-meta-two-cases}
  \end{align}

  Combining Eq.~\eqref{eq:proof-clipping-constant} through Eq.~\eqref{eq:proof-meta-two-cases}, and noting that
  \begin{align*}
     \abs{\bar{r}_{t, \is} - m_{t, \is}} = \abs{\frac{B_{t-1}}{B_t}(r_{t,i} - m_{t,i})} \leq \abs{r_{t,i} - m_{t,i}},
  \end{align*}
  we complete the proof.
\end{proof}

  \section{Proof of Interval Dynamic Regret with Smoothness Constant $L$}
\label{appendix:proof-interval-dynamic-regret}
In this section, we show that Algorithm~\ref{alg:la-adaptive-oco} achieves a Lipschitz-adaptive interval dynamic regret guarantee. We begin by establishing a seemingly sub-optimal bound: for any interval $I$ and any comparator sequence, we obtain
$\tilde{\mathcal{O}}\left(P_I \sqrt{\bar{V}_I}\right)$,
where $\bar{V}_{I}$ denotes the empirical gradient variation defined in~\eqref{eq:empirical-gradient-variation}. We then demonstrate that once this result is established, it can be improved to
$\tilde{\mathcal{O}}\left(\sqrt{P_I \cdot \bar{V}_{I}}\right)$
through a carefully designed interval partitioning.

\subsection{Justification for Known $L$ and Unknown $G$}
\label{subappendix:justification}
We focus on the setting where the smoothness constant $L$ is known but the Lipschitz constant $G$ is unknown. One motivation for studying this case is that the resulting guarantee directly implies the scenario where both $G$ and $L$ are known. This allows for a fair comparison with previous work, which typically assumes knowledge of both constants to achieve computationally efficient results.

In addition, we provide a justification for the ``known $L$, unknown $G$'' setting using the canonical example of online linear regression. At each round $t$, the learner selects a parameter vector $\x_t \in \mathcal{X}$ from a bounded domain. The environment then reveals a data point $(\z_t, y_t)$, and the learner incurs the loss $f_t(\x) = \frac{1}{2} (\x^\top \z_t - y_t)^2$.

The smoothness constant $L$ is the upper bound on the spectral norm of the Hessian matrix, $\nabla^2 f_t(\x) = \z_t \z_t^\top$. In our online setting, the algorithm at round $t$ only requires a smoothness parameter that depends on past data. We can therefore use an online estimate based on all data seen up to time $t-1$:
$$ \hat{L}_{t-1} = \max_{s \in [t-1]} \norm{\nabla^2 f_s(\x)}_2 =  \max_{s \in [t-1]} \norm{\z_s \z_s^\top}_2 =  \max_{s \in [t-1]} \norm{\z_s}_2^2. $$
At the beginning of round $t$, $\hat{L}_{t-1}$ is a known constant computed from past data. As noted in Algorithm~\ref{alg:la-adaptive-oco}, the smoothness parameter required at time $t$ is indeed only dependent on data up to time $t-1$.

In contrast, a bound on the Lipschitz constant $G$ is required at time $t$ \emph{before} the loss function $f_t(\x)$ is revealed. The gradient is $\nabla f_t(\x) = (\x^\top \z_t - y_t) \z_t$. A tight upper bound on the gradient norm must account for the current data:
$$\hat{G}_t = \max_{s \in [t], \x \in \mathcal{X}} \norm{\nabla f_s(\x)}_2 = \sup_{s \in [t], \x \in \X} \norm{(\x^\top \z_s - y_s) \z_s}.$$
The magnitude of the gradient $\nabla f_t(\x)$ directly depends on the target value $y_t$. In a non-stationary online setting, $y_t$ cannot be predicted in advance. Therefore, any tight bound $\hat{G}_t$ that includes the $t$-th term is unknown at the start of round $t$, and the Lipschitz constant must be learned in an online fashion.

\subsection{Key Lemmas}
\begin{myLemma}
    \label{lemma:interval-dynamic-interval-in-Ct}
    Under the same assumptions and algorithmic configurations as in Theorem~\ref{thm:la-adaptive-oco-with-L}, consider any interval $I = [s_i, s_j - 1] \in \tilde{\mathcal{S}}$ and any time $\tau \in [s_i, s_j - 1]$.
    For any sequence of comparators $\u_{s_i}, \dots, \u_{\tau}$, Algorithm~\ref{alg:la-adaptive-oco} guarantees the following anytime regret bound:
    \begin{align*}
        \sum_{t=s_i}^{\tau} \inner{\nabla f_t(\x_t)}{\x_t - \u_t} \leq \sqrt{1 + \bar{V}_{[s_i, \tau]}} \cdot \left(P_{[s_i, \tau]} + D\cdot \frac{3 \ln (2 i + 1) + \tilde{\Gamma}_{\tau}}{\sqrt{\ln  (2 i + 1)}} + \frac{5D}{2}\right) + 3(\tilde{\Gamma}_{\tau}+\ln(2i + 1)) B_{\tau} + B_{\tau} - B_{s_i},
    \end{align*}
    where $\bar{V}_{[a, b]}$ denotes the empirical gradient variation defined in~\eqref{eq:empirical-gradient-variation}, $\tilde{\Gamma}_{\tau}$ is defined in~\eqref{eq:def-gamma}, and $B_{\tau} = \max\{\max_{t \in [1, \tau], i \in A_t} \abs{\inner{\nabla f_t(\x_t) - \nabla f_{t-1}(\x_{t-1})}{\x_t - \x_{t,i}}}, B_{0}\}$ denotes the maximum input scale of the meta algorithm.
\end{myLemma}
\begin{proof}
    For the algorithmic design in Algorithm~\ref{alg:la-adaptive-oco}, we know that there exists a base learner, the $i$-th base learner, which is active throughout the interval. For this base learner, we can decompose the regret into two terms:
    \begin{align*}
        \sum_{t=s_i}^{\tau} \inner{\nabla f_t(\x_t)}{\x_t - \u_t} = \underbrace{\sum_{t=s_i}^{\tau} \inner{\nabla f_t(\x_t)}{\x_t - \x_{t,i}}}_{\textsc{Meta Reg}} + \underbrace{\sum_{t=s_i}^{\tau} \inner{\nabla f_t(\x_t)}{\x_{t,i} - \u_t}}_{\textsc{Base Reg}}.
    \end{align*}

    For \textsc{Meta Reg}, because in Line~\ref{line:lagair-meta-optimism} of Algorithm~\ref{alg:la-adaptive-oco}, we set the optimism for the $i$-th base learner as $m_{t,i} = \inner{\nabla f_{t-1}(\x_{t-1})}{\x_t - \x_{t,i}}$ and $\gamma_i = \ln (2i + 1)$, by Theorem~\ref{thm:meta-algorithm}, we have:
    \begin{align}
        \textsc{Meta Reg} &\leq \sqrt{\sum_{t=s_{i}}^{\tau} ({r}_{t,i} - m_{t,i})^2 } \cdot \frac{2 \ln (2 i + 1) + \ln N_{\tau+1} + \tilde{\Gamma}_{\tau}}{\sqrt{\ln  (2 i + 1)}} + 3(\tilde{\Gamma}_{\tau} + \ln N_{\tau + 1}) B_{\tau} + B_{\tau} - B_{s_i} \notag \\
        &= \sqrt{\sum_{t=s_{i}}^{\tau} \inner{\nabla f_t(\x_t) - \nabla f_{t-1}(\x_{t-1})}{\x_t - \x_{t,i}}^2 } \cdot \frac{2 \ln (2 i + 1)+ \ln N_{\tau+1} + \tilde{\Gamma}_{\tau}}{\sqrt{\ln  (2 i + 1)}} + 3(\tilde{\Gamma}_{\tau} + \ln N_{\tau + 1}) B_{\tau} + B_{\tau} - B_{s_i} \notag \\
        &\leq D\sqrt{\bar{V}_I} \cdot \frac{2 \ln (2 i + 1)+ \ln N_{\tau+1} + \tilde{\Gamma}_{\tau}}{\sqrt{\ln  (2 i + 1)}} + 3(\tilde{\Gamma}_{\tau} + \ln N_{\tau+1}) B_{\tau} + B_{\tau} - B_{s_i} \label{eq:proof-before-N-tau}
    \end{align}
    where $\tilde{\Gamma}_{\tau}$ is an algorithm-dependent quantity, formally defined in Theorem~\ref{thm:meta-algorithm}, Eq.~\eqref{eq:def-gamma}, and $N_{\tau + 1}$ denotes the number of initialized base learners at time $\tau + 1$.

    We next upper bound $N_{\tau + 1}$. When employing Algorithm~\ref{alg:leo-ada-ml-prod} as the meta algorithm in Algorithm~\ref{alg:la-adaptive-oco}, we initialize a base learner every time a marker is registered. Therefore, there are $i$ base learners in total at time $t = s_i$. Meanwhile, by the construction of the problem-dependent schedule $\tilde{\mathcal{S}}$ defined in~\eqref{eq:problem-dependent-schedule}, the subscripts of markers satisfy $j \leq 2i$ for $[s_i, s_j - 1] \in \tilde{\mathcal{S}}$. Therefore, we can estimate the number of initialized base learners up to time $s_j$,
    \begin{align*}
       N_{\tau+1} \leq  N_{s_j} \leq 2i + 1.
    \end{align*}
    Plugging the preceding inequality into Eq.~\eqref{eq:proof-before-N-tau} finishes the analysis of \textsc{Meta Reg}.

    As for \textsc{Base Reg}, in Line~\ref{line:lagair-base-optimism} of Algorithm~\ref{alg:la-adaptive-oco}, we set the optimism for the $i$-th base learner as $M_t = \nabla f_{t-1}(\x_{t-1})$, then by Lemma~\ref{thm:two-step-algorithm}, we can obtain
    \begin{align*}
        \textsc{Base Reg} \leq \left(P_I + \frac{5D}{2}\right)\sqrt{1 + \bar{V}_I}.
    \end{align*}

    Combining the results of \textsc{Meta Reg} and \textsc{Base Reg}, we conclude the proof.
\end{proof}

\begin{myLemma}
    \label{lemma:counting-number-barv}
     Under the same assumptions and algorithmic configurations as in Theorem~\ref{thm:la-adaptive-oco-with-L}, for any interval $I = [s_i, s_{i+1} - 1]$ defined by consecutive markers, the following lower bound holds:
    \begin{align}
        \label{eq:counting-number-barv}
        \min_{\u_{s_i}, \dots, \u_{s_{i+1} - 1}} \left\{ 2\sum_{t=s_i}^{s_{i+1} - 1} f_t(\u_t) + 28L P_{[s_i, s_{i+1} - 1]}^2 + P_{[s_i, s_{i+1} - 1]} \right\} \geq \frac{1}{2}\mathcal{T}_{\text{LA}, s_{i+1}},
    \end{align}
    where $P_{[r, s]} = \sum_{t=r+1}^s \norm{\u_t - \u_{t-1}}_2$ denotes the path length of the comparator sequence over the interval $[r, s]$, with $\{\u_{t}\}_{t=s_i}^{s_{i+1}-1}$ corresponding to the minimizers in the minimum operator in Eq.~\eqref{eq:counting-number-barv}, and $\mathcal{T}_{\text{LA}, s_{i+1}} = \mathcal{G}_{\text{LA}}(s_{i+1} - 1, s_i, i)$ is the threshold value when registering the marker $s_{i+1}$ with the threshold generation function defined in~\eqref{eq:la-threshold-function}.

    More specifically, for any stopping time $\tau \in [s_i, s_{i+1} - 1]$, we have an instantaneous upper bound on $i$:
    \begin{align*}
         i &\leq 1 + \frac{2}{\mathcal{T}_{\text{LA}, s_{1}}} \min_{\u_1, \dots, \u_\tau \in \X} \left(2\sum_{t=1}^{\tau} f_t(\u_t) +28L  P_{[1, \tau]}^2 + P_{[1, \tau]} \right)\\
         &\leq 1 + \frac{4}{\mathcal{T}_{\text{LA}, s_{1}}} \min_{\u \in \X} \sum_{t=1}^{\tau} f_t(\u).
    \end{align*}
\end{myLemma}
\begin{proof}
    Lemma~\ref{lemma:interval-dynamic-interval-in-Ct} presents an anytime result. Therefore, we can set $\tau = s_{i+1} -1$, and begin by applying this lemma:
    \begin{align}
        &\sum_{t=s_i}^{s_{i+1}-1} f_t(\x_t) - f_t(\u_t) \notag \\
        &\leq \sum_{t=s_i}^{s_{i+1}-1} \inner{\nabla f_t(\x_t)}{\x_t - \u_t} \notag \\
        &\leq \sqrt{1 + \bar{V}_I} \cdot \left(P_I + D\cdot \frac{3 \ln (2 i + 1) + \tilde{\Gamma}_{s_{i+1} - 1}}{\sqrt{\ln  (2 i + 1)}} + \frac{5D}{2}\right) + 3(\tilde{\Gamma}_{s_{i+1} - 1} + \ln (2i+1)) B_{s_{i+1} - 1} + B_{s_{i+1} - 1} - B_{s_i} \notag \\
        &\leq \sqrt{1 + 4\sum_{t=s_i}^{s_{i+1} - 1} \norm{\nabla f_t(\x_t)}_2^2} \cdot \left(P_I + D\cdot \frac{3 \ln (2 i + 1) + \tilde{\Gamma}_{s_{i+1} - 1}}{\sqrt{\ln  (2 i + 1)}} + \frac{5D}{2}\right) \notag\\
        &\quad + 3(\tilde{\Gamma}_{s_{i+1} - 1} + \ln (2i+1)) B_{s_{i+1} - 1} + B_{s_{i+1} - 1} - B_{s_i} \notag \\
        &\leq \sqrt{1 + 8L\sum_{t=s_i}^{s_{i+1} - 1} f_t(\x_t)} \cdot \left(P_I + D\cdot \frac{3 \ln (2 i + 1) + \tilde{\Gamma}_{s_{i+1} - 1}}{\sqrt{\ln  (2 i + 1)}} + \frac{5D}{2}\right) \notag\\
        &\quad  + 3(\tilde{\Gamma}_{s_{i+1} - 1} + \ln (2i+1)) B_{s_{i+1} - 1} + B_{s_{i+1} - 1} - B_{s_i}, \label{eq:proof-lemma-counting-number-barv-convert}
    \end{align}
    where the third inequality follows from the fact:
    $$\bar{V}_I = \sum_{t=s_i + 1}^{s_{i+1} - 1} \norm{\nabla f_t(\x_t) - \nabla f_{t-1}(\x_{t-1})}_2^2 \leq \sum_{t=s_i + 1}^{s_{i+1} - 1} 2\norm{\nabla f_t(\x_t)}_2^2 + 2\norm{\nabla f_{t-1}(\x_{t-1})}_2^2 \leq 4\sum_{t=s_i}^{s_{i+1} - 1} \norm{\nabla f_t(\x_t)}_2^2,$$
    and the last inequality is by Lemma~\ref{lemma:self-bounded}, the self-bounding property of smooth and non-negative functions.

    Our next step is to employ Lemma~\ref{lemma:substitute-F_T} to substitute the terms on the right-hand side involving $f_t(\x_t)$ by $f_t(\u_t)$. To simplify the notation, we introduce
    \begin{align*}
        \alpha_{I} = D\cdot \frac{3 \ln (2 i + 1) + \tilde{\Gamma}_{s_{i+1} - 1}}{\sqrt{\ln  (2 i + 1)}} + \frac{5D}{2};\quad \beta_I = 3(\tilde{\Gamma}_{s_{i+1} - 1} + \ln (2i+1)) B_{s_{i+1} - 1} + B_{s_{i+1} - 1} - B_{s_i}, \quad \text{for } I = [s_i, s_{i+1} - 1],
    \end{align*}
    then Eq.~\eqref{eq:proof-lemma-counting-number-barv-convert} becomes:
    \begin{align*}
        \sum_{t=s_i}^{s_{i+1}-1} f_t(\x_t) - f_t(\u_t) &\leq (\alpha_{I}+P_I) \cdot \sqrt{1 + 8L\sum_{t=s_i}^{s_{i+1} - 1} f_t(\x_t)} + \beta_I
    \end{align*}

    By using Lemma~\ref{lemma:substitute-F_T}, the above inequality implies:
    \begin{align*}
        \sum_{t=s_i}^{s_{i+1}-1} f_t(\x_t) - \sum_{t=s_i}^{s_{i+1}-1} f_t(\u_t) &\leq (\alpha_I + P_I) \cdot \sqrt{1 + 8L\sum_{t=s_i}^{s_{i+1} - 1} f_t(\u_t)} + 12L(\alpha_I + P_I)^2+ \frac{3}{2}\beta_I\\
        &\leq \sum_{t=s_i}^{s_{i+1} - 1} f_t(\u_t) + 14L(\alpha_I + P_I)^2 + \alpha_I + P_I + \frac{3}{2}\beta_I,
    \end{align*}
    where the last inequality is by the AM-GM inequality. Rearranging and splitting the unobservable terms involving $\{\u_t\}_{t=s_i}^{s_{i+1}-1}$ by using $(a+b)^2 \leq 2a^2 + 2b^2$:
    \begin{align*}
        \sum_{t=s_i}^{s_{i+1}-1} f_t(\x_t) &\leq 2\sum_{t=s_i}^{s_{i+1} - 1} f_t(\u_t) + 28LP_I^2 + P_I + \underbrace{28L\alpha_I^2 + \alpha_I + \frac{3}{2}\beta_I}_{=\frac{1}{2}\mathcal{T}_{\text{LA}, s_{i+1}}}.
    \end{align*}
    By the algorithmic design, the marker $s_{i+1}$ is registered because the cumulative loss up to time $s_{i+1}-1$ exceeded the threshold: $\sum_{t=s_i}^{s_{i+1}-1} f_t(\x_t) \geq \mathcal{T}_{\text{LA}, s_{i+1}}$. With this condition, we can provide a lower bound of the unobservable quantities:
    \begin{align*}
         2\sum_{t=s_i}^{s_{i+1} - 1} f_t(\u_t) + 28LP_I^2 + P_I \geq \sum_{t=s_i}^{s_{i+1}-1} f_t(\x_t) - \frac{1}{2}\mathcal{T}_{\text{LA}, s_{i+1}} \geq \frac{1}{2}\mathcal{T}_{\text{LA}, s_{i+1}}.
    \end{align*}
    Notice that the above inequality holds for any comparator sequence $\u_{s_i}, \dots, \u_{s_{i+1}-1}$, therefore, it also holds for the minimum value of the left-hand side over all comparator sequences, which concludes the first part of the proof.

    Next, by summing the above inequality from $[s_1, s_2 - 1]$ to $[s_{i-1}, s_{i} - 1]$, we can estimate an upper bound on $i$:
    \begin{align*}
        \sum_{j=1}^{i-1} \left(2\sum_{t=s_j}^{s_{j+1} - 1} f_t(\u_t) + 28LP_{[s_j, s_{j+1}-1]}^2 + P_{[s_j, s_{j+1}-1]}\right) &\geq \sum_{j=1}^{i-1} \frac{1}{2}\mathcal{T}_{\text{LA}, s_{j+1}} \geq \frac{i-1}{2} \cdot \mathcal{T}_{\text{LA}, s_{1}}.
    \end{align*}
    Rearranging the above inequality:
    \begin{align*}
       i &\leq 1+ \frac{2}{\mathcal{T}_{\text{LA}, s_{1}}} \left(2\sum_{t=1}^{s_{i}-1} f_t(\u_t) + P_{[1, s_{i}-1]} +28L \left(\sum_{j \in [i-1]} P_{[s_j, s_{j+1}-1]}\right)^2 \right)\\
       &= 1 + \frac{2}{\mathcal{T}_{\text{LA}, s_{1}}} \left(2\sum_{t=1}^{s_{i}-1} f_t(\u_t) + P_{[1, s_{i}-1]} +28L  P_{[1, s_{i}-1]}^2\right)\\
       &\leq 1 + \frac{2}{\mathcal{T}_{\text{LA}, s_{1}}} \left(2\sum_{t=1}^{\tau} f_t(\u_t) + P_{[1, \tau]} +28L  P_{[1, \tau]}^2\right) \tag*{($\forall \tau \in [s_i, s_{i+1} - 1]$)}
    \end{align*}
    Notice that the above inequality holds for any comparators; therefore, by choosing $\u_1 = \dots = \u_{\tau} = \argmin_{\u \in \X} \sum_{t=1}^{\tau} f_t(\u)$, we have:
    \begin{align*}
        i &\leq 1 + \frac{4}{\mathcal{T}_{\text{LA}, s_{1}}} \min_{\u \in \X} \sum_{t=1}^{\tau} f_t(\u),
    \end{align*}
    which concludes the proof.
\end{proof}

\begin{myLemma}
    \label{lemma:interval-dynamic-regret-untuned}
    Under the same assumptions and algorithmic configurations as in Theorem~\ref{thm:la-adaptive-oco-with-L}, for any arbitrary interval $[r, s] \subseteq [1, T]$, there exists a time $\tau \in [r,s]$ such that Algorithm~\ref{alg:la-adaptive-oco} guarantees the following interval dynamic regret:
    \begin{align*}
        \sum_{t=r}^s f_t(\x_t) -  \sum_{t=r}^s f_t(\u_t)
        &\leq \left(P_{[\tau, s]} + D\cdot \left(3\sqrt{\ln \left(3 + \frac{4}{\mathcal{T}_{\text{LA}, s_{1}}}F_{[1, s]}\right)} + \tilde{\Gamma}_{s}\right)+ \frac{5D}{2}\right) \sqrt{\bar{V}_{[\tau, s]}} \cdot \sqrt{1 + \log_2\left(1 + \frac{4}{\mathcal{T}_{\text{LA}, s_{1}}}F_{[r, s]}\right)}\notag\\
        &\quad + \left(P_{[\tau, s]} + D\cdot \left(3\sqrt{\ln \left(3 + \frac{4}{\mathcal{T}_{\text{LA}, s_{1}}}F_{[1, s]} \right)} + \tilde{\Gamma}_{s}\right)+ \frac{5D}{2}  \right) \cdot \log_2\left(1 + \frac{4}{\mathcal{T}_{\text{LA}, s_{1}}}F_{[r, s]}\right)\\
        &\quad + 3 \log_2\left(1 + \frac{4}{\mathcal{T}_{\text{LA}, s_{1}}}F_{[r, s]}\right) \ln \left(3 + \frac{4}{\mathcal{T}_{\text{LA}, s_{1}}}F_{[1, s]}\right) B_{s} + GD + 2B_{s} - B_{r} - \sum_{t\in[\tau, s]}\D_{f_t} (\u_t, \x_t)\\
        &\quad + \mathcal{T}_{\text{LA}, \text{marker before $r$}}\\
        &= \O\Bigg((P_{[r, s]} + \sqrt{\log F_{[1,s]}} + {\tilde{\Gamma}_{s}}) \sqrt{\bar{V}_{[r,s]} \cdot \log F_{[r,s]}} \\
        &\quad +  (P_{[r, s]} + \sqrt{\log F_{[1,s]}}+ {\tilde{\Gamma}_{s}})\log F_{[r,s]} + \log F_{[1,s]}\log F_{[r,s]}B_s + \tilde{\Gamma}_{s}^2\Bigg).
    \end{align*}
    where $\bar{V}_{[a,b]}$ is the empirical gradient variation on the interval $[a,b]$, defined in Eq.~\eqref{eq:empirical-gradient-variation}, $F_{[a,b]} = \min_{\x \in \X} \sum_{t=a}^b f_t(\x)$ denotes the small loss on the interval $[a,b]$, $\tilde{{\Gamma}}_{[a,b]}$ is defined in~\eqref{eq:def-gamma}, $\mathcal{T}_{\text{LA}, s_1} = \mathcal{G}_{\text{LA}}(0, s_0, 0)$ is the threshold value when registering the first marker $s_1$ with the threshold generation function defined in~\eqref{eq:la-threshold-function}.
    We use $\mathcal{T}_{\text{LA}, \text{marker before $r$}}$ to denote the threshold value when registering the marker just before time $r$. This value can be upper bounded as:
    \begin{align*}
         \mathcal{T}_{\text{LA}, \text{marker before $r$}} &\leq  56LD^2\left(3\sqrt{\ln \left(3 + \frac{4}{\mathcal{T}_{\text{LA}, s_{1}}}F_{[1, s]}\right)} +\tilde{\Gamma}_{s} +\frac{5}{2}\right)^2 + 2D\left(3\sqrt{\ln \left(3 + \frac{4}{\mathcal{T}_{\text{LA}, s_{1}}}F_{[1, s]}\right)} +\tilde{\Gamma}_{s}\right)\\
        &\quad + 5D+ 3\left(3\left(\tilde{\Gamma}_{s} + \ln \left(3 + \frac{4}{\mathcal{T}_{\text{LA}, s_{1}}}F_{[1, s]}\right)\right)B_{s} + B_s - B_{0}\right).
    \end{align*}
\end{myLemma}

\begin{proof}
    For any given interval $I = [r, s]$, we define two key markers: $s_p$ as the smallest marker such that $s_p > r$, and $s_q$ as the largest marker such that $s_q \leq s$. This implies that $s_{p-1} \leq r < s_p$ and $s_q \leq s < s_{q+1}$.

    By applying Lemma~\ref{lemma:pcgc-number}, we can identify a sequence of $v$ consecutive intervals from the problem-dependent schedule $\tilde{\mathcal{S}}$ (defined in~\eqref{eq:problem-dependent-schedule}): \begin{align*} I_1=[s_{i_1}, s_{i_2} - 1], I_2 = [s_{i_2}, s_{i_3} - 1], \ldots, I_v = [s_{i_v}, s_{i_{v+1}} - 1], \end{align*} where these intervals are chosen such that $i_1 = p$ and $i_v \leq q < i_{v+1}$. The number of such intervals, $v$, is bounded by: \begin{align} \label{eq:proof-number-of-intervals-v} v \leq \lceil \log_2(q - p + 2) \rceil. \end{align} Since $s < s_{q+1}$ and, by construction, $s_{q+1} \leq s_{i_{v+1}}$, it follows that $s \leq s_{i_{v+1}} - 1$. This establishes a cover for the interval $I$, given by $I \subseteq [s_{p-1}, s_p-1] \cup \bigcup_{j=1}^v I_j$.

    With the aforementioned intervals, specifically, we can decompose the given interval $[r, s]$ into two parts: $[r, s_{p} - 1]$ and $[s_{i_1}, s]$. For the latter part, Lemma~\ref{lemma:interval-dynamic-interval-in-Ct} provides the corresponding guarantee. For the former part, we will deal with it using the non-negative assumption of the loss functions~(Assumption~\ref{ass:smoothness}) as follows:
    \begin{align}
        \label{eq:proof-use-threshold-to-bound}
        \sum_{t=r}^{s_{p} - 1}f_t(\x_t) -  \sum_{t=r}^{s_{p} - 1}f_t(\u_t) \leq \sum_{t=r}^{s_{p} - 1}f_t(\x_t) \leq \sum_{t=s_{p-1}}^{s_{p} - 1}f_t(\x_t) =  \sum_{t=s_{p-1}}^{s_{p} - 2}f_t(\x_t) + f_{s_{p} - 1}(\x_{s_{p} - 1}) \leq \mathcal{T}_{\text{LA}, s_{p}} + GD,
    \end{align}
    where the last inequality follows from the fact that the cumulative loss at time $s_p - 2$ does not exceed the threshold, together with the boundedness assumption on the loss function values. For the remaining interval $[s_{i_1}, s]$, we will employ Lemma~\ref{lemma:interval-dynamic-interval-in-Ct} for the analysis. For simplicity, we slightly abuse notation by writing $I_v = [s_{i_v}, s_{i_{v+1}} - 1] \cap [r,s]$ :
    \begin{align}
        &\sum_{t=s_{i_1}}^{s}f_t(\x_t) - f_t(\u_t) \notag \\
        &= \sum_{k=1}^v \sum_{t \in I_k } f_t(\x_t)- f_t(\u_t)\notag\\
        &=\sum_{k=1}^v \sum_{t \in I_k } \inner{\nabla f_t(\x_t)}{\x_t - \u_t} - \D_{f_t} (\u_t, \x_t)\notag \\
        &\leq\sum_{k=1}^v \sqrt{1 + \bar{V}_{I_k }} \cdot \left(P_{I_k } + D\cdot \frac{3 \ln (2 i_k + 1) + \tilde{\Gamma}_{ s_{i_{k+1}} - 1}}{\sqrt{\ln  (2 i_k + 1)}} + \frac{5D}{2}\right) + 3(\tilde{\Gamma}_{ s_{i_{k+1}} - 1} + \ln (2i_k + 1)) B_{\min\{s_{i_{k+1}} - 1, s\}} \notag\\
        &\quad + B_{\min\{s_{i_{k+1}} - 1, s\}} - B_{s_{i_k}} - \sum_{t\in[s_{i_1}, s]}\D_{f_t} (\u_t, \x_t)\notag\\
        &\leq \left(P_{[s_{i_1}, s]} + D\cdot \left(3\sqrt{\ln (2i_{v} + 1)} + \tilde{\Gamma}_{s}\right)+ \frac{5D}{2}\right) \sqrt{v  + \bar{V}_{[s_{i_1}, s]}} \cdot \sqrt{v}\notag\\
        &\quad + 3\cdot v\cdot (\tilde{\Gamma}_{s}+\ln(1+2i_v))B_{s} + B_{s} - B_{s_{i_1}} - \sum_{t\in[s_{i_1}, s]}\D_{f_t} (\u_t, \x_t), \label{eq:proof-before-cal-v-i-v}
    \end{align}
    where in the last line we apply the Cauchy-Schwarz inequality. Next, we aim to provide a more explicit upper bound for the number of intervals $v$ and the number of initialized base learners $i_v$. By Lemma~\ref{lemma:counting-number-barv}, from $s_p$ to $s_q$:
    \begin{align*}
        \min_{\u_{s_p}^\prime, \dots, \u_{s_q - 1}^\prime }\sum_{k=p}^{q}   \sum_{t \in [s_k, s_{k+1} - 1]} 2f_t(\u^\prime_t) + 28L P_{[s_k, s_{k+1}-1]}^2 + P_{[s_k, s_{k+1}-1]} &\geq \sum_{k=p}^{q} \frac{1}{2}\mathcal{T}_{\text{LA}, s_{k+1}} \geq \frac{q-p+1}{2} \cdot \mathcal{T}_{\text{LA}, s_{1}},
    \end{align*}
    which implies that
    \begin{align*}
        v\leq 1 + \log_2 (q - p + 2) \leq 1 +\log_2\left(  1 +  \min_{\u_{r}, \dots, \u_{s} } \frac{2}{\mathcal{T}_{\text{LA}, s_{1}}}\sum_{t=r}^{s} 2f_t(\u_t) + P_{[r, s]} + 28L P_{[r, s]}^2 \right).
    \end{align*}
    Again, by Lemma~\ref{lemma:counting-number-barv}, and by following the same arguments while summing from $s_1$ to $s_q$, we can determine the value of $i_v$,
    \begin{align*}
        i_v \leq q \leq 1 + \min_{\u_{1}, \dots, \u_{s} } \frac{2}{\mathcal{T}_{\text{LA}, s_{1}}} \left(2\sum_{t=1}^{s} f_t(\u_t) + P_{[1, s]} +28L  P_{[1, s]}^2\right).
    \end{align*}
    Plugging the above two inequalities into the previous regret bound~(Eq.~\eqref{eq:proof-before-cal-v-i-v}), and choosing the comparators $\{\u_t\}$ in the minimum operator properly to be the small-loss optimizers, we conclude the proof of the regret over $[s_{i_1}, s]$.

    For the remaining segment $[r, s_p - 1]$, we conduct a worst-case analysis of the quantity $\mathcal{T}_{\text{LA}, s_p}$ in Eq.~\eqref{eq:proof-use-threshold-to-bound}, as follows:
    \begin{align*}
        \mathcal{T}_{\text{LA}, s_p}&= \mathcal{G}_{\text{LA}}(s_p - 1, s_{p-1}, p - 1)\\
        &\leq 56LD^2\left(3\sqrt{\ln (2 p + 1)} +\tilde{\Gamma}_{s} +\frac{5}{2}\right)^2 + 2D(3\sqrt{\ln (2 p + 1)} +\tilde{\Gamma}_{s})\\
        &\quad + 5D+ 3\left(3(\tilde{\Gamma}_{s} + \ln (2p + 1))B_{s} + B_s - B_{s_{p-1}}\right)\\
        &\leq 56LD^2\left(3\sqrt{\ln (2 i_v + 1)} +\tilde{\Gamma}_{s} +\frac{5}{2}\right)^2 + 2D(3\sqrt{\ln (2 i_v + 1)} +\tilde{\Gamma}_{s})\\
        &\quad + 5D+ 3\left(3(\tilde{\Gamma}_{s} + \ln (2i_v + 1))B_{s} + B_s - B_{0}\right)\\
        &\leq 56LD^2\left(3\sqrt{\ln \left(3 + \frac{4}{\mathcal{T}_{\text{LA}, s_{1}}}F_{[1, s]}\right)} +\tilde{\Gamma}_{s} +\frac{5}{2}\right)^2 + 2D\left(3\sqrt{\ln \left(3 + \frac{4}{\mathcal{T}_{\text{LA}, s_{1}}}F_{[1, s]}\right)} +\tilde{\Gamma}_{s}\right)\\
        &\quad + 5D+ 3\left(3\left(\tilde{\Gamma}_{s} + \ln \left(3 + \frac{4}{\mathcal{T}_{\text{LA}, s_{1}}}F_{[1, s]}\right)\right)B_{s} + B_s - B_{0}\right)\\
        &=\O\left(\tilde{\Gamma}_{s}^2 + \left(\tilde{\Gamma}_{s} + \log(F_{[1,s]})\right)B_s \right),
    \end{align*}
    which finishes the proof.
\end{proof}

\subsection{Interval Dynamic Regret Guarantee for GAIR-L with Known Smoothness}
\label{subappendix:interval-dynamic-formal-with-L}
\begin{myThm}[formal]
    \label{thm:interval-dynamic-regret-formal}
    Under Assumptions~\ref{ass:bounded-domain}--\ref{ass:smoothness}, additionally assuming that the smoothness constant $L$ is known to the learner, choosing Algorithm~\ref{alg:la-adaptive-oco} as the meta algorithm, setting the step size of OMD in~\eqref{eq:base-algorithm-omd} as in Theorem~\ref{thm:adaptive-oco}, applying the problem-dependent schedule defined in Eq.~\eqref{eq:problem-dependent-schedule}, setting the threshold function as specified in Eq.~\eqref{eq:la-threshold-function}, then for any comparators $\u_r, \dots, \u_s \in \X$ and any interval $I = [r,s ]\subseteq [T]$, GAIR-L achieves the following interval dynamic regret bound:
    \begin{align*}
        \sum_{t=r}^s f_t(\x_t) -  \sum_{t=r}^s f_t(\u_t)& \leq \O\Bigg( \Big(\sqrt{\log F_{[1,s]}} +\tilde{\Gamma}_{s}\Big)\sqrt{ \min\{V_I^{\u}, F_I^{\u}\}(1 + P_I) \log F_{[r,s]}} \\ & \quad\quad\quad +\Big(\sqrt{\log F_{[1,s]}} + \tilde{\Gamma}_{s}\Big)^2\log F_{[r,s]}  P_I + (\tilde{\Gamma}_{s} + \log F_{[1,s]})\log F_{[r,s]} B_{s}P_I \Bigg),
    \end{align*}
  where $F_{[a, b]} = \min_{\x \in \X} \sum_{t=a}^b f_t(\x)$ denotes the small loss, $F_I^{\u} = \sum_{t=r}^s f_t(\u_t)$ denotes the cumulative loss of comparators, $V_I^{\u} = \sum_{t=r+1}^s \norm{\nabla f_t(\u_t) - \nabla f_{t-1}(\u_t)}_2^2 \leq V_I$ represents the gradient variation in terms of comparators $\u_r, \dots, \u_s$, $P_{[a,b]} = \sum_{t=a+1}^b \norm{\u_t - \u_{t-1}}_2$ is the path length, $\tilde{\Gamma}_{s}$ is defined in~\eqref{eq:def-gamma}, and $B_{s} = \max\{\max_{t \in [1, s], i \in A_t} \abs{\inner{\nabla f_t(\x_t) - \nabla f_{t-1}(\x_{t-1})}{\x_t - \x_{t,i}}},2G_0D\}$ denotes the maximum input scale of meta algorithm.
\end{myThm}
\subsection{Proof of Theorem~\ref{thm:interval-dynamic-regret-formal}}
\label{sub-appendix:proof-interval-dynamic-regret}
\begin{proof}
    The proof is presented in $\tilde{\mathcal{O}}(\cdot)$ notation, omitting logarithmic dependencies on $T$ and $B_T$, as well as constant factors that do not affect the asymptotic order.
    Consider any interval $I = [r, s] \subseteq [1, T]$ and an arbitrary comparator sequence $\u_r, \dots, \u_s \in \X$. By the bounded domain assumption (Assumption~\ref{ass:bounded-domain}), the path length satisfies $P_I = \sum_{t=r+1}^s \norm{\u_t - \u_{t-1}}_2 \leq (s - r) D$. We can partition $I$ into subintervals $I_1, \dots, I_k$ such that, for each $I_i = [r_i, s_i]$, the path length is at most a constant, i.e.,
    \begin{align*}
        P_{I_i} = \sum_{t = r_i + 1}^{s_i} \norm{\u_t - \u_{t-1}}_2 \leq D.
    \end{align*}
    Such a decomposition always exists; in the worst case, we can partition $I$ into consecutive intervals of the form $\{[t, t+1]\}$, each with path length at most $D$. The number of intervals $k$ is thus bounded by
    \begin{align}
        \label{eq:proof-number-of-D-partition}
        k \leq \lceil P_I / D \rceil \leq 1 + P_I / D.
    \end{align}
    For each subinterval $I_i = [r_i, s_i]$, Lemma~\ref{lemma:interval-dynamic-regret-untuned} ensures the existence of $\tau_i \in [r_i, s_i]$ such that
    \begin{align*}
        \sum_{t=r_i}^{s_i} f_t(\x_t) -  \sum_{t=r_i}^{s_i} f_t(\u_t) \leq \Ot\left(P_{[\tau_i, s_i]} \cdot \sqrt{\bar{V}_{[\tau_i,s_i]}} + P_{[\tau_i, s_i]} + B_{s_i} - \sum_{t=\tau_i}^{s_i} \D_{f_t}(\u_t, \x_t) \right),
    \end{align*}
    where $\bar{V}_{[\tau_i,s_i]} = \sum_{t=\tau_i}^{s_i} \norm{\nabla f_t(\x_t) - \nabla f_{t-1}(\x_{t-1})}_2^2$ denotes the empirical gradient variation over $[\tau_i, s_i]$.
    Summing over $i=1$ to $k$ and using $P_{I_i} \leq D$ for all $i$, we obtain
    \begin{align}
        \sum_{t=r}^s f_t(\x_t) -  \sum_{t=r}^s f_t(\u_t) &= \sum_{i=1}^k \left(\sum_{t=r_i}^{s_i} f_t(\x_t) -  \sum_{t=r_i}^{s_i} f_t(\u_t)\right) \notag \\
        &\leq \Ot\left(\sum_{i=1}^k P_{[\tau_i, s_i]} \cdot \sqrt{\bar{V}_{[\tau_i,s_i]}} + P_{[\tau_i, s_i]} + B_{s_i} - \sum_{t=\tau_i}^{s_i} \D_{f_t}(\u_t, \x_t) \right) \notag \\
        &= \Ot\left(\sum_{i=1}^k D \cdot \sqrt{\bar{V}_{[\tau_i,s_i]}} + D + B_{s_i} - \sum_{t=\tau_i}^{s_i} \D_{f_t}(\u_t, \x_t) \right) \notag \\
        &\leq \Ot\left(D\cdot \sqrt{k \cdot \sum_{i=1}^k \bar{V}_{[\tau_i,s_i]}} + k D + k B_{s} - \sum_{i=1}^k \sum_{t=\tau_i}^{s_i} \D_{f_t}(\u_t, \x_t) \right) \notag \\
        &= \Ot\left(D\sqrt{(1 + P_I/D)\cdot \sum_{i=1}^k \bar{V}_{[\tau_i,s_i]}} + (1 + P_I/D)(D + B_s) - \sum_{i=1}^k \sum_{t=\tau_i}^{s_i} \D_{f_t}(\u_t, \x_t) \right) \notag \\
        &= \Ot \left(\sqrt{(1 + P_I)\cdot \sum_{i=1}^k \bar{V}_{[\tau_i,s_i]}} + B_s P_I/D - \sum_{i=1}^k \sum_{t=\tau_i}^{s_i} \D_{f_t}(\u_t, \x_t)\right), \label{eq:proof-before-case-analysis}
    \end{align}
    where the fourth inequality follows from the Cauchy-Schwarz inequality.

    \paragraph{Small-loss bound} The empirical gradient variation $\bar{V}_{[\tau_i,s_i]}$ can be upper bounded by the cumulative gradient norms:
    \begin{align*}
          \sum_{t=r}^s f_t(\x_t) -  \sum_{t=r}^s f_t(\u_t) &\leq \Ot \left(\sqrt{(1 + P_I)\cdot \sum_{i=1}^k \bar{V}_{[\tau_i,s_i]}} + P_I - \sum_{i=1}^k \sum_{t=\tau_i}^{s_i} \D_{f_t}(\u_t, \x_t)\right)\\
          &\leq \Ot \left(\sqrt{(1 + P_I)\cdot \sum_{i=1}^k \sum_{t=\tau_i}^{s_i} \norm{\nabla f_t(\x_t)}_2^2} + P_I \right)\\
          &\leq \Ot \left(\sqrt{(1 + P_I)\cdot\sum_{t=r}^s \norm{\nabla f_t(\x_t)}_2^2} + P_I \right)\\
          &\leq \Ot \left(\sqrt{(1 + P_I)\cdot\sum_{t=r}^s f_t(\x_t)} + P_I \right),
    \end{align*}
    where the last inequality uses Lemma~\ref{lemma:self-bounded}, which states the self-bounding property of smooth and non-negative functions. Applying Lemma~\ref{lemma:substitute-F_T} to substitute $f_t(\x_t)$ with $f_t(\u_t)$ yields the small-loss bound:
    \begin{align*}
         \sum_{t=r}^s f_t(\x_t) -  \sum_{t=r}^s f_t(\u_t) \leq \Ot \left(\sqrt{(1 + P_I)\cdot\sum_{t=r}^s f_t(\u_t)} + P_I \right) = \Ot \left(\sqrt{(1+P_I)(F_I + P_I)}\right).
    \end{align*}
    \paragraph{Gradient-variation bound} Following the decomposition approach from Section~\ref{subsec:key-analysis}, but with additional care for dynamic comparators, we note that for dynamic sequences, the negative Bregman divergence terms become $-\sum_{t\in I}\D_{f_t}(\u_t, \x_t)$. We decompose the gradient variation as follows:
    \begin{align*}
        \norm{\nabla f_t(\x_t) - \nabla f_{t-1}(\x_{t-1})}_2^2 &\leq 3\left( \norm{\nabla f_t(\x_t) - \nabla f_t(\u_t)}_2^2 + \norm{ \nabla f_t(\u_t) -  \nabla f_{t-1}(\u_{t-1})}_2^2 + \norm{ \nabla f_{t-1}(\u_{t-1}) - \nabla f_{t-1}(\x_{t-1})}_2^2 \right)\\
        &\leq  3\left( 2L \D_{f_t}(\u_t, \x_t) + \norm{ \nabla f_t(\u_t) -  \nabla f_{t-1}(\u_{t-1})}_2^2 + 2L \D_{f_{t-1}}(\u_{t-1}, \x_{t-1}) \right),
    \end{align*}
    where the first and last terms can be canceled by the negative Bregman divergence terms in Eq.~\eqref{eq:proof-before-case-analysis}. The middle term can be further bounded by the path length and the gradient variation:
    \begin{align*}
        \norm{ \nabla f_t(\u_t) -  \nabla f_{t-1}(\u_{t-1})}_2^2 &\leq 2\norm{ \nabla f_t(\u_t) -  \nabla f_{t-1}(\u_{t})}_2^2 + 2\norm{ \nabla f_{t-1}(\u_t) -  \nabla f_{t-1}(\u_{t-1})}_2^2\\
        &\leq 2\norm{ \nabla f_t(\u_t) -  \nabla f_{t-1}(\u_{t})}_2^2 + 2L^2\norm{\u_t - \u_{t-1}}_2^2,
    \end{align*}
    where the second inequality follows from the smoothness assumption. Thus, for dynamic comparators, we have the key decomposition:
    \begin{align*}
         \norm{\nabla f_t(\x_t) - \nabla f_{t-1}(\x_{t-1})}_2^2 \leq \O\left( \D_{f_t}(\u_t, \x_t) + \D_{f_{t-1}}(\u_{t-1}, \x_{t-1}) + \norm{ \nabla f_t(\u_t) -  \nabla f_{t-1}(\u_{t})}_2^2 +\norm{\u_t - \u_{t-1}}_2^2 \right).
    \end{align*}

    Combining the above, and starting from Eq.~\eqref{eq:proof-before-case-analysis}, we obtain:
    \begin{align*}
         \sum_{t=r}^s f_t(\x_t) -  \sum_{t=r}^s f_t(\u_t) &\leq \Ot \left(\sqrt{(1 + P_I)\cdot \sum_{i=1}^k \bar{V}_{[\tau_i,s_i]}} + P_I - \sum_{i=1}^k \sum_{t=\tau_i}^{s_i} \D_{f_t}(\u_t, \x_t)\right)\\
         &\leq \Ot \Bigg(\sqrt{(1 + P_I)\cdot \sum_{i=1}^k \left(\sum_{t=\tau_i }^{s_i} \D_{f_t}(\u_t, \x_t) + \sum_{t=\tau_i + 1}^{s_i} \norm{\nabla f_t(\u_t) - \nabla f_{t-1}(\u_t)}_2^2 + \norm{\u_t - \u_{t-1}}_2^2 \right)}\\
         &\quad  + P_I - \sum_{i=1}^k \sum_{t=\tau_i}^{s_i} \D_{f_t}(\u_t, \x_t)\Bigg)\\
         &\leq \Ot \Bigg(\sqrt{(1 + P_I)\cdot \sum_{i=1}^k \sum_{t=\tau_i + 1}^{s_i} \norm{\nabla f_t(\u_t) - \nabla f_{t-1}(\u_t)}_2^2} \\
         &\quad + \sqrt{(1 + P_I)\cdot \sum_{i=1}^k \sum_{t=\tau_i }^{s_i} \D_{f_t}(\u_t, \x_t)} + 2 P_I - \sum_{i=1}^k \sum_{t=\tau_i}^{s_i} \D_{f_t}(\u_t, \x_t)\Bigg)\\
         &\leq \Ot \Bigg(\sqrt{(1 + P_I)\cdot\sum_{t=r + 1}^{s} \sup_{\x \in \X} \norm{\nabla f_t(\x) - \nabla f_{t-1}(\x)}_2^2} + 3P_I \Bigg)\\
         &= \Ot \left(\sqrt{(1+P_I) \left(V_I + P_I\right)}\right),
    \end{align*}
    where the third inequality uses $\sqrt{a + b} \leq \sqrt{a} + \sqrt{b}$, and the last step leverages the negative Bregman divergence and the AM-GM inequality $\sqrt{ab} - b \leq a/4$ to absorb the additional terms.
\end{proof}

\section{Proof of Interval Dynamic Regret without Smoothness Constant $L$}
\label{appendix:proof-interval-dynamic-regret-without-L}
In this section, we provide proofs of the results when both the Lipschitz constant $G$ and the smoothness constant $L$ are unknown. The key difference in the algorithmic design and analysis lies in the use of the problem-independent schedule $\mathcal{S}$ defined in~\eqref{eq:problem-independent-schedule}, instead of the problem-dependent schedule $\tilde{\mathcal{S}}$. Consequently, we no longer need to design a threshold generation function, and thus the smoothness constant $L$ is not required.

It is worth noting that the proof here closely follows that in Appendix~\ref{appendix:proof-interval-dynamic-regret}, and hence we only provide a sketch. In Appendix~\ref{appendix:proof-interval-dynamic-regret}, part of the effort is devoted to designing the threshold generation function and proving that the number of maintained base learners is related to the small-loss quantity. Once the problem-independent schedule is adopted, such additional arguments are no longer necessary.
\subsection{Key Lemmas}
Similar to Lemma~\ref{lemma:interval-dynamic-interval-in-Ct}, we can also establish an anytime regret bound for any interval in the problem-independent schedule $\mathcal{S}$ defined in~\eqref{eq:problem-independent-schedule}, summarized in Corollary~\ref{cor:interval-dynamic-interval-in-C}.
\begin{myCor}
    \label{cor:interval-dynamic-interval-in-C}
    Under the same assumptions and algorithmic configurations as in Theorem~\ref{thm:la-adaptive-oco-without-L}, for any interval in the problem-independent schedule, $I = [i, j - 1] \in {\mathcal{S}}$, for any time $\tau \in [i, j - 1]$, for any arbitrary comparators $\u_{i}, \dots, \u_{\tau}$, then Algorithm~\ref{alg:la-adaptive-oco} guarantees an anytime regret bound:
    \begin{align*}
        \sum_{t=i}^{\tau} \inner{\nabla f_t(\x_t)}{\x_t - \u_t} \leq \sqrt{1 + \bar{V}_I} \cdot \left(P_{[i, \tau]} + D\cdot \frac{3 \ln (2 i + 1) + \tilde{\Gamma}_{\tau}}{\sqrt{\ln  (2 i + 1)}} + \frac{5D}{2}\right) + 3(\tilde{\Gamma}_{\tau} + \ln(2i+1)) B_{\tau} + B_{\tau} - B_{i},
    \end{align*}
    where $\tilde{\Gamma}_{\tau}$ is defined in~\eqref{eq:def-gamma}, and $B_{\tau} = \max\{\max_{t \in [1, \tau], i \in A_t} \abs{\inner{\nabla f_t(\x_t) - \nabla f_{t-1}(\x_{t-1})}{\x_t - \x_{t,i}}}, B_0\}$ denotes the maximum input scale of meta algorithm.
\end{myCor}
The proof of Corollary~\ref{cor:interval-dynamic-interval-in-C} is almost identical to that of Lemma~\ref{lemma:interval-dynamic-interval-in-Ct}, except that the interval $[i, j - 1]$ is used instead of $[s_i, s_j - 1]$; the proof is therefore omitted. Based on this corollary, we can further establish a lemma analogous to Lemma~\ref{lemma:interval-dynamic-regret-untuned}.
Notably, compared to Lemma~\ref{lemma:interval-dynamic-regret-untuned}, the analysis here becomes simpler because the problem-independent schedule initializes a new base learner at every time step $t \in [T]$, ensuring that any arbitrary interval can be directly covered without additional analysis.

\begin{myLemma}
    \label{lemma:interval-dynamic-regret-untuned-independent}
    Under the same assumptions and algorithmic configurations as in Theorem~\ref{thm:la-adaptive-oco-without-L}, for any arbitrary interval $[r, s] \subseteq [1, T]$, Algorithm~\ref{alg:la-adaptive-oco} guarantees the following interval dynamic regret with suboptimal dependence on the path length:
    \begin{align*}
        \sum_{t=r}^s f_t(\x_t) -  \sum_{t=r}^s f_t(\u_t)\leq \O\left( \Big(P_{I} + \sqrt{\log s} + \frac{\tilde{\Gamma}_{s}}{\sqrt{\log r}}\Big)\sqrt{ \bar{V}_{I} \cdot \log (s-r)} +(\tilde{\Gamma}_{s} + \log s)\log(s-r) B_{s} -  \sum_{t=r}^s \D_{f_t}(\u_t, \x_t)\right),
    \end{align*}
    where $\tilde{\Gamma}_{s}$ is defined in~\eqref{eq:def-gamma}, and $B_{\tau} = \max\{\max_{t \in [1, \tau], i \in A_t} \abs{\inner{\nabla f_t(\x_t) - \nabla f_{t-1}(\x_{t-1})}{\x_t - \x_{t,i}}}, B_0\}$ denotes the maximum input scale of meta algorithm.
\end{myLemma}
\begin{proof}
    By Lemma~\ref{lemma:cgc-number}, for any interval $[r, s] \subseteq [1, T]$, we can find a sequence of $v$ consecutive intervals from the problem-independent schedule $\mathcal{S}$ defined in~\eqref{eq:problem-independent-schedule}: \begin{align*} I_1=[i_1, i_2 - 1], I_2 = [i_2, i_3 - 1], \ldots, I_v = [i_v, i_{v+1} - 1], \end{align*} where these intervals are chosen such that $i_1 = r$ and $i_v \leq s \leq i_{v+1} - 1$. The number of such intervals, $v$, is bounded by:
    \begin{align*}
        v \leq \lceil \log_2(s - r + 2) \rceil.
    \end{align*}

Corollary~\ref{cor:interval-dynamic-interval-in-C} ensures an anytime regret bound for each interval $I_k \in \mathcal{S}, k \in [v]$; therefore, for simplicity of notation we assume that $i_{v+1} - 1 = s$. By summing over all such intervals and applying the Cauchy-Schwarz inequality, we can derive the following results:
\begin{align*}
    &\sum_{t=r}^s f_t(\x_t) - \sum_{t=r}^s f_t(\u_t)\\
    &= \sum_{t=r}^s \inner{\nabla f_t(\x_t)}{\x_t - \u_t} - \sum_{t=r}^s \D_{f_t}(\u_t, \x_t)\\
    &\leq \O\left(\sum_{k=1}^{v} \Big(P_{I_k} + \sqrt{\log i_k} + \frac{\tilde{\Gamma}_{{i_{k+1} - 1} }}{\sqrt{\log i_k}}\Big)\sqrt{ \bar{V}_{I_k}} +(\tilde{\Gamma}_{{i_{k+1} - 1} } + \log i_k) B_{i_{k+1}} -  \sum_{t=r}^s \D_{f_t}(\u_t, \x_t)\right)\\
    &\leq \O\left( \Big(P_{I} + \sqrt{\log s} + \frac{\tilde{\Gamma}_{s}}{\sqrt{\log r}}\Big)\sqrt{ \bar{V}_{I} \cdot \log (s-r)} +(\tilde{\Gamma}_{s} + \log s)\log(s-r) B_{s} -  \sum_{t=r}^s \D_{f_t}(\u_t, \x_t)\right),
\end{align*}
which finishes the proof.
\end{proof}

\subsection{Formal Statement of Interval Dynamic Regret without Smoothness Constant $L$}
\label{subappendix:interval-dynamic-formal-without-L}
\begin{myThm}[formal]
    \label{thm:interval-dynamic-regret-without-L-formal}
    Under the same assumptions and algorithmic configurations as in Theorem~\ref{thm:la-adaptive-oco-without-L}, for any arbitrary interval $[r, s] \subseteq [1, T]$, Algorithm~\ref{alg:la-adaptive-oco} guarantees the following interval dynamic regret:
    \begin{align*}
        \sum_{t=r}^s f_t(\x_t) -  \sum_{t=r}^s f_t(\u_t)& \leq \O\Bigg( \Big(\sqrt{\log s} + \frac{\tilde{\Gamma}_{s}}{\sqrt{\log (1+r)}}\Big)\sqrt{ \min\{V_I^{\u}, F_I^{\u}\}(1 + P_I) \log (s-r)} \\ & \quad\quad\quad +\Big(\sqrt{\log s} + \frac{\tilde{\Gamma}_{s}}{\sqrt{\log (1+r)}}\Big)^2\log(s-r) P_I + (\tilde{\Gamma}_{s} + \log s)\log(s-r) B_{s}P_I \Bigg),
    \end{align*}
    where $V_I^{\u} = \sum_{t=r+1}^s \norm{\nabla f_t(\u_t) - \nabla f_{t-1}(\u_t)}_2^2 \leq V_I$ represents the gradient variation in terms of comparators $\u_r, \dots, \u_s$, $F_I^{\u} = \sum_{t=r}^s f_t(\u_t)$ denotes the cumulative loss of comparators, $P_I = \sum_{t=r+1}^s \norm{\u_t - \u_{t-1}}_2$ is the path length, $\tilde{\Gamma}_{s}$ is defined in~\eqref{eq:def-gamma}, and $B_{s} = \max\{\max_{t \in [1, s], i \in A_t} \abs{\inner{\nabla f_t(\x_t) - \nabla f_{t-1}(\x_{t-1})}{\x_t - \x_{t,i}}}, B_0\}$ denotes the maximum input scale of meta algorithm.
\end{myThm}
The proof of this theorem follows essentially the same steps as that of Theorem~\ref{thm:interval-dynamic-regret-formal} in Appendix~\ref{sub-appendix:proof-interval-dynamic-regret}, leveraging Lemma~\ref{lemma:interval-dynamic-regret-untuned-independent}, decomposing the interval $[r, s]$ into subintervals with small path length, and exploiting the negative Bregman divergence terms to cancel parts of the gradient variation. Therefore, we omit the details.

  \section{Proof of Stochastic Extended Adversarial Model}
\label{appendix:proof-sea}
\begin{proof}[Proof of Corollary~\ref{cor:sea}]
    We start by noticing that the gradient variation based on noisy gradients in Theorem~\ref{thm:interval-dynamic-regret-without-L-formal} can be decomposed in the following manner:
    \begin{align*}
        \norm{\nabla f_t(\u_t) - \nabla f_{t-1}(\u_t)}_2^2 &\leq \O\left(\norm{\nabla f_t(\u_t) - \nabla F_{t}(\u_t)}_2^2+\norm{\nabla F_t(\u_t) - \nabla F_{t-1}(\u_t)}_2^2+\norm{\nabla F_{t-1}(\u_t) - \nabla f_{t-1}(\u_t)}_2^2\right).
    \end{align*}
    Here, $\u_t$ denotes an arbitrary comparator and is therefore independent of the randomness in the algorithm and the environment. Therefore, by taking the expectation on both sides of the inequality in Theorem~\ref{thm:interval-dynamic-regret-without-L-formal} we have:
    \begin{align*}
        &\E\left[\sum_{t=r}^s f_t(\x_t) -  \sum_{t=r}^s f_t(\u_t)\right]\\
        &\leq \Ot\left(\E\left[\sqrt{(1+P_{[r,s]})\sum_{t=r}^s \norm{\nabla f_t(\u_t) - \nabla f_{t-1}(\u_t)}_2^2} + P_{[r,s]} \right]\right)\\
        &\leq \Ot\left(\E\left[\sqrt{(1+P_{[r,s]})\left(\sum_{t=r}^s \norm{\nabla f_t(\u_t) - \nabla F_{t}(\u_t)}_2^2 + \sum_{t=r}^s \norm{\nabla F_t(\u_t) - \nabla F_{t-1}(\u_t)}_2^2\right)} + P_{[r,s]} \right]\right)\\
        &\leq \Ot\left(\E\left[\sqrt{(1+P_{[r,s]})\left(\sum_{t=r}^s \norm{\nabla f_t(\u_t) - \nabla F_{t}(\u_t)}_2^2 + \sum_{t=r}^s \sup_{\x\in\X}\norm{\nabla F_t(\x) - \nabla F_{t-1}(\x)}_2^2\right)} + P_{[r,s]} \right]\right)\\
        &\leq \Ot\left(\sqrt{(1+P_{[r,s]})\left(\sigma_{[r,s]}^2 + \Sigma_{[r,s]}^2\right)} + P_{[r,s]} \right),
    \end{align*}
    where the second inequality uses the decomposition of gradient variation (the first and third terms both contribute to the stochastic variance $\sigma_{[r,s]}^2$ upon taking expectations), and the last inequality uses the concavity of the square root function $\sqrt{x}$.
\end{proof}

  \section{Supporting Lemmas}
This section collects several useful lemmas used in the proof of the main results.
\begin{myLemma}[Extension of Lemma 14 from \cite{COLT'14:second-order-Hedge}]
    \label{lemma:self-confident-int}
    Let $a_0 > 0$ and $a_t \in [0, B]$ be real numbers for all $t \in [T]$ and let $f:(0, +\infty) \to [0, +\infty)$ be a nonincreasing function. Then
    \begin{align*}
        \sum_{t=1}^T a_t f\left(\sum_{s=0}^{t-1} a_s \right) \leq B\cdot f(a_0) + \int_{a_0}^{\sum_{t=0}^T a_t} f(u)\mathrm{d}u.
    \end{align*}
\end{myLemma}
\begin{myLemma}[Lemma 3.5 of~\cite{JCSS'02:Auer-self-confident}]
    \label{lemma:self-confident-tuning}
    Let $a_1, \dots, a_T$ and $\delta$ be non-negative real numbers. Then
    \begin{align*}
        \sum_{t=1}^T \frac{a_t}{\sqrt{\delta + \sum_{s=1}^t a_s}} \leq 2\left(\sqrt{\delta + \sum_{t=1}^T a_t} - \sqrt{\delta}\right).
    \end{align*}
\end{myLemma}
\begin{myLemma}[Lemma 3.1 of~\cite{NIPS'10:smooth}]
    \label{lemma:self-bounded}
    Suppose $f: \R^d \to \R$ is an $L$-smooth function, then for any $\x \in \R^d$, we have
    \begin{equation*}
        \norm{\nabla f(\x)}_2^2 \leq 2L \cdot \left( f(\x) - \min_{\x \in \R^d} f(\x) \right).
    \end{equation*}
    Furthermore, when the function is non-negative, we have $\norm{\nabla f(\x)}_2^2 \leq 2L \cdot f(\x)$.
\end{myLemma}
\begin{myLemma}
    \label{lemma:substitute-F_T}
    For any $x, y, a, b, c \in \mathbb{R}_+$ satisfying $x - y \leq \sqrt{ax + b} + c$, we have
    $$x - y \leq \sqrt{ay + b + ac} + a + c \leq \sqrt{ay + b } + \frac{3}{2}(a + c).$$
\end{myLemma}
\begin{myLemma}[Lemma 10 of~\cite{ICML19:Zhang-Adaptive-Smooth}]
    \label{lemma:cgc-number}
    For any interval $[r, s] \subseteq [T]$, we can find a sequence of consecutive intervals $$I_1 = [i_1, i_2 - 1], I_2 = [i_2, i_3 - 1], \ldots, I_v = [i_v, i_{v+1} - 1] \in \mathcal{S},$$
    such that $i_1 = r$, $i_v \leq s < i_{v+1}$, and $v \leq \lceil \log_2(s-r+2) \rceil$.
\end{myLemma}
\begin{myLemma}[Lemma 11 of~\cite{ICML19:Zhang-Adaptive-Smooth}]
    \label{lemma:pcgc-number}
    For any consecutive intervals $$[s_a, s_b - 1], [s_b, s_c - 1] \in \tilde{\mathcal{S}},$$ we must have $c-b \geq 2 (b-a)$.

    Let $[s_p, s_q] \subseteq [T]$ be an arbitrary interval that starts from a marker $s_p$ and ends at another marker $s_q$. Then, we can find a sequence of consecutive intervals $$I_1=[s_{i_1}, s_{i_2} - 1], I_2 = [s_{i_2}, s_{i_3} - 1], \ldots, I_v = [s_{i_v}, s_{i_{v+1}} - 1] \in \tilde{\mathcal{S}},$$  such that $i_1 = p$, $i_v\leq q < i_{v+1}$, and $v \leq \lceil \log_2(q-p+2) \rceil$.
\end{myLemma}

}

\end{document}